\newcommand*{\COLT}{}

\newcommand*{\CAMREADY}{}

\ifdefined\NIPS
	\PassOptionsToPackage{numbers}{natbib}
	\documentclass{article}
	\ifdefined\CAMREADY	
		\usepackage[final]{nips_2016}
	\else
		\usepackage{nips_2016}
	\fi
	\usepackage[utf8]{inputenc}
	\usepackage[T1]{fontenc}
	\usepackage{hyperref}       	
	\usepackage{url}            	
	\usepackage{booktabs}
	\usepackage{amsfonts}
	\usepackage{nicefrac}       	
	\usepackage{microtype}
\fi
\ifdefined\CVPR
	\documentclass[10pt,twocolumn,letterpaper]{article}
	\usepackage{cvpr}
	\usepackage{times}
	\usepackage{epsfig}
	\usepackage{graphicx}
	\usepackage{amsmath}
	\usepackage{amssymb}
	\usepackage[square,comma,numbers]{natbib}
\fi
\ifdefined\AISTATS
	\documentclass[twoside]{article}
	\ifdefined\CAMREADY
		\usepackage[accepted]{aistats2015}
	\else
		\usepackage{aistats2015}
	\fi
	\usepackage{url}
	\usepackage[round,authoryear]{natbib}
\fi
\ifdefined\ICML
	\documentclass{article}
	\usepackage{times}
	\usepackage{graphicx}
	\usepackage{subfigure}
	\usepackage{algorithm}
	\usepackage{algorithmic}
	\usepackage{hyperref}
	
	\ifdefined\CAMREADY
		\usepackage[accepted]{icml2016}
	\else
		\usepackage{icml2016}
	\fi
	\usepackage{natbib}
\fi
\ifdefined\ICLR
	\documentclass{article}
	\usepackage{iclr2017_conference,times}
	\usepackage{hyperref}
	\usepackage{url}

	\ifdefined\CAMREADY
		\iclrfinalcopy
	\fi
\fi
\ifdefined\COLT
	\ifdefined\CAMREADY
		\documentclass{colt2017}
	\else
		\documentclass[anon,12pt]{colt2017}
	\fi
	\usepackage{times}
\fi

\usepackage{color}
\usepackage{amsfonts}
\usepackage{algorithm}
\usepackage{algorithmic}
\usepackage{graphicx}
\usepackage{nicefrac}
\usepackage{bbm}
\usepackage{wrapfig}
\usepackage{tikz}
\usepackage[titletoc,title]{appendix}
\usepackage{stmaryrd}
\ifdefined\COLT
\else
	\usepackage{amsthm}
\fi

\global\long\def\bra#1{\left\langle #1\right|}
\global\long\def\ket#1{\left|#1\right\rangle }
\global\long\def\braket#1#2{\left\langle #1|#2\right\rangle }

\ifdefined\COLT
	\newtheorem{claim}[theorem]{Claim}
	\newtheorem{fact}[theorem]{Fact}
	
\else
	
	\newtheorem{lemma}{Lemma}
	
	\newtheorem{theorem}{Theorem}

	\newtheorem{claim}{Claim}
	
\fi

\def\be{\begin{equation}}
\def\ee{\end{equation}}
\def\beas{\begin{eqnarray*}}
\def\eeas{\end{eqnarray*}}
\def\bea{\begin{eqnarray}}
\def\eea{\end{eqnarray}}
\newcommand{\h}{{\mathbf h}}
\newcommand{\x}{{\mathbf x}}

\newcommand{\uu}{{\mathbf u}}
\newcommand{\vv}{{\mathbf v}}
\newcommand{\w}{{\mathbf w}}

\newcommand{\aaa}{{\mathbf a}}

\newcommand{\A}{{\mathcal A}}
\newcommand{\B}{{\mathcal B}}

\newcommand{\G}{{\mathcal G}}

\newcommand{\X}{{\mathcal X}}
\newcommand{\Y}{{\mathcal Y}}

\newcommand{\R}{{\mathbb R}}
\newcommand{\N}{{\mathbb N}}

\newcommand{\abs}[1]{\left\lvert#1 \right\rvert}

\DeclareMathOperator*{\argmax}{argmax}
\DeclareMathOperator*{\argmin}{argmin}

\newcommand{\cupdot}{\mathbin{\mathaccent\cdot\cup}}

\newcommand{\mat}[1]{\llbracket#1\rrbracket}

\newcommand{\tmm}{sharir2016tensorial}

\ifdefined\CVPR
	\usepackage[breaklinks=true,bookmarks=false]{hyperref}
	\ifdefined\CAMREADY
		\cvprfinalcopy
	\fi
	
\fi

\ifdefined\ICML
	\newcommand*{\ABBR}{}
\fi
\ifdefined\NIPS
	\newcommand*{\ABBR}{}
\fi
\ifdefined\ICLR
	\newcommand*{\ABBR}{}
\fi
\ifdefined\COLT
	\newcommand*{\ABBR}{}
\fi
\ifdefined\ABBR
	\newcommand{\eg}{\emph{e.g.}}
	\newcommand{\ie}{\emph{i.e.}}

	\newcommand{\wrt}{w.r.t.}

\fi

\begin{document}


\ifdefined\ICLR
	\title{Deep Learning and Quantum Entanglement: Fundamental Connections with Implications to Network Design}
	\author{}
	\maketitle
\fi
\ifdefined\COLT
	\title[Deep Learning and Quantum Entanglement]{Deep Learning and Quantum Entanglement:\\ Fundamental Connections with Implications to Network Design}
	\coltauthor{\Name{Yoav Levine} \Email{yoavlevine@cs.huji.ac.il}\\
	\Name{David Yakira} \Email{davidyakira@cs.huji.ac.il}\\
	\Name{Nadav Cohen} \Email{cohennadav@cs.huji.ac.il}\\
	\Name{Amnon Shashua} \Email{shashua@cs.huji.ac.il}\\
	\addr The Hebrew University of Jerusalem}
	\maketitle
\fi

\begin{abstract}
Deep convolutional networks have witnessed unprecedented success in various machine learning applications. Formal understanding on what makes these networks so successful is gradually unfolding, but for the most part there are still significant mysteries to unravel. The inductive bias, which reflects prior knowledge embedded in the network architecture, is one of them. In this work, we establish a fundamental connection between the fields of quantum physics and deep learning. We use this connection for asserting novel theoretical observations regarding the role that the number of channels in each layer of the convolutional network fulfills in the overall inductive bias. Specifically, we show an equivalence between the function realized by a deep convolutional arithmetic circuit (ConvAC) and a quantum many-body wave function, which relies on their common underlying tensorial structure. This facilitates the use of quantum entanglement measures as well-defined quantifiers of a deep network's expressive ability to model intricate correlation structures of its inputs. Most importantly, the construction of a deep convolutional arithmetic circuit in terms of a Tensor Network is made available. This description enables us to carry a graph-theoretic analysis of a convolutional network, tying its expressiveness to a min-cut in the graph which characterizes it. Thus, we demonstrate a direct control over the inductive bias of the designed deep convolutional network via its channel numbers, which we show to be related to the min-cut in the underlying graph. This result is relevant to any practitioner designing a convolutional network for a specific task. We theoretically analyze convolutional arithmetic circuits, and empirically validate our findings on more common convolutional networks which involve ReLU activations and max pooling. Beyond the results described above, the description of a deep convolutional network in well-defined graph-theoretic tools and the formal structural connection to quantum entanglement, are two interdisciplinary bridges that are brought forth by this work.

\end{abstract}

\section{Introduction} \label{sec:intro}

A central factor in the application of machine learning to a given task is the restriction of the hypothesis space of learned functions known as \emph{inductive bias}.
The restriction posed by the inductive bias is necessary for practical learning, and reflects prior knowledge regarding the task at hand.
In deep convolutional networks, prior knowledge is embedded in architectural features such as number of layers, number of channels per layer, the pattern of pooling, various schemes of connectivity and convolution kernel defined by size and stride (see \cite{LeCun:2015dt} for an overview). Formal understanding of the inductive bias behind convolutional networks is limited~--~the assumptions encoded into these models, which seem to form an excellent prior knowledge for imagery data (\eg~\cite{Krizhevsky:2012wl,simonyan2014very,Szegedy:2014tb,he2015deep}) are for the most part a mystery.

A key observation that facilitates reasoning about inductive bias, is that the influence of an architectural attribute (such as connectivity, number of channels per layer) can be measured by its contribution to the effectiveness of the representation of correlations between regions of the input. In this regard, one considers different partitions that divide input regions into disjoint sets, and asks how far the function (realized by the network) is from being separable with respect to these partitions. For example, \cite{cohen2017inductive} show that when separability is measured through the algebraic notion of separation-rank, deep convolutional arithmetic circuits (ConvACs) support exponentially high separation ranks for certain input partitions, while being limited to polynomial or linear (in network size) separation ranks for others. They show that the network's pooling geometry effectively determines which input partitions are favored in terms of separation rank, \ie~which partitions enjoy the possibility of exponentially high separation rank with polynomial network size, and which require the network to be exponentially large.

In this work, we draw upon formal similarities between how physicists represent a system of many-particles as a quantum mechanical wave function, to how machine learning practitioners map a many-regions image to a set of output labels through a deep network. In particular, we show that there is a one-to-one structural equivalence between a function modeled by a ConvAC (\cite{cohen2016expressive}) and a many-body quantum wave function.
This allows employment of the well-established physical notion of quantum entanglement measures (see overview in \cite{plenio2005introduction}),  which subsumes other algebraic notions of separability such as the separation rank mentioned above, for the analysis of correlations modeled by deep convolutional networks.

Moreover, and most importantly, quantum entanglement is used by physicists as a prior knowledge to form compact representations of the many-body wave functions in what is known as {\it Tensor Networks} (TNs), (\cite{ostlund1995thermodynamic,verstraete2004renormalization,vidal2008class}). In machine learning, a network in the form of a ConvAC is effectively a compact representation of a multi-dimensional array containing the convolutional weights. The function realized by the network is analyzed via tensor decompositions --- where the representations are based on linear combinations of outer-products of lower-order tensors. Such analyses of a ConvAC via tensor decompositions follows several recent works utilizing tensor decompositions for theoretical studies of deep learning (see for example~\cite{Janzamin:2015uz,sedghi2016training}), and in particular builds on the equivalence between hierarchical tensor decompositions and convolutional networks established in~\cite{cohen2016expressive} and~\cite{cohen2016convolutional}. A TN, on the other hand, is a way to compactly represent a higher-order tensor through contractions (or inner-products) among lower-order tensors. A TN also has the important quality of a representation through an underlying graph. Although the fundamental language is different, we show that a ConvAC can be mapped to a TN.

Once a ConvAC is described in the language of a TNs we obtain a substantial advantage that today is lacking from the tool set of machine learning. A convolutional network, and ConvAC in particular, is sometimes described in the language of nodes and edges but those descriptions are merely illustrations since a deep network is not a graph in the graph-theoretic sense. A variety of works study active-learning by associating the input of the network with a graph, e.g. \cite{blum1998combining,blum2004semi,argyriou2005combining,guillory2009label,gu2012towards,dasarathy2015s2}, and \cite{henaff2015deep} show how to construct deep convolutional networks that fit graph-structured Data. \cite{bruna2013spectral} use spectral measures to present an alternative construction of deep neural networks that utilizes the connectivity of graphs, and \cite{niepert2016learning} propose methods for applying convolutional networks to graph-based learning problems. Graph theoretic measures and tools are not widely used for the analysis of the function realized by a deep convolutional network, as for example the notion of an edge-cut and flow have no meaning in the context of the network structure. A TN, on the other hand, is a graph in the graph-theoretic sense and in particular notions of max-flow, min-cut convey important meaning.

This brings us back to the inductive bias mentioned above. Using the fact that max-flow over an edge-cut set describes the expressivity of the network per partition of the input, we obtain upper and lower bounds on this expressivity using a min-cut analysis. Specifically, the ability of a ConvAC to represent correlations between input regions is upper-bounded by a min-cut over all edge-cut sets that separate the corresponding nodes in the associated TN. Furthermore, we show that under a quite general setting of the number of channels per layer the bound is also tight. This kind of result enables one to avoid bottle-necks and adequately tailor the design of a network architecture through application of prior knowledge. Our results are theoretically proven for a deep ConvAC architecture, and their applicability to a conventional deep convolutional network architecture (ConvNet) which involves ReLU activations and max pooling is demonstrated through experiments.

Generally, the bounds we derive connect the inductive bias to the number of channels in each layer, and imply how these should be optimally set in order to satisfy given prior knowledge on the task at hand.  Some empirical reasoning regarding the influence of the number of channels has been suggested (e.g. \cite{szegedy2016rethinking}), mainly regarding the issue of bottle-necks which is naturally explained via our theoretical analysis below. Those bounds and insights on the architectural design of a deep network are new to the machine learning literature, yet rely on known bounds on TNs (albeit only recently) in the physics literature --- those are known as `quantum min-cut max-flow' bounds introduced by \cite{cui2016quantum}. The mapping we present between ConvACs to TNs opens many new possibilities for the use of graph-theory in deep networks, where min-cut analysis could be just the beginning. Additionally, the connections we derive to quantum entanglement may open the door to further well established physical insights regarding correlation structures modeled by deep networks.

The connections between Physics and deep neural networks cover a spectrum of contributions, among which are \cite{beny2013deep} who discussed similarities between deep learning and the renormalization group (RG),   \cite{mehta2014exact} who connected RG to deep learning architectures based on Restricted Boltzmann Machines (RBMs), and \cite{lin2016does} who related basic physical properties such as symmetry, locality and others to the operation of neural networks. The use of TNs in machine learning has appeared in an empirical context where \cite{stoudenmire2016supervised} trained a matrix product state (MPS) TN architecture to preform supervised learning tasks on the MNIST data-set (\cite{lecun1998mnist}) using a sweeping optimization method inspired by the density matrix renormalization group (DMRG) algorithm (\cite{white1992density}), renowned in the numerical physics community for its ability to obtain a good approximation for the physical attributes of many-body systems. Additionally, There is a growing interest in the physics community in RBM based forms for a variational many-body wave functions (\cite{carleo2017solving,torlai2016learning,deng2016exact,huang2017accelerated}). \cite{chen2017equivalence} have presented a theoretical mapping between RBMs and TNs in the form of MPS and the resembling two-dimensional projected entangled pair state (PEPS), with which they propose to remove redundancies in RBMs in one direction by using a canonical representation of the MPS, and more compactly represent a TN wave function by an RBM in the other direction. Moreover, this construction allows them to connect the entanglement bounds of the Tensor Network state to the expressiveness of the corresponding RBM.

The remainder of this paper is organized as follows. In sec.~\ref{sec:prelim} we provide preliminary background in the field of tensor analysis, and present the ConvAC architecture. In sec.~\ref{sec:PhysIntro} we establish the formal connection between the function realized by a ConvAC and a many-body quantum wave function. In sec.~\ref{sec:CorrelationEntanglement} we present the concept of quantum entanglement measures, describe correlations in the context of machine learning and use the connection above to define quantitative correlation measures for a convolutional network. In sec.~\ref{sec:TensorNetworks} we provide an introduction to TNs and describe the tensor decompositions that are used in the analysis of the ConvAC architecture. In sec.~\ref{sec:translations} we construct the TN architecture that is equivalent to a ConvAC. Our analysis of the effect of the number of channels on the correlations modeled by a ConvAC is given in sec.~\ref{sec:mincutclaim}, followed by experiments empirically extending our findings to ConvNets in sec.~\ref{sec:experiments}.  Finally, sec.~\ref{sec:discussion} concludes.

\section{Preliminaries} \label{sec:prelim}

The analyses of ConvACs and TNs that are carried out in this paper rely on concepts and results from the field of tensor analysis. We cover below in sec.~\ref{sec:prelim:TensorAnalysis} basic background on tensors, tensor products, and tensor matricization. Then, we describe in sec.~\ref{sec:prelim:ConvAC} ConvACs which form a family of convolutional networks described using the language of tensors. ConvACs can be simply thought of as regular convolutional networks (ConvNets), but with linear activations and product pooling layers, instead of the more common non-linear activations (e.g. ReLU) and average/max pooling.

\subsection{Concepts from Tensor Analysis} \label{sec:prelim:TensorAnalysis}

We provide below minimal background on tensors which is sufficient to follow the definitions and arguments of this paper. For a comprehensive introduction it is worthwhile to refer to \cite{Hackbusch-book}.
The core concept in tensor analysis is a \emph{tensor}, which may be thought of as a multi-dimensional array.
The \emph{order} of a tensor is defined to be the number of indexing entries in the array, which are referred to as \emph{modes}.
The \emph{dimension} of a tensor in a particular mode is defined as the number of values that may be taken by the index in that mode. If $\A$ is a tensor of order $N$ and dimension $M_i$ in each mode $i\in[N]:=\{1,\ldots,N\}$, the space of all configurations it can take is denoted, quite naturally, by $\R^{M_1{\times\cdots\times}M_N}$.

An important concept we will make use of is \emph{matricization}, which is essentially the rearrangement of a tensor as a matrix.
Suppose $\A$ is a tensor of order $N$ and dimension $M_i$ in each mode $i\in[N]$, and let $(I,J)$ be a partition of $[N]$, \ie~$I$ and~$J$ are disjoint subsets of $[N]$ whose union gives~$[N]$.
We may write $I=\{i_1,\ldots,i_{\abs{I}}\}$ where $i_1<\cdots<i_{\abs{I}}$, and similarly $J=\{j_1,\ldots,j_{\abs{J}}\}$ where $j_1<\cdots<j_{\abs{J}}$.
The \emph{matricization of $\A$ \wrt~the partition $(I,J)$}, denoted $\mat{\A}_{I,J}$, is the $\prod_{t=1}^{\abs{I}}M_{i_t}$-by-$\prod_{t=1}^{\abs{J}}M_{j_t}$ matrix holding the entries of $\A$ such that $\A_{d_1{\ldots}d_N}$ is placed in row index $1+\sum_{t=1}^{\abs{I}}(d_{i_t}-1)\prod_{t'=t+1}^{\abs{I}}M_{i_{t'}}$ and column index $1+\sum_{t=1}^{\abs{J}}(d_{j_t}-1)\prod_{t'=t+1}^{\abs{J}}M_{j_{t'}}$.

A fundamental operator in tensor analysis is the \emph{tensor product}, which we denote by $\otimes$.
It is an operator that intakes two tensors $\A\in\R^{M_1{\times\cdots\times}M_P}$ and $\B\in\R^{M_{P+1}{\times\cdots\times}M_{P+Q}}$ (orders $P$ and $Q$ respectively), and returns a tensor $\A\otimes\B\in\R^{M_1{\times\cdots\times}M_{P+Q}}$ (order $P+Q$) defined by: $(\A\otimes\B)_{d_1{\ldots}d_{P+Q}}=\A_{d_1{\ldots}d_P}\cdot\B_{d_{P+1}{\ldots}d_{P+Q}}$.
Notice that in the case $P=Q=1$, the tensor product reduces to the standard outer product between vectors, \ie~if $\uu\in\R^{M_1}$ and $\vv\in\R^{M_2}$, then $K=\uu\otimes\vv$ is no other than the rank-$1$ matrix $K=\uu\vv^\top\in\R^{M_1{\times}M_2}$, whose entries hold the value: $K_{ij}=u_iv_j$. A generalization to the tensor product of $N$ vectors $\vv^{(j)}\in\R^{M_j}$ for $j\in[N]$, results in an order $N$ tensor $\A_{d_1...d_N}^{\textrm{(rank 1)}}=\vv^{(1)}\otimes\cdots\otimes\vv^{(N)}$, whose entries hold the values:
\be
\A_{d_1...d_N}^{\textrm{(rank 1)}}=\prod_{j=1}^{N} v_{d_j}^{(j)}.
\label{eq:rank1tensor}
\ee
Tensors of this form are regarded as having \emph{rank-$1$} (assuming $\vv^{(j)} \neq 0~~\forall j$).


\subsection{Convolutional Arithmetic Circuits } \label{sec:prelim:ConvAC}
\begin{figure}
\centering
\includegraphics[width=\linewidth]{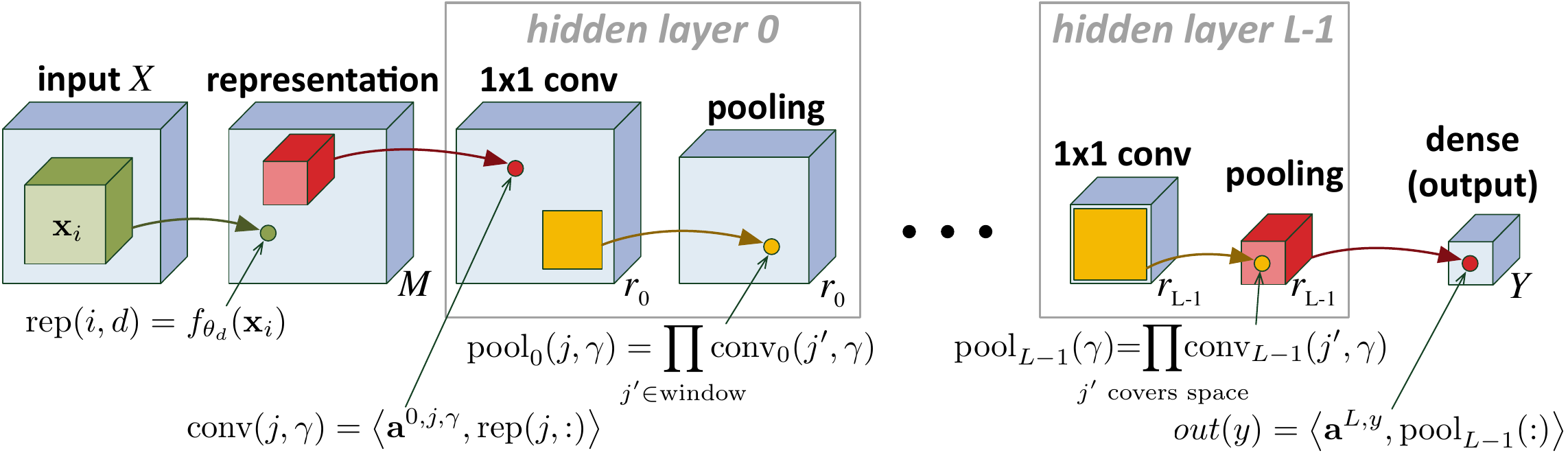}
\caption{The original Convolutional Arithmetic Circuits network as presented by \cite{cohen2016expressive}.}
\label{fig:original_convac}
\end{figure}

Our construction will relate to the convolutional arithmetic circuit (ConvAC) architecture introduced by \cite{cohen2016expressive}. The ConvAC network, illustrated in fig.~\ref{fig:original_convac}, is a deep convolutional network that operates exactly as a regular ConvNet, only with linear activations and product pooling layers (which introduce the non-linearity) instead of the more common non-linear activations (e.g. ReLU) and average/max pooling. From an empirical perspective, ConvACs work well in many practical settings, e.g.
for optimal classification with missing data (\cite{\tmm}), and for compressed networks (\cite{cohen2016deep}). Furthermore, their underlying operations lend themselves to mathematical analyses based on measure theory and tensor analysis --- the depth efficiency of deep convolutional networks was shown using this framework. Importantly, through the concept of generalized tensor decompositions, ConvACs can be transformed to standard ConvNets
with ReLU activations and average/max pooling, which laid the foundation for extending its proof methodologies
to general ConvNets (\cite{cohen2016convolutional}). This deep learning architecture was chosen for our analysis below due to its underlying tensorial structure which resembles the quantum many-body wave function, as we show in sec.~\ref{sec:PhysIntro}. 

In the ConvAC architecture (see fig.~\ref{fig:original_convac}), each point in the input space of the network, denoted by $X=(\x_1,\ldots,\x_N)$,
is represented as an $N$-length sequence of $s$-dimensional vectors $\x_1,\ldots,\x_N \in \R^s$. $X$ is typically
thought of as an image, where each $\x_j$ corresponds to a local patch from that image and $s$ is the number of pixels in each patch. The first layer of the
network is referred to as the representation layer, which involves the application of $M$ representation functions
$f_{\theta_1},\ldots,f_{\theta_M}:\R^s \to \R$ on each local patch $\x_j$, giving rise to $M$ feature maps.
Under the common setting, where the representation functions are selected to
be $f_{\theta_d}(\x) = \sigma(\w_d^T\x + b_d)$ for some point-wise activation $\sigma(\cdot)$ and
parameterized by $\theta_d = (\w_d,b_d) \in \R^s \times \R$, the representation layer reduces to the
standard convolutional layer. Other possibilities, e.g. gaussian functions with diagonal covariances,
have also been considered in \cite{cohen2016expressive}. Following the representation layer, are
hidden layers indexed by $l=0,\ldots,L-1$, each beginning with a $1\times1$ \emph{conv} operator,
which is just an $r_{l-1}\times1\times1$ convolutional layer with $r_{l-1}$ input channels and $r_l$
output channels. Following each \emph{conv} layer is a spatial pooling,
that takes products of non-overlapping two-dimensional windows covering the output of the previous layer,
where for $l=L-1$ the pooling window is the size of the entire spatial dimension (i.e. global pooling),
reducing its output's shape to a $r_{L-1}\times 1 \times 1$, i.e. an $r_{L-1}$-dimensional vector.
The final $L$ layer maps this vector with a dense linear layer into the $Y$ network outputs, denoted by
$\h_y(\x_1,\ldots,\x_N)$, representing score functions classifying each $X$ to one of the classes through:
$y^* = \argmax_y \h_y(\x_1,\ldots,\x_N)$. As shown in \cite{cohen2016expressive}, these functions have the
following form:
\begin{align}\label{eq:convac}
\h_y(\x_1,\ldots,\x_N) = \sum_{d_1,\ldots,d_N=1}^M \A^y_{d_1...d_N} \prod_{j=1}^N f_{\theta_{d_j}}(\x_j),
\end{align}
where $\A^y$, called the \emph{convolutional weights tensor}, is a tensor of order $N$ and dimension $M$ in each mode, with entries given by polynomials in the network's convolutional
weights.

Several decompositions of the convolutional weights tensor were introduced in previous works. Any tensor can be expressed as a sum of rank-$1$ tensors, and the decomposition of $\A$ in such a manner is referred to as the CANDECOMP/PARAFAC decomposition of $\A$, or in short the \emph{CP decomposition}. This corresponds to a network such as depicted in fig.~ \ref{fig:original_convac} with one hidden layer collapsing the entire spatial structure through global pooling. Another decomposition which was shown to be universal, is the Hierarchical Tucker tensor decomposition which we refer to in short as \emph{HT decomposition}. A restricted version of this decomposition (which includes same channel pooling; to be presented more thoroughly below), corresponds to a network such as depicted in fig.~ \ref{fig:original_convac} with $L=\log_{4}N$ hidden layers which pool over size-$2\times 2$ windows. This deep convolutional network was shown to have an exponential advantage in expressiveness over the shallow one realizing the CP decomposition (\cite{cohen2016expressive}) (our analysis below reproduces this result as a by-product, see sec.~\ref{sec:mincutclaim:implications_depth}). We will discuss such tensors decompositions extensively in sec.~\ref{sec:TensorNetworks:decomp} as we will tie them to Tensor Networks, a tool used by physicists when describing many-body quantum wave functions, leading us to the main results of this work.

Finally, it is instructive for our purposes to define $N$ vectors $\vv^{(j)}\in\R^{M}$ for $j\in[N]$ that hold the values $v_{d}^{(j)}=f_{\theta_{d}}(\x_j)$ for $d\in[M]$. This construction implies that the result of applying the $d^{th}$ representation function on the $j^{th}$ image patch is stored in the $d^{th}$  entry of $\vv^{(j)}$. Using the form of the rank $1$ tensor presented in eq.~\ref{eq:rank1tensor}, the score functions may then be written as:
\begin{align}\label{eq:convacrank1}
\h_y(\x_1,\ldots,\x_N) = \sum_{d_1,\ldots,d_N=1}^M \A^y_{d_1...d_N} \A_{d_1...d_N}^{\textrm{(rank 1)}}(\x_1,\ldots,\x_N),
\end{align}
where as described above, $\A_{d_1...d_N}^{\textrm{(rank 1)}}(\x_1,\ldots,\x_N)$ is obtained by applying the representation functions on the input patches  and  $\A^y_{d_1...d_N}$ represents the network's convolutional weights.

In a way, ConvACs form a bridge between ConvNets and TNs as they are described using the language of tensors, similarly to TNs. The tensor language underlying ConvACs is based on rank decompositions through linear combinations of outer-products of low-order tensors (see sec.~\ref{sec:TensorNetworks:decomp}) whereas TNs are described by contractions over low-order tensors (sec.~\ref{sec:TensorNetworks:intro}). Before we can make this connection, we describe below the quantum mechanical language necessary to map between the two worlds, and show a structural equivalence between ConvACs and many-body wave functions.

\section{Quantum Wave Functions and Convolutional Networks} \label{sec:PhysIntro}

When describing the quantum mechanical properties of a system composed of many interacting particles, referred to as a \emph{many-body} quantum system, physicists are required to employ functions which are able to express an elaborate relation between many inputs to an output. Similarly, machine learning tasks such as supervised or unsupervised learning, require functions with the ability to express a complex relation between many inputs, e.g. many pixels in an image, to an output. In this section, we will formulate this analogy.
After a short introduction to the notation used by physicists when describing quantum mechanical properties of a many-body system, we show how the function realized by a ConvAC, given in eqs.~\ref{eq:convac} and~\ref{eq:convacrank1}, is mathematically equivalent to a quantum wave function of $N$ particles. This construction, which constitutes a solid structural connection between the two seemingly unconnected fields of machine learning and quantum physics, is enabled via the tensorial description of a deep convolutional network that is brought forth by the ConvAC.
We follow relevant derivations in \cite{preskill1998lecture} and refer the interested reader to \cite{hall2013quantum} for  a comprehensive mathematical introduction to quantum mechanics.

\subsection{The Quantum Many-Body Wave Function} \label{sec:PhysIntro:SingleParticle}
A state of a system, which is a complete description of a physical system, is given in quantum mechanics as a \emph{ray} in a Hilbert space (to be defined below). Relevant Hilbert spaces in quantum mechanics are vector spaces over the complex numbers. We will restrict our discussion to vector spaces over $\R$, as the properties related to complex numbers are not required for our analysis and do not affect it. Physicists denote such vectors in the `ket' notation, in which a vector $\psi$ is denoted by: $\left|\psi\right\rangle \in{\cal H}$. The Hilbert space ${\cal H}$ has an inner product denoted by $\braket{\phi}{\psi}$, that maps a pair of two vectors in ${\cal H}$ to a scalar. This inner product operation is also referred to as `projecting $\ket{\psi}$ onto $\ket{\phi}$'. A ray is an equivalence class of vectors that differ by multiplication by a nonzero scalar. For any nonzero ray, a representative of the class, $\ket{\psi}$, is conventionally chosen to have a unit norm: $\braket{\psi}{\psi}=1$. A `bra' notation $\bra{\phi}$, is used for the `dual vector' which formally is a linear mapping between vectors to scalars defined as $\ket{\psi}\mapsto\braket{\phi}{\psi}$. We can intuitively think of a `ket' as a column vector and `bra' as a row vector.

Relevant Hilbert spaces can be infinite dimensional or finite dimensional. We will limit our discussion to quantum states which reside in finite dimensional Hilbert spaces, as eventually these will be at the heart of our analogy to convolutional networks. Besides being of interest to us, these spaces are extensively investigated in the physics community as well. For example, the spin component of a spinful particle's wave function resides in a finite dimensional Hilbert space. One can represent a general state $\ket{\psi}\in \cal H$, where $\textrm{dim}({\cal H})=M$, as a linear combination of some orthonormal basis vectors:
\be
\ket{\psi}=\sum_{d=1}^{M}v_d\ket{\psi_d},
\label{eq:singleWf}
\ee
where $\vv\in\mathbb{R}^{M}$ is the vector of coefficients compatible with the basis $\{\ket{\psi_d}\}_{d=1}^{M}$ of $\cal H$, each entry of which can be calculated by the projection: $v_d=\braket{\psi_d}{\psi}$.	

The extension to the case of $N$ particles, each with a wave function residing in a local finite dimensional Hilbert space ${\cal{H}}_j$ for $j\in[N]$ (e.g. $N$ spinful particles), is readily available through the tool of a tensor product. In order to define a Hilbert space which is the tensor product of the local Hilbert spaces: ${\cal H}:=\otimes_{j=1}^{N}{\cal{H}}_j$, we will specify its scalar product. Denote the scalar product in ${\cal{H}}_j$ by $\braket{\cdot}{\cdot}_j$, then the scalar product of the finite dimensional Hilbert space ${\cal H}=\otimes_{j=1}^{N}{\cal{H}}_j$ between $\ket{\psi}:=\otimes_{j=1}^{N}\ket{\psi^{(j)}}\in{\cal H}$ and $\ket{\phi}:=\otimes_{j=1}^{N}\ket{\phi^{(j)}}\in{\cal H} $ is defined by: $\braket{\phi}{\psi}:=\prod_{j=1}^{N}\braket{\phi^{(j)}}{\psi^{(j)}}_j~,~\forall \ket{\psi^{(j)}},\ket{\phi^{(j)}}\in {\cal{H}}_j$.

For simplicity, we set the dimensions of the local Hilbert spaces ${\cal{H}}_j$ to be equal for all $j$, i.e. $\forall j: \textrm{dim}({\cal{H}}_j)=M$. Physically, this means that the particles have the same spin, e.g. for $N$  electrons (spin $1/2$), $M=2$. Denoting the orthonormal basis of the local Hilbert space by $\{\ket{\psi_{d}}\}_{d=1}^{M}$, the many-body quantum wave function $\left|\psi\right\rangle \in{\cal H}=\otimes_{j=1}^{N}{\cal{H}}_j$ can be written as:
\be
\left|\psi\right\rangle =\sum_{d_{1}...d_{N}=1}^{M}\mathcal{A}_{d_{1}...d_{N}}\left|\psi_{d_1}\right\rangle \otimes\cdots\otimes\left|\psi_{d_N}\right\rangle ,
\label{eq:manyWf}
\ee
where $\left|\psi_{d_1}\right\rangle \otimes\cdots\otimes\left|\psi_{d_N}\right\rangle $
is a basis vector of the $M^{N}$ dimensional Hilbert space ${\cal H}$,
and $\mathcal{A}_{d_{1}...d_{N}}$ is the tensor holding the corresponding
coefficients.


\subsection{Equivalence to a Convolutional Network } \label{sec:PhysIntro:EquivalenceWFCAC}

We will tie between the function realized by a ConvAC, given in eqs.~\ref{eq:convac} and~\ref{eq:convacrank1}, and the many-body quantum wave function given in eq.~\ref{eq:manyWf}. First, we consider a special case of $N$ particles  which exhibit no quantum correlations (to be formulated in sec.~\ref{sec:CorrelationEntanglement} below). The state of such a system is called a \emph{product state}, and can be written down as a single tensor product of local states $\ket{\phi_j}\in {\cal{H}}_j$: $\ket{\psi^{~\textrm{ps}}}=\left|\phi_{1}\right\rangle \otimes\cdots\otimes\left|\phi_{N}\right\rangle$. Let $\{\ket{\psi_{d_j}}\}_{d_j=1}^{M}$  be an orthonormal basis for ${\cal{H}}_j$. By expanding each local state in this basis:
\be
\ket{\phi_{j}}=\sum_{d_j=1}^{M} v^{(j)}_{d_j}\ket{\psi_{d_{j}}},
\ee
the product state assumes a form similar to eq.~\ref{eq:manyWf}:
\be
\ket{\psi^{~\textrm{ps}}}=\sum_{d_{1}...d_{N}=1}^{M}\A^{~\textrm{ps}}_{d_{1}...d_{N}}\left|\psi_{d_1}\right\rangle \otimes\cdots\otimes\left|\psi_{d_N}\right\rangle,
\label{productState}
\ee
with the underlying rank $1$ coefficients tensor $\A^{~\textrm{ps}}_{d_{1}...d_{N}} = \prod_{j=1}^N v^{(j)}_{d_j}$. In a similar construction as one presented in sec~\ref{sec:prelim:ConvAC}, if we compose each local state $\ket{\phi_j}$ s.t. its projection on the local basis vector would equal $v_{d}^{(j)}= \braket{\psi_{d}}{\phi_j}=f_{\theta_{d}}(\x_j)$, then the projection of the many-body quantum
state $\left|\psi\right\rangle $ onto the product state $\ket{\phi}$
is equal to\footnote{The bases were chosen to be orthonormal and the representation functions are only linearly independent and not necessarily orthogonal. Similar to the argument presented in \cite{cohen2017inductive},
as the linear independence of the representation functions implies that the dimension of $span\{f_{\theta_1}{\ldots}f_{\theta_M}\}$ is~$M$, upon transformation of the conv weights in hidden layer~$0$ the overall function~$\h_y$ remains unchanged. }:
\be
\label{eq:equivalence}
\braket{\psi^{~\textrm{ps}}}{\psi}=\sum_{d_{1}...d_{N}=1}^{M}\mathcal{A}_{d_{1}...d_{N}}\prod_{j=1}^{N}f_{\theta_{d_{j}}}\left(\x_{j}\right)=\sum_{d_{1}...d_{N}=1}^{M}\mathcal{A}_{d_{1}...d_{N}}\A^{~\textrm{ps}}_{d_{1}...d_{N}}\left(\x_{1},...,\x_{N}\right),
\ee
reproducing eqs.~\ref{eq:convac} and \ref{eq:convacrank1} for a single class $y$, as $\A^{~\textrm{ps}}_{d_{1}...d_{N}}=\A^{~\textrm{(rank 1)}}_{d_{1}...d_{N}}$ by construction. This result ties between the function realized by a convolutional network to that which a many-body wave function models. Specifically, the tensor holding the convolutional weights is analogous to the coefficients tensor of the many-body wave function, while the input to the convolutional network is analogous to the constructed separable state. In the following sections, we will use this analogy to acquire means of describing and analyzing the expressiveness of the convolutional network via properties of its underlying tensor.

\section{Measures of Entanglement and Correlations} \label{sec:CorrelationEntanglement}

The formal connection between the many-body wave function and the function realized by a ConvAC, given in eq.~\ref{eq:equivalence}, creates an opportunity to employ well-established physical insights and tools for the analysis of convolutional networks. Physicists pay special attention to the inter-particle correlation structure characterizing the many-body wave function, as it has broad implications regarding the physical properties of the examined system. Though this issue has received less attention in the machine learning domain, it can be intuitively understood that correlations characterizing the problem at hand, e.g. the correlations between pixels in a typical image from the data-set, should be taken into consideration when addressing a machine learning problem. We shall see that the demand for `expressiveness' of a function realized by a convolutional network, or equivalently of a many-body wave function, is in fact a demand which relates to the ability of the function to model the relevant intricate correlation structure. In this section, we begin by presenting the means with which physicists quantify correlations, and then move on to discuss how to transfer such means to analyses in the machine learning domain.

\subsection{Measures of Quantum Entanglement } \label{sec:CorrelationEntanglement:quantumentanglement}

Physical correlations are discussed in several related contexts. Here, we will focus on a type of correlation referred to as \emph{quantum entanglement}\footnote{Besides discussing the  quantum entanglement that characterizes the wave function itself, physicists describe related correlations via the concept of quantum operators, which we did not introduce here for conciseness.}.  Consider a partition of the above described system of $N$ particles labeled by integers $[N]:=\{1,\ldots,N \}$, which splits it into two disjoint subsystems $A \cupdot B=[N]$ such that $A = \{a_1, \ldots, a_{|A|}\}$ with $a_1<...<a_{|A|}$ and $B = \{b_1, \ldots, b_{|B|}\}$ with $b_1<...<b_{|B|}$. Let ${\cal{H}} ^{A}$ and ${\cal{H}} ^{B}$ be the Hilbert spaces in which the many-body wave functions of the particles in subsystems $A$ and $B$ reside, respectively, with ${\cal H} ={\cal{H}} ^{A}\otimes{\cal{H}} ^{B}~$\footnote{Actually, ${\cal H} \cong{\cal{H}} ^{A}\otimes{\cal{H}} ^{B}$ with equality obtained upon a permutation of the local spaces that is compliant with the partition $(A,B)$.} . The many-body wave function in eq.~\ref{eq:manyWf} can be now written as:
\be
\ket{\psi}=\sum_{\alpha=1}^{\textrm{dim}({\cal{H}} ^{A})}\sum_{\beta=1}^{\textrm{dim}({\cal{H}} ^{B})}(\mat{{\cal A}}_{A,B})_{\alpha,\beta}\ket{\psi_{\alpha}^{A}}\otimes\ket{\psi_{\beta}^{B}},
\label{eq:BipartitionWf}
\ee
where $\{\ket{\psi_{\alpha}^{A}}\}_{\alpha=1}^{\textrm{dim}({\cal{H}} ^{A})}$ and $\{|\psi_{\beta}^{B}\rangle\}_{\beta=1}^{\textrm{dim}({\cal{H}} ^{B})}$ are bases for ${\cal{H}} ^{A}$ and ${\cal{H}}^{B}$, respectively\footnote{It is possible to write $\ket{\psi_{\alpha}^{A}}=\left|\psi_{d_{a_1}}\right\rangle \otimes\cdots\otimes|\psi_{d_{a_{|A|}}}\rangle $ and $\ket{\psi_{\beta}^{B}}=|\psi_{d_{b_1}}\rangle \otimes\cdots\otimes|\psi_{d_{b_{|B|}}}\rangle $ with some mapping from  $\{a_1,\ldots,a_{|A|}\}$ to $\alpha$ and from $\{b_1,\ldots,b_{|B|}\}$ to $\beta$ which corresponds to the matricization formula given in sec.~\ref{sec:prelim:TensorAnalysis}.}, and $\mat{{\cal A}}_{A,B}$ is the matricization of $\A$ w.r.t. the partition $(A,B)$. Let us denote the maximal rank of $\mat{{\cal A}}_{A,B}$ by $r:=\textrm{min}(\textrm{dim}({\cal{H}} ^{A}),\textrm{dim}({\cal{H}} ^{B}))$. A singular value decomposition on $\mat{{\cal A}}_{A,B}$ results in the following form (also referred to as the Schmidt decomposition):
\be
\left|\psi\right\rangle =\sum_{\alpha=1}^{r}\lambda_\alpha\ket{\phi_{\alpha}^{A}}\otimes\ket{\phi_{\alpha}^{B}},
\label{eq:SchmidtDecomp}
\ee
where $\lambda_1 \geq \cdots \geq \lambda_r $ are the singular values of $\mat{{\cal A}}_{A,B}$ , and $\{\ket{\phi_{\alpha}^{A}}\}_{\alpha=1}^{r}$, $\{\ket{\phi_{\alpha}^{B}}\}_{\alpha=1}^{r}$ are $r$ vectors in new bases for ${\cal{H}} ^{A}$ and ${\cal{H}} ^{B}$, respectively, obtained by the decomposition. It is noteworthy that since $\ket\psi$ is conventionally chosen to be normalized, the singular values uphold $\sum_{\alpha}\abs{\lambda_\alpha}^2=1$, however this constraint can be relaxed for our needs below. Eq.~\ref{eq:SchmidtDecomp} represents the $N$ particle wave function in terms of a tensor product between two disjoint parts of it. Each summand in eq.~\ref{eq:SchmidtDecomp} is a separable state w.r.t the partition $(A,B)$, which is defined as a state that can be written down as a single tensor product: $\ket{\psi^{{\textrm {sp} (A,B)}}}=\ket{\phi^{A}}\otimes\ket{\phi^{B}}$ where $\ket{\phi^{A}}\in{\cal{H}} ^{A}$ and $\ket{\phi^{B}}\in{\cal{H}} ^{B}$. Intuitively, the more correlated these two parts are the more `complicated' the function describing their relation should be. We will now present the formulation of these notions as physicists address them, in terms of quantum entanglement.

Several \emph{measures of entanglement} between subsystems $A$ and $B$ can be defined using the singular values $\lambda_\alpha$. A measure of entanglement is a well-defined concept which we will not present in full here, we refer the interested reader to (\cite{plenio2005introduction}).  Essentially, a measure of entanglement w.r.t the partition $(A,B)$ is a quantity that represents the difference between the state in question to a state separable with respect to the partition $(A,B)$.

The \emph{entanglement entropy} is the most conventional measure of entanglement among physicists, and is defined as $S=-\sum_{\alpha}\abs{\lambda_\alpha}^2\ln\abs{\lambda_\alpha}^2$. The minimal entanglement entropy, $S=0$, is received when the rank of $\mat{{\cal A}}_{A,B}$ is $1$. When $\mat{{\cal A}}_{A,B}$ is fully ranked, the entanglement entropy can obtain its maximal value of $\ln(r)$ (upon normalization of the singular values).
An additional measure of entanglement is the \emph{geometric measure}, defined as the $L^2$ distance of $\ket{\psi}$ from the set of separable states: $\min_{\ket{\psi^{\textrm {sp} (A,B)}}} |\braket{\psi^{\textrm {sp} (A,B)}}{\psi}|^2$ which can be shown (e.g. \cite{cohen2017inductive}) to be: $D=\sqrt{1-\frac{|\lambda_1|^2}{\sum_{\alpha=1}^{r}\abs{\lambda_\alpha}^2}}$.
A final measure of entanglement we mention is the \emph{Schmidt number}, which is simply the rank of $\mat{{\cal A}}_{A,B}$, or the number of its non-zero singular values. All of these measures are minimal for states which are separable w.r.t. the partition $(A,B)$ (also said to be \emph{unentangled} w.r.t. this partition), and increase when the correlation between sub-systems $A$ and $B$ is more complicated. It is noteworthy that a product state is unentangled under any partition.

\subsection{Correlations Modeled by a Convolutional Network } \label{sec:CorrelationEntanglement:MLcorrelation}

\begin{figure}
\centering
\includegraphics[scale=0.28]{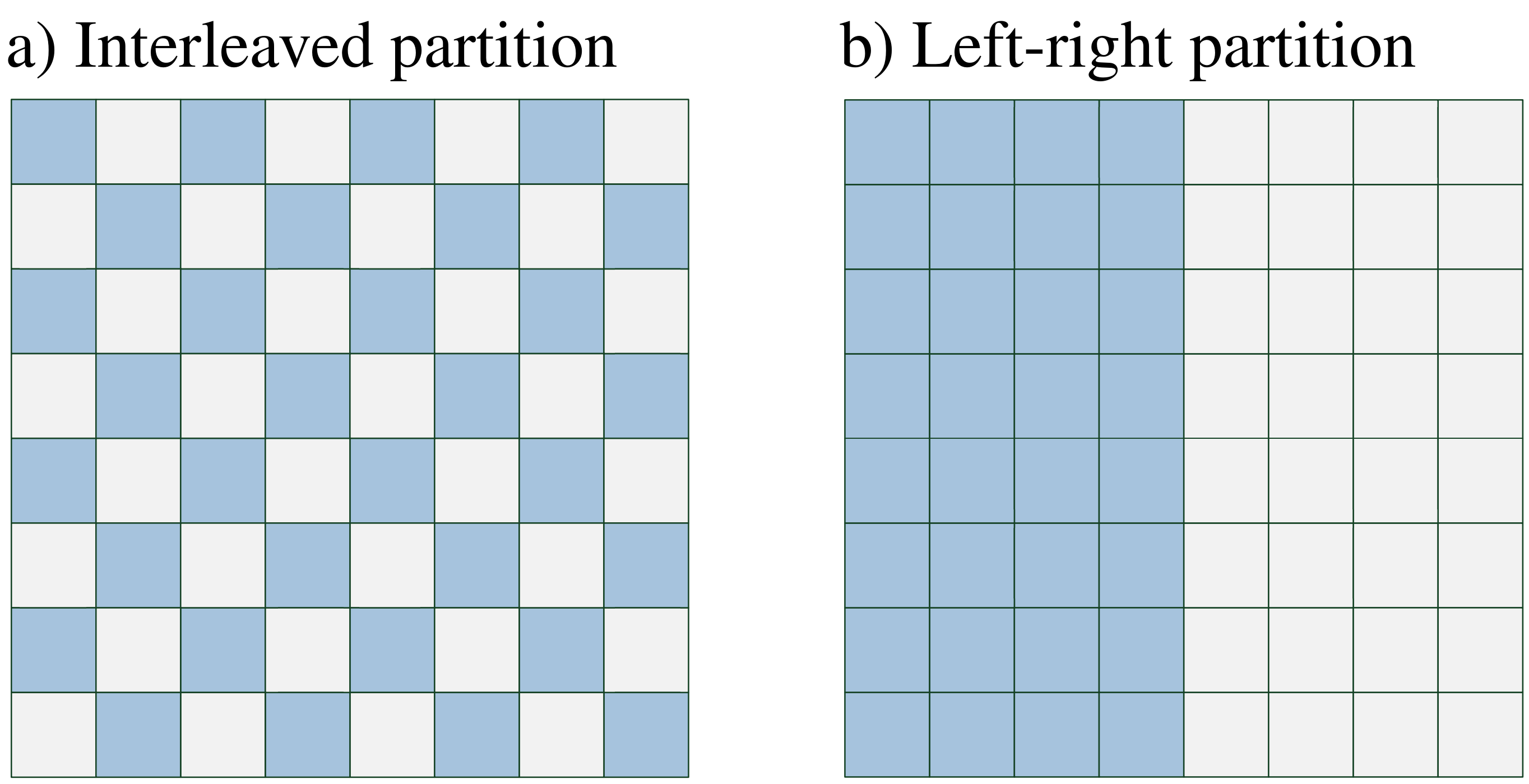}
\caption{\label{fig:partitions} An illustration of a) the interleaved partition and b) the left-right partition for an $8\times 8$ example. The network should support high entanglement measures for the left-right partition if one wishes to model an intricate correlations structure between the two sides of the image (e.g. for face images), and for the interleaved partition if one wishes to do so for neighboring pixels (e.g. for natural images). In sec.~\ref{sec:mincutclaim}, we show how this control over the inductive bias of the convolutional network can be achieved by adequately tailoring the number of channels in each of its layers.}
\end{figure}

The above defined methods of quantifying correlations can now be transferred into the machine learning domain in a straightforward manner, with their role kept fully intact. Utilizing the analogy that was established in sec.~\ref{sec:PhysIntro:EquivalenceWFCAC}, the measures of entanglement provide us an instrument with which we can describe the correlations a convolutional network can model. Specifically, an entanglement measure can be identically defined using the singular values of a matricization of the convolutional weights tensor $\A_{d_{1}...d_{N}}$ according to a  partition of interest.

In accordance with the physical entanglement case, a network which supports high entanglement measures w.r.t a certain partition of its inputs, is in fact able to model a more elaborate correlation structure between the respective two groups of inputs. This notion can be intuitively related to statistical dependence by observing the limit of minimum entanglement with respect to a certain partition. In this case, the overall function can be decomposed into a single multiplication of two disjoint parts of it, which in a statistical setting implies independence between the two parts.

This connection places the analysis of correlations as a key ingredient in the proper harnessing of the inductive bias when constructing a deep network architecture. If one is able to identify a characteristic correlation structure of the inputs, it is clear that constructing the network such that these correlations are given high entanglement measures is advisable. For example, in a natural image, pixels which are closer to each other are more correlated than far away pixels. Therefore, the relevant partition to favor when the inputs are natural images is the interleaved partition, presented in fig.~\hyperref[fig:partitions]{~\ref{fig:partitions}(a)}, which splits the image in a checkerboard manner such that $A$ is composed of the pixels located in blue positions and $B$ is composed of the pixels located in yellow positions. This partition will split many pixels which are correlated to one another. Intuitively, this correlation manifests itself in the following manner: given an image composed only of the pixels in the blue positions, one would still have a good idea of what is in the picture. Thus, for natural images, constructing the network such that it supports exponentially high entanglement measures for the interleaved partition is optimal. Similarly, if the input is composed of symmetric images, such as human face images for example, one expects pixels positioned symmetrically around the middle to be highly correlated. Accordingly, constructing the network such that it supports exponentially high entanglement measures for the left-right partition, shown in fig.~\hyperref[fig:partitions]{~\ref{fig:partitions}(b)}, is advisable in this case.

It is noteworthy, that while the notion of entanglement measures naturally originates in the quantum physics discourse, \cite{cohen2017inductive} have discussed the expressiveness of the ConvAC architecture by using the geometric measure and the Schmidt measure, without explicitly identifying them by these names. They showed that a polynomially sized deep network supports exponentially high matricization ranks (the Schmidt measure) for certain partitions of the input, while being limited to polynomial ranks for others. Moreover, they showed that by altering architectural features of the deep network (specifically the pooling scheme), one can control for which partitions of the input the network will support exponentially high ranks. We are now able to view this approach as motivated by well established physical notions, and proceed to extend it.

To conclude this section, we see that the coefficients or convolutional weights tensor $\mathcal{A}_{d_{1}...d_{N}}$, which has $M^{N}$ entries, fully encapsulates the information regarding the correlations of the many-body quantum wave function or of the function realized by a ConvAC. The curse of dimensionality manifests itself in the exponential dependence on the number of particles or image patches. In a general quantum many-body setting, this renders impractical the ability to investigate or even store a wave function of more than a few dozens of interacting particles. A common tool used to tackle this problem in the physics community is a Tensor Network, which allows utilizing the prior knowledge regarding correlations when attempting to represent an exponentially complicated function with a polynomial amount of resources. In the next section, we present the concept of Tensor Networks, which are a key component in the analysis presented in this work.

\section{Tensor Networks and Tensor Decompositions} \label{sec:TensorNetworks}

In the previous sections, we've seen the integral role that the order $N$ tensor $\A_{d_{1}...d_{N}}$ plays in the simulation of an $N$ particle wave function in the physics domain, and in the classification task of an image with $N$ pixels (or image patches)  in the machine learning domain. Directly storing the entries of a general order $N$ tensor, though very efficient in lookup time (all the entries are stored `waiting' to be called upon), is very costly in storage --- exponential in $N$. A tradeoff can be employed, in which only a polynomial amount of parameters can be kept, while the lookup time increases. Namely, some calculation is to be performed in order to obtain the entries of $\A_{d_{1}...d_{N}}$ \footnote{It is noteworthy, that for a given tensor there is no guarantee that the amount of parameters can be actually reduced. This is dependant on its rank and on how fitting the decomposition is to the tensor's correlations.}. The prevalent approach to the implementation of such a tradeoff in the physics community is called a Tensor Network, and will be thoroughly introduced in the first subsection. This tool allows physicists to construct an architecture that is straightforwardly compliant with the correlations characterizing the state in question.

Observing fig.~\ref{fig:original_convac}, a ConvAC is in effect such a calculation of the tensor $\A$. In other words, the tensor at the heart of a ConvAC calculation is in some sense already `stored' efficiently when keeping a polynomial number of convolutional weights, and one seeks to represent an exponentially complex function with this representation. The common method in machine learning to describe such a representation of a tensor, is called a tensor decomposition. Though both are essentially aimed at efficiently representing a tensor of a high order, Tensor Networks and tensor decompositions differ in some respects. In the second subsection, we will describe these differences and present the tensor decompositions that are employed for modeling the ConvAC.

\subsection{Introduction to Tensor Networks} \label{sec:TensorNetworks:intro}

A \emph{Tensor Network} is formally represented by an underlying undirected graph that has some special attributes, we will elaborate on this formal definition in sec~\ref{sec:mincutclaim}. In the following, we give a more intuitive description of a TN, which is nonetheless exact and somewhat more instructive.
\begin{figure}
\centering
\includegraphics[width=\linewidth]{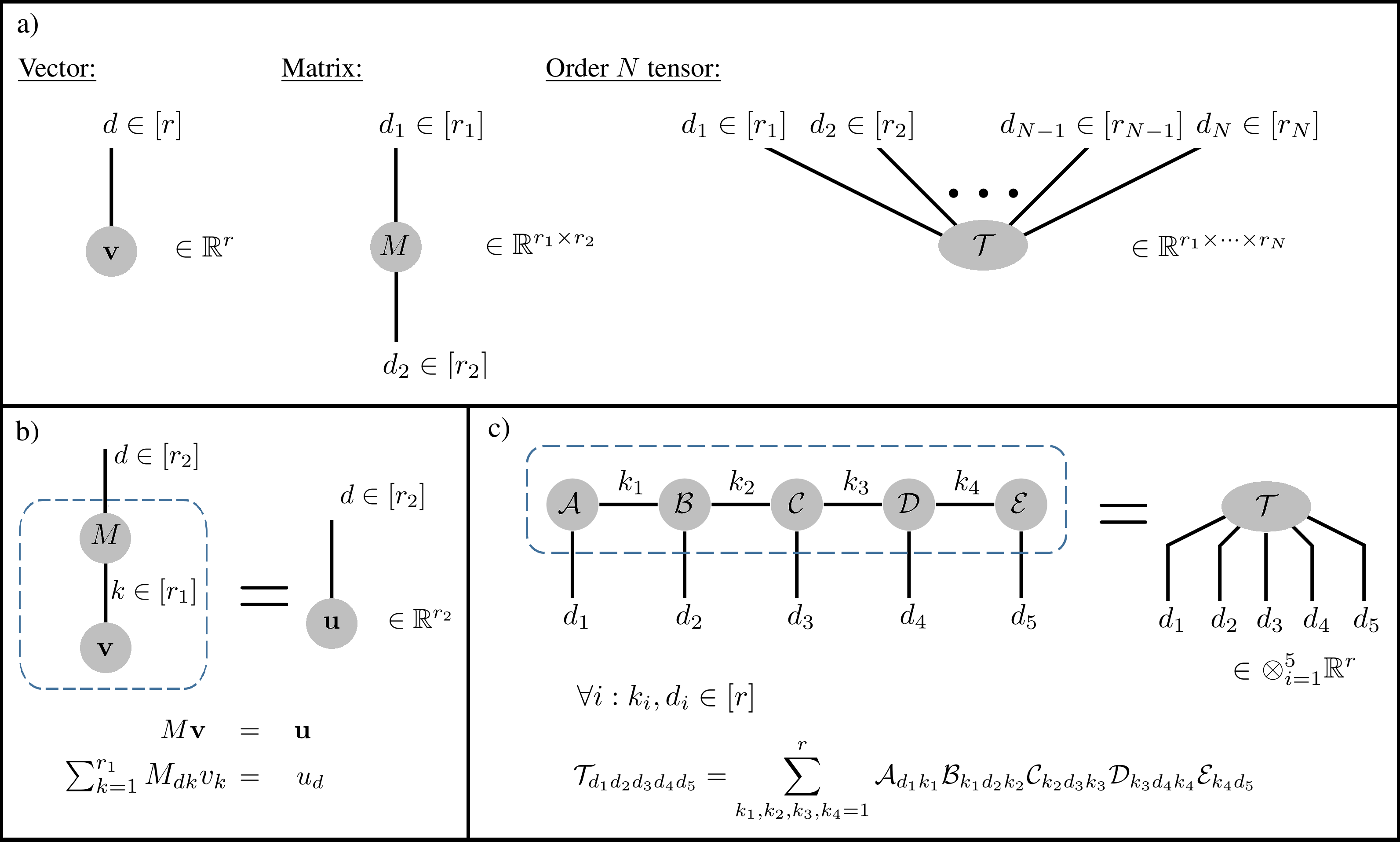}
\caption{A quick introduction to Tensor Networks. a) Tensors in the TN are represented by nodes. The degree of the node corresponds to the order of the tensor represented by it. b) A matrix multiplying a vector in TN notation. The contracted indices are denoted by $k$ and are summed upon. The open indices are denoted by $d$, their number equals the order of the tensor represented by the entire network. All of the indices receive values that range between $1$ and their bond dimension. The contraction is marked by the dashed line. c) A more elaborate example, of a network representing a higher order tensor via contractions over sparsely interconnected lower order tensors. This network is a simple case of a decomposition known as a \emph{tensor train} (\cite{oseledets2011tensor}) in the tensor analysis community or a \emph{matrix product state} (see overview in e.g. \cite{orus2014practical}) in the condensed matter physics community.  }
\label{fig:Sec_5_example}
\end{figure}
The basic building blocks of a TN are tensors, which are represented by nodes in the network. The order of a tensor represented by a node, is equal to its degree --- the number of edges incident to it, also referred to as its legs. Fig.~\hyperref[fig:Sec_5_example]{~\ref{fig:Sec_5_example}(a)} shows three examples: 1) A vector, which is a tensor of order $1$, is represented by a node with one leg. 2) A matrix, which is a tensor of order $2$, is represented by a node with two legs. 3) Accordingly, a tensor of order $N$ is represented in the TN as a node with $N$ legs. In a TN, each edge is associated with a number called the \emph{bond dimension}. The bond dimension assigned to a specific leg of a node, is simply the dimension of the corresponding mode of the tensor represented by this node (see definitions for a mode and its dimension in sec.~\ref{sec:prelim}).

A TN is a collection of such tensors represented by nodes, with edges that can either be connected to a node on one end and loose on the other end or connect between two nodes. Each edge in a TN is represented by an index that runs between $1$ and its bond dimension. An index representing an edge which connects between two tensors is called a contracted index, while an index representing an edge with one loose end is called an open index. The set of contracted indices will be denoted by $K=\{k_1,...,k_{P}\}$ and the set of open indices will be denoted by $D=\{d_1,...,d_{N}\}$. The operation of contracting the network is defined by summation over all of the $P$ contracted indices. The tensor represented by the network, $\A_{d_1...d_{N}}$, is of order $N$,~\ie~its modes correspond to the open indices. Given the entries of the internal tensors of the network, $\A_{d_1...d_{N}}$ can be calculated by contracting the entire network.

An example for a contraction of a simple TN is depicted in fig.~\hyperref[fig:Sec_5_example]{~\ref{fig:Sec_5_example}(b)}. There, a TN corresponding to the operation of multiplying a vector $\vv \in \R^{r_1}$ by a matrix  $M\in \R^{r_2 \times r_1}$ is performed by summing over the only contracted index, $k$. As there is only one open index, $d$, the result of contracting the network is an order $1$ tensor (a vector): $\uu \in \R^{r_2}$ which upholds $\uu = M\vv$. In fig.~\hyperref[fig:Sec_5_example]{~\ref{fig:Sec_5_example}(c)} a somewhat more elaborate example is illustrated, where a TN composed of order $2$ and $3$ tensors represents a tensor of order $5$. This network represents a decomposition known as a \emph{tensor train} (\cite{oseledets2011tensor}) in the tensor analysis community or a \emph{matrix product state} (MPS) (see overview in e.g. \cite{orus2014practical}) in the condensed matter physics community, which arranges order $3$ tensors in such a `train' architecture and allows the representation of an order $N$ tensor with a linear (in $N$) amount of parameters. The MPS exemplifies a typical desired quality of TNs. The decomposition of a higher order tensor into a set of sparsely interconnected lower order tensors, was shown (\cite{oseledets2009breaking};~\cite{ballani2013black}) to greatly diminish effects related to the curse of dimensionality discussed above.

The architectural parameters of the TN are the inter-connectivity of the internal tensors (which tensors are connected to each other) and the bond dimension for each edge (what are the dimensions of each mode). Given tensors of interest, such as the tensor holding the weights of the ConvAC (eq.~\ref{eq:convac}) or the coefficients tensor of the many-body wave function (\ref{eq:manyWf}), it is unclear what TN architecture is best fitting to represent it. This question is in the heart of our analysis below, and we shall see the answer for it is closely tied to the correlations that are to be represented by the tensor.  \cite{stoudenmire2016supervised} preformed supervised learning tasks on the MNIST data-set with an MPS TN, thus demonstrating the ability of a TN to be trained and perform successful machine learning tasks. However, an MPS is a simple network with an extensive associated tool-set (it is highly investigated in the physics community), and it is unclear how a TN with a general connectivity can be trained. By the mapping of a ConvAC to a TN we show below, we effectively migrate a successful and trainable architecture to the language of physicists. Thus, beyond the results obtained in this paper that are beneficial to the machine learning community, the presentation of the ConvAC in the language of TN may allow physicists a new tool-set that takes advantage of the fast-evolving machine learning arsenal. Before we present this translation of ConvACs to TNs, we will familiarize ourselves with their current description --- in the form of tensor decompositions.

\subsection{Tensor Decompositions} \label{sec:TensorNetworks:decomp}

A tensor decomposition is a construction aimed at addressing the same issues as the TN presented above, namely a reduction of an exponentially large tensor to a calculation with a polynomial number of parameters. However, there are some structural differences between the two, which we shall see below are mostly semantic.

In the Tensor Networks approach presented above, the entries of the underlying tensor are obtained by a contraction of indices, which greatly resembles an inner product of vectors. Contrarily, a tensor decomposition is naturally presented in terms of tensor products of lower order tensors, which for vectors are also referred to as outer products. As a simple example of this, we can observe the most famous tensor decomposition, which is for the order two tensor: the singular value decomposition of a matrix. preforming an SVD of a matrix is simply writing it as a sum of outer products between its singular vectors. Explicitly, the decomposition of a rank $r$ matrix $M\in \R^{r_1 \times r_2}$ can be written as: $M=\sum_{i=1}^{r}\lambda_i~\uu^{(i)}\otimes\vv^{(i)}$ where $\lambda_i$ are the singular values and $\{\uu^{(i)}\}_{i=1}^r $ and $\{\vv^{(i)}\}_{i=1}^r$ are $r$ vectors in the appropriate bases of $\R^{r_1}$ and $\R^{r_2}$, respectively. The TN equivalent of an SVD is presented in fig.~\hyperref[fig:CPExample]{~\ref{fig:CPExample}(a)}, which for this case also resembles a natural way of writing the SVD as a multiplication of matrices.

A somewhat less trivial example, which involves a decomposition of an order $N$ tensor, is the CP decomposition. As mentioned in sec. \ref{sec:prelim},  the CP decomposition of the convolutional weights tensor corresponds to a ConvAC depicted in fig.~ \ref{fig:original_convac} with one hidden layer, which collapses the entire spatial structure through global pooling. Explicitly, the CP decomposition  of the order $N$ convolutional weights tensor of a specific class $y$ is a sum of rank-$1$ tensors, each of which is attained by a tensor product of $N$ weights vectors:
\be
\A^y = \sum_{k=1}^K v^y_k\cdot\aaa^{k,1} \otimes \cdots \otimes \aaa^{k,N},
\label{cpformula}
\ee
where $\vv^y\in \R^K, \forall y\in[Y]$ and $\aaa^{k,j}\in\R^M, \forall k\in[K],j\in[N]$.

The deep version of fig.~\ref{fig:original_convac}, where the pooling windows between convolutional layers are of minimal size, corresponds to a specific tensor decomposition of $\A^y$, which is a restricted version of a \emph{hierarchical Tucker decomposition}, referred to in short as the HT decomposition. The restriction is related to the fact that the pooling scheme of the ConvAC architecture presented in  fig.~\ref{fig:original_convac} involves only entries from the same channel, while in the general HT decomposition pooling operations would involve entries from different channels. For brevity of notation, we will present the expressions for a scenario where the input patches are aligned along a one-dimensional line (can also correspond to a one-dimensional signal, e.g. sound or text), and the pooling widows are of size 2. The extension to the two-dimensional case follows quite trivially, and was presented in \cite{cohen2017inductive}. Under the above conditions, the decomposition corresponding to a deep ConvAC can be defined recursively by(\cite{cohen2016expressive}):
\bea
\phi^{1,j,\gamma} &=& \sum_{\alpha=1}^{r_0} a_\alpha^{1,j,\gamma}
\aaa^{0,2j-1,\alpha} \otimes  \aaa^{0,2j,\alpha}
\nonumber \\
&\cdots&
\nonumber\\
\phi^{l,j,\gamma} &=& \sum_{\alpha=1}^{r_{l-1}} a_\alpha^{l,j,\gamma}
\underbrace{\phi^{l-1,2j-1,\alpha}}_{\text{order $2^{l-1}$}} \otimes
\underbrace{\phi^{l-1,2j,\alpha}}_{\text{order $2^{l-1}$}}
\nonumber\\
&\cdots&
\nonumber\\
\A^y &=& \sum_{\alpha=1}^{r_{L-1}} a_\alpha^{L,y}
\underbrace{\phi^{L-1,1,\alpha}}_{\text{order $\frac{N}{2}$}} \otimes
\underbrace{\phi^{L-1,2,\alpha}}_{\text{order $\frac{N}{2}$}}.
\label{eq:ht_decomp}
\eea

The decomposition in eq.~\ref{eq:ht_decomp} recursively constructs the convolutional weights tensor  $\{\A^y\}_{y\in[Y]}$ by assembling vectors $\{\aaa^{0,j,\gamma}\}_{j\in[N],\gamma\in[r_0]}$ into tensors $\{\phi^{l,j,\gamma}\}_{l\in[L-1],j\in[N/2^l],\gamma\in[r_l]}$ in an incremental fashion. This is done in the form of tensor products, which are the natural form for tensor decompositions.
The index $l$ stands for the level in the decomposition, corresponding to the $l^{th}$ layer of the ConvAC network given in fig.~\ref{fig:original_convac}.  $j$ represents the `location' in the feature map of level $l$, and $\gamma$ corresponds to the
individual tensor in level $l$ and location $j$.
The index $r_l$ is referred to as \emph{level-$l$ rank}, and is defined to be the number of tensors in each location of level $l$ (we denote for completeness $r_L:=Y$). In the ConvAC network given in fig.~\ref{fig:original_convac}, $r_l$ is equal to the number of channels in the $l^{th}$ layer --- this will be important in our analysis of the role played by the channel numbers.
The tensor $\phi^{l,j,\gamma}$ has order $2^l$, and we assume for simplicity that $N$~--~the order of $\A^y$, is a power of $2$ (this is merely a technical assumption also made in \cite{Hackbusch-book}, it does not limit the generality of the analysis). The parameters of the decomposition are the final level weights $\{\aaa^{L,y}\in\R^{r_{L-1}}\}_{y\in[Y]}$, the intermediate levels' weights $\{\aaa^{l,j,\gamma} \in \R^{r_{l-1}}\}_{l\in[L-1],j\in[N/2^l],\gamma\in[r_l]}$, and the first level weights $\{\aaa^{0,j,\gamma}\in\R^M\}_{j\in[N],\gamma\in[r_0]}$.

As mentioned above, the ConvAC framework shares many of the same traits as modern ConvNets,
i.e. locality, sharing and pooling and it can be shown to form a universal hypotheses space, exhibit complete depth-efficiency (\cite{cohen2016expressive}), and relate inductive bias to the pattern of pooling (\cite{cohen2017inductive}).
Additionally, through the concept of generalized tensor decompositions, ConvACs can be transformed to standard ConvNets
with ReLU activations and average/max pooling, which laid the foundation for extending its proof methodologies
to general ConvNets (\cite{cohen2016convolutional}). Finally, from an empirical perspective, they tend to work well in many practical settings, e.g.
for optimal classification with missing data (\cite{\tmm}), and for compressed networks (\cite{cohen2016deep}).

To conclude, the ConvAC is a deep convolutional machine learning architecture (given in fig.~\ref{fig:original_convac}), that was described and analyzed thus far by an underlying tensor decomposition.  Having understood the basics of Tensor Networks, a tool used by physicists to efficiently describe their tensor of interest, we will proceed to the presentation the ConvAC in terms of a TN.

\section{A Convolutional Network as a Tensor Network} \label{sec:translations}

Eq.~\ref{eq:equivalence} shows the equivalence between the ConvAC generating function described in eq.~\ref{eq:convacrank1} and the projection of a wave function with an order $N$ coefficients tensor $\mathcal{A}$ onto a product-state with a rank-$1$ coefficient tensor $\mathcal{A}^{\textrm {ps}}$. As is shown in sec.~\ref{sec:TensorNetworks:decomp}, the ConvAC network effectively uses tensor decomposition rules to turn a function with an exponential number of terms to a polynomially sized engine. Tensor networks achieve the same goal of turning a wave function with an exponential number of terms to a more manageable calculation by using a graph whose nodes correspond to low-order tensors and whose edges correspond to modes along which two tensors contract. Although in both cases tensors are involved, the language is different ---  for ConvACs the basic operations are tensor products and for TNs it is contractions. In this section we show how to map a ConvAC into a TN representation, basically describing the ConvAC in terms of contractions and with an underlying graph. The mapping we introduce is not meant to provide an alternative, more efficient, way to represent a ConvAC but rather a way to leverage the fact that TNs are subject to graph-theoretic analyses and begin a migration of TNs-related analyses to the domain of deep neural networks.

\subsection{Tensor Network Construction of a Shallow Convolutional Network} \label{sec:translations:shallow}

In order to construct the TN equivalent of the CP decomposition (eq.~\ref{cpformula}), we define the order $N$ tensor $\delta\in\R^{K\times\cdots\times K}$, referred to as the \emph{$\delta$ tensor}, as follows:
\begin{equation}
\begin{array}{c}
\delta_{d_1...d_N}:=\left\{ \begin{array}{c}
~1,\quad d_1=\cdots=d_N\\
0,\quad  ~~~~~~otherwise
\end{array}\right.,\\
\end{array}\label{eq:deltadef}
\end{equation}
with $d_j\in[K]~\forall j\in[N]$, \ie~its entries are equal to $1$ only on the super-diagonal and are zero otherwise. We shall draw nodes which correspond to such $\delta$ tensors as triangles in the TN, to remind the reader of the restriction given in eq.~\ref{eq:deltadef}. Observing eq.~\ref{cpformula}, let $G\in \R^{Y \times K}$ be the matrix holding the convolutional weight vector of the final layer $\vv^y\in \R^K$ in its $y^{th}$ row and let $A^{(j)}\in \R^{K \times M}$ be the matrix holding the convolutional weights vector $\aaa^{k,j}\in\R^{M}$ in its $k^{th}$ row. One can observe that per class $y$, the $k^{th}$ summand in eq.~\ref{cpformula} is equal to the tensor product of the $N$ vectors residing in the $k^{th}$ rows of all the matrices $A^{(j)}, j\in[N]$, multiplied by a final weight associated with class $y$. Tensors represented by nodes in the TN will have parenthesis in the superscript, which denote labels such as the position $j$ in the above, to differentiate them from `real' indices that must be taken into consideration when contracting the TN. Per convention, such `real' indices will be written in the subscript.

Having defined the above, the TN equivalent of the CP decomposition is illustrated in fig.~\hyperref[fig:CPExample]{~\ref{fig:CPExample}(b)}.  Indeed, though they represent the same exact quantity, the form of eq.~\ref{cpformula}  isn't apparent at a first glance of the network portrayed in fig.~\hyperref[fig:CPExample]{~\ref{fig:CPExample}(b)}. Essentially, the TN equivalent of the CP decomposition involves contractions between the matrices $A^{(j)}$ and the $\delta$ tensor, as can be seen in the expression representing it:
\be
\A_{d_1...d_N}=\sum_{k_1,...,k_{N+1}=1}^{K}\delta_{k_1...k_{N+1}}A^{(1)}_{k_1d_1}\cdots A^{(N)}_{k_Nd_N}G_{yk_{N+1}}.
\label{cpformulatensors}
\ee
\begin{figure}
\centering
\includegraphics[width=\linewidth]{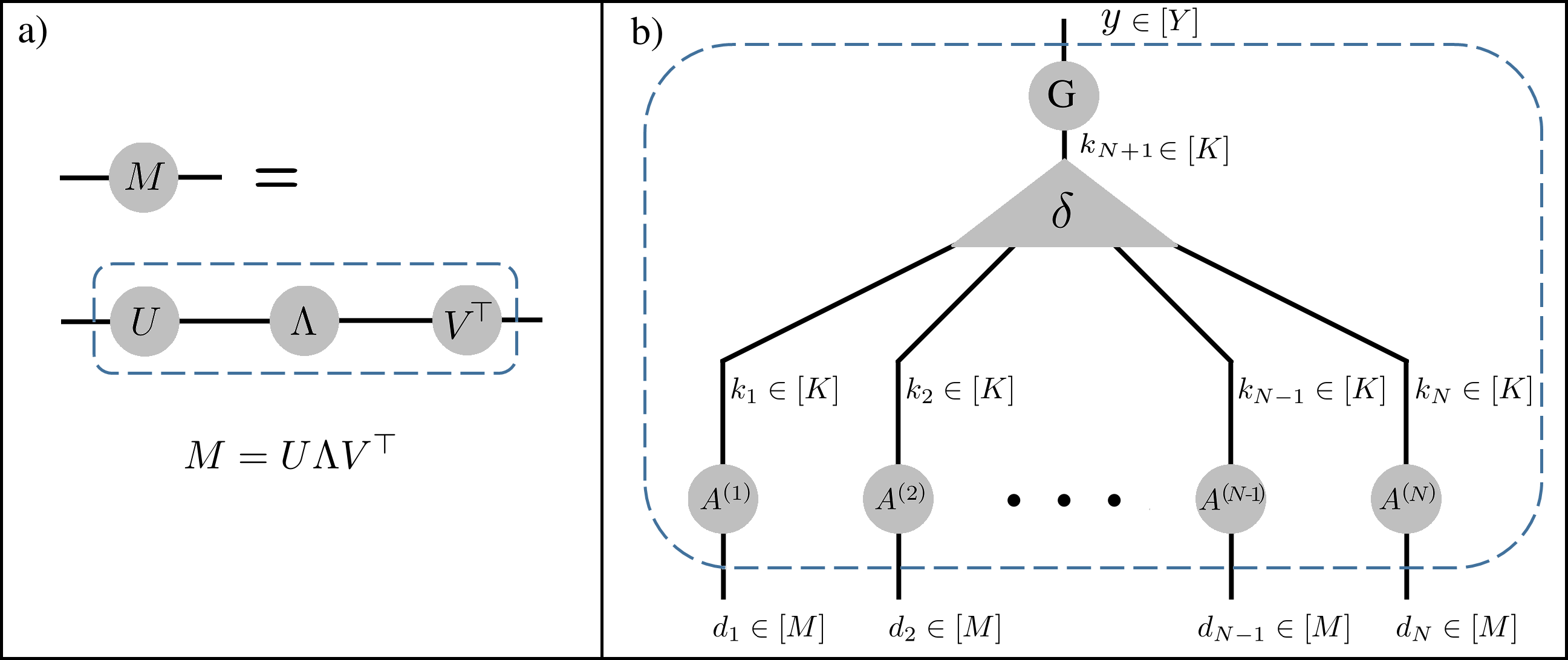}
\caption{ a) A singular value decomposition in a form of a TN. The $\Lambda$ node represents a diagonal matrix and the $U,V$ nodes represent orthogonal matrices. b) The TN equivalent of the CP decomposition. This is a TN representation of the order $N$ weights tensor $\A_{d_1...d_N}$ underlying the calculation of the ConvAC in fig~.\ref{fig:original_convac} in its shallow form, \ie~with one hidden layer followed by a global pooling operation which collapses the feature maps into $Y$ different class scores. The matrices $A^{(j)}$ hold the convolutional weights of the hidden layer and the matrix $G$ holds the weights of the final dense layer. The central $\delta$ tensor effectively enforces the same channel pooling, as can be seen by its form in eq.~\ref{eq:deltadef} and its role in the calculation of this TN given in eq.~\ref{cpformulatensors}.}
\label{fig:CPExample}
\end{figure}
The role of the $\delta$ tensor in eq.~\ref{cpformulatensors} can be observed as `forcing' the $k^{th}$ row of any matrix $A^{(j)}$ to be multiplied only by $k^{th}$ rows of the other matrices which in effect enforces same channel pooling\footnote{If one were to switch the $\delta_{k_1...k_N}$ in eq.~\ref{cpformulatensors} by a general tensor $\G_{k_1...k_N}\in\R^{K\times\cdots\times K}$, a TN equivalent of an additional acclaimed decomposition would be attained, namely the \emph{Tucker decomposition}. Similar to other tensor decompositions, the Tucker decomposition is more commonly given in an outer product form:
$\A = \sum_{k_1,...,k_N=1}^K \G_{k_1...k_N} \aaa^{k_1,1} \otimes \cdots \otimes \aaa^{k_N,N}$.
}.

\subsection{Tensor Network Construction of a Deep Convolutional Network} \label{sec:translations:deep}

\begin{figure}
\centering
\includegraphics[width=\linewidth]{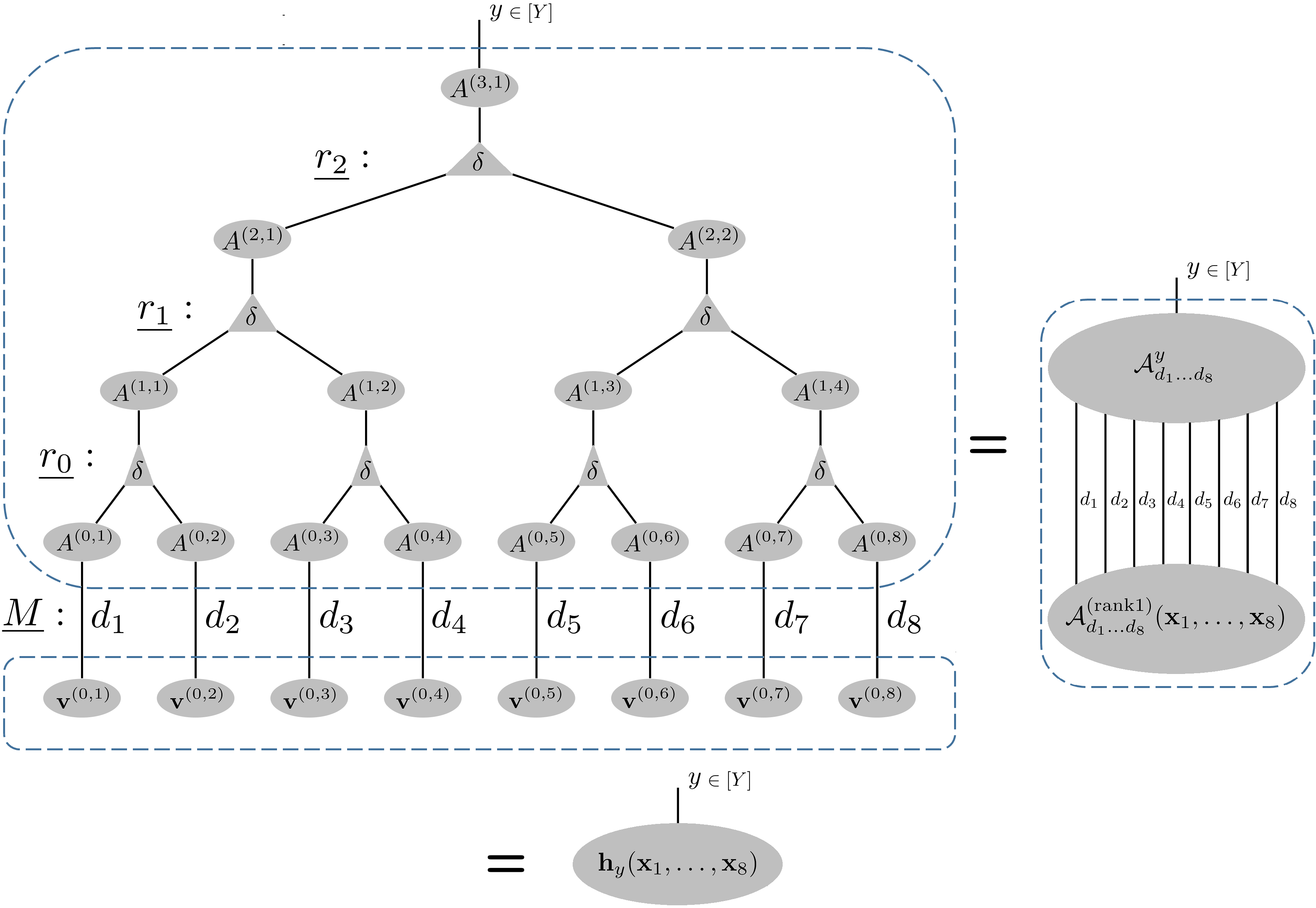}
\caption{The TN equivalent to the HT decomposition with a same channel pooling scheme corresponding to the calculation of a deep ConvAC given in fig.~\ref{fig:original_convac} with $N=8$. Further details in sec.~\ref{sec:translations:deep}.}
\label{fig:HT_TN}
\end{figure}

We describe below a TN corresponding to the deep ConvAC calculation, which is presented in
eq.~\ref{eq:ht_decomp}. The ConvAC calculation is constructed as an inner-product between two tensors:  $\A^y_{d_1...d_N}$ holding the weights of the network and $\A_{d_1...d_N}^{\textrm{(rank 1)}}(\x_1,\ldots,\x_N)$ which is a rank-$1$ tensor holding the $N\cdot M$ values of the representation layer ($M$ representation functions applied on $N$ input patches).

Fig.~\ref{fig:HT_TN} displays in full the TN for an $N=8$ ConvAC calculation. The upper block separated by a dashed line is the TN representing $\A_{d_1...d_8}$ whereas the lower block represents the rank-$1$ inputs tensor. Considering the upper block, it is worth noting that it is not a sketch of a TN but the actual full description complaint with the graph notations described in sec.~\ref{sec:TensorNetworks:intro}.
Accordingly, the two legged nodes represent matrices, where each matrix $A^{(l,j)}\in \R^{r_l\times r_{l-1}}$ (with $r_{-1}:=M$) is constructed such that it holds the convolutional weight vector  $\aaa^{l,j,\gamma}\in \R^{r_{l-1}}, \gamma \in [r_l]$ in its $\gamma^{th}$ row. The triangle node appearing between levels $l-1$ and $l$ represents an order $3$ tensor $\delta\in\R^{r_{l-1}\times r_{l-1}\times r_{l-1}}$, obeying eq.~\ref{eq:deltadef}. The $\delta$ tensor is the element which dictates the same channel pooling in this TN construction.

As mentioned above, the lower block in fig.~\ref{fig:HT_TN}  is the TN representing $\A_{d_1...d_8}^{\textrm{(rank 1)}}(\x_1,\ldots,\x_8)$. This simple TN is merely a single outer product of $N=8$ vectors $\vv^{(0,j)}\in\R^{M},j\in[N]$ composing the representation layer presented in sec.~\ref{sec:prelim:ConvAC}, holding the values $v^{(0,j)}_{d_j}=f_{\theta_{d_j}}(\x_j)$. In compliance with the analogy between the function realized by the ConvAC and the projection of a many-body wave function onto a product state shown in eq.~\ref{eq:equivalence}, the form $\A_{d_1...d_8}^{\textrm{(rank 1)}}$ assumes is exactly the form the coefficients tensor of a product state assumes when represented as a TN. As can be seen in fig.~\ref{fig:HT_TN}, a final contraction of the indices $d_1,...,d_8$ results in the class scores vector calculated by the ConvAC, $\h_y(\x_1,...,\x_8)$.

In appendix~\ref{app:HTrecursive}, we present a recursive definition for the TN representing the a deep ConvAC of a general size (s.t. ${\textrm {log}}_2N\in\N$). As is demonstrated in appendix~\ref{app:HTrecursive}, a contraction of a TN in an incremental fashion, starting with the input layer and moving up the network, exactly reproduces the computation preformed level-by-level along the network given in fig.~\ref{fig:original_convac}. A generalization to a ConvAC with any pooling scheme is straightforward, where the size of the pooling window is translated into the number of edges incident to the $\delta$ tensors from below.

To conclude this section, we have presented a translation of the computation performed by a ConvAC to a TN. The convolutional weights are arranged as matrices (two legged nodes) placed along the network, and the same channel pooling characteristic is made available due to three legged $\delta$ tensors in a deep network, and an $N+1$ legged $\delta$ tensor in a shallow network. Finally, and most importantly for our upcoming analysis, the bond dimension of each level in the TN representing the ConvAC is equal to $r_l$, which is the number of feature maps (\ie~the number of channels) comprising that level in the corresponding ConvAC architecture. As we shall see in the following sections, this last equivalence will allow us to provide prescriptions regarding the number of channels in each level, when attempting to fit the correlations modeled by the network to the input correlations.


\section{How the Number of Channels Affects the Expressiveness of a Deep Network } \label{sec:mincutclaim}

Thus far, we have presented several structural and conceptual bridges between the fields of machine learning and quantum physics. The equivalence between the function realized by a deep convolutional network and the wave function of many correlated particles, shown in sec.~\ref{sec:PhysIntro}, was the natural starting point of our discussion. Though seemingly unrelated, at the heart of these two problems lies the need to express a complicated function over many inputs (pixels or particles), which must model different forms of correlations between them. In sec.~\ref{sec:CorrelationEntanglement} we showed how well-established physical means of quantifying such correlations, namely measures of quantum entanglement, can be readily transferred  and utilized for analyses in the machine learning domain. In sec.~\ref{sec:TensorNetworks} we presented Tensor Networks and tensor decompositions, the common mathematical tools used in each field to efficiently store, manipulate and analyze the complicated functions of interest. The translation of the computation performed by a ConvAC to a Tensor Network shown in sec.~\ref{sec:translations}, is the link that will provide us with the ability to translate insights and results from the domain of quantum physics into novel conclusions in the  domain of machine learning (and specifically deep learning).

As mentioned in sec.~\ref{sec:TensorNetworks}, the architectural parameters of any TN are the inter-connectivity of the internal tensors and the bond dimension of each edge. Since the connectivity of the Tensor Network representing the ConvAC's calculation, presented in fig.~\ref{fig:HT_TN}, is the key to the equivalence to a ConvAC, we will not touch it presently. However, the role of the bond dimensions in each level of this TN, which as we have shown are equal to the number of feature maps (channels) in that level in the corresponding ConvAC architecture, is open for analysis. In this section, we present the main results of our work. By relying on the TN description of the ConvAC, we show how adequately tailoring the number of channels of each layer in the deep ConvAC can enhance its expressiveness  by fitting the form of the function realized by it to given correlations of the input. In this we show how the parameters of the ConvAC can be most efficiently distributed given prior knowledge of the input, which is in fact an alignment of its inductive bias to the task at hand.

Our analysis makes use of the most recent advances in the study of the quantitative connection between quantum entanglement and Tensor Networks. Quantum entanglement has been a key factor in most studies of TNs  performed over the past two decades (\eg\cite{white1992density},\cite{vidal2009entanglement}). More recent works (\cite{calegari2010positivity},\cite{cui2016quantum}), discuss bounds on measures of entanglement in the context of the architecture of a general TN. The enabling key feature in those studies is the fact that a TN is subject to graph-theoretic analysis. This was not known to be the case for ConvACs thus far. In the following, we adapt the work of  (\cite{cui2016quantum}), who introduced the concept of quantum-min-cut, to the TN we constructed in sec.~\ref{sec:translations:deep} with the constrained $\delta$-tensors, and draw a direct relation between the number of channels in each layer of a deep ConvAC and the functions it is able to model.

\subsection{The ConvAC Tensor Network as a Graph} \label{sec:mincutclaim:graph}

\begin{figure}
\centering
\includegraphics[scale=0.33]{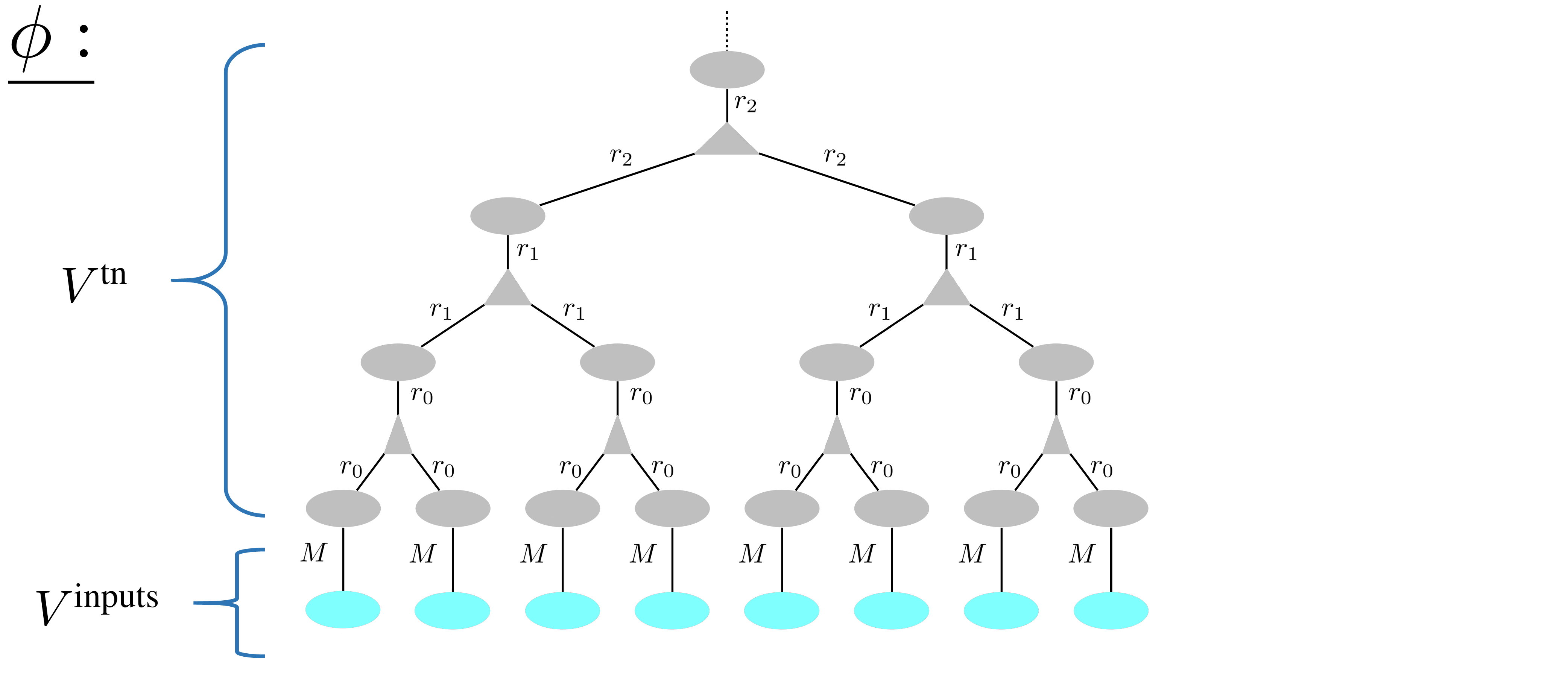}
\caption{The components comprising a `ConvAC-weights TN' $\phi$ that describes the weights tensor $\A^y$ of a ConvAC, are an undirected graph $G(V,E)$ and a bond dimensions function $c$. The bond dimension is specified next to each edge $e\in E$, and is given by the function $c(e)$. As shown in sec.~\ref{sec:translations:deep}, the bond dimension of the edges in each layer of this TN is equal to the number of channels in the corresponding layer in the ConvAC. The node set in the graph $G(V,E)$ presented above decomposes to $V=V^{\textrm {tn}}\cupdot V^{\textrm {inputs}}$, where $V^{\textrm {tn}}$ (grey) are vertices which correspond to tensors in the ConvAC TN and $V^{\textrm {inputs}}$ (blue) are degree $1$ vertices which correspond to the $N$ open edges in the ConvAC TN. The vertices in $V^{\textrm {inputs}}$ are `virtual' --- were added for completeness,  so $G$ can be viewed as a legal graph. The open edge emanating from the top-most tensor (marked by a dashed line) is omitted from the graph, as it does not effect our analysis below --- no flow between any two input groups can pass through it.}
\label{fig:virtual_nodes}
\end{figure}

Convolutional networks, though often presented by graphical schemes that include nodes and edges which represent activation functions and weights, cannot be readily used as well-defined objects in graph theory. A general TN, on the other hand, is naturally represented by an underlying undirected graph describing its connectivity and a function that assigns a bond dimension to each edge. In this section we describe the TN we constructed in sec.~\ref{sec:translations:deep} in graph-theoretic terminology.

The ability to represent a deep convolutional network (ConvAC) as a `legal' graph, is a key accomplishment that the Tensor Networks tool brings forth. Our main results rely on this graph-theoretic description and tie the expressiveness of a ConvAC to a minimal cut in the graph characterizing it, via the connection to quantum entanglement measures. This is in fact a utilization of the `Quantum min-cut max-flow' concept presented by \cite{cui2016quantum}. Essentially, the quantum max-flow between $A$ and $B$ is a measure of the ability of the TN to model correlations between $A$ and $B$, and the quantum min-cut is a quantity that bounds this ability and can be directly inferred from the graph defining it --- that of the corresponding TN.

As noted in sec.~\ref{sec:translations:deep}, the TN that describes the calculation of a ConvAC, \ie~a network that receives as inputs the $M$ representation channels and outputs $Y$ labels, is the result of an inner-product of two tensors $\A^y_{d_1...d_N}$ holding the weights of the network and $\A_{d_1...d_N}^{\textrm{(rank 1)}}(\x_1,\ldots,\x_N)$ which is a rank-$1$ tensor holding the `input' of $N$ vectors $\vv^{(0,j)}\in\R^{M},j\in[N]$ composing the representation layer. In this section we focus on the TN that describes $\A^y_{d_1...d_N}$, which is the upper block of fig.~\ref{fig:HT_TN} and is also reproduced as a stand-alone TN in fig.~\ref{fig:virtual_nodes}, referred to as the `ConvAC-weights TN' and denoted by $\phi$. The TN $\phi$ has $N$ open edges with bond dimension $M$ that are to be contracted with the inputs $\vv^{(0,j)}\in\R^{M},\ j\in[N]$ and one open edge with bond dimension $Y$ representing the  values $\A^y_{d_1...d_N},\ y\in[Y]$ upon such a contraction, as is shown in fig~\ref{fig:HT_TN}.

To turn $\phi$ into a graph we do the following. First, we remove the open edge associated with the output. As our analysis is going to be based on flow between groups of input vertices, no flow can pass through that open edge therefore removing it does not influence our analysis. Second, we add $N$ virtual vertices incident to the open edges associated with the input. Those virtual vertices are the only vertices whose degree is equal to $1$ (see fig.~\ref{fig:virtual_nodes}). The TN $\phi$ is now described below using graph terminology:
\begin{itemize}
\vspace{-2mm}
\item An undirected graph $G(V,E)$, with a set of vertices $V$ and a set of edges $E$. The set of nodes is divided into two subsets $V=V^{\textrm {tn}}\cupdot V^{\textrm {inputs}}$, where $V^{\textrm {inputs}}$ are the $N$ degree-$1$ virtual vertices and $V^{\textrm {tn}}$ corresponds to tensors of the TN.
\vspace{-2mm}
\item A function $c~:~E\rightarrow\N$, associating a number $r\in\N$ with each edge in the graph, that equals to the bond dimension of the corresponding edge in the TN.
\vspace{-2mm}

\end{itemize}
\bigskip

Having described the object representing the ConvAC-weights TN $\phi$, let us define an edge-cut set with respect to a partition of the $N$ nodes of $V^{\textrm {inputs}}$, and then introduce a cut weight associated with such a set. An edge-cut set with respect to the partition $V^A\cupdot V^B=V^{\textrm {inputs}}$ is a set of edges $C$ s.t. there exists a partition $\tilde{V}^A\cupdot\tilde{V}^B=V$ with $V^A \subset \tilde{V}^A~,~ V^B \subset \tilde{V}^B$, and $C = \{(u,v) \in E: u \in \tilde{V}^A, v \in \tilde{V}^B \}$. We note that this is a regular definition of an edge-cut set in a graph $G$ with respect to the partition of vertices $(V^A,V^B)$. Let $C=\{e_1,...,e_{|C|}\}$ be such a set, we define its multiplicative cut weight as:
\be
W_C=\prod\nolimits_{i=1}^{\abs{C}} c(e_i).
\label{CutWeight}
\ee
The weight definition given in eq.~\ref{CutWeight} is simply a multiplication of the bond dimensions of all the edges in a cut. Fig.~\ref{fig:mincut} shows a pictorial demonstration of this weight definition, which is at the center of our results to come. In the following section, we use a max-flow / min-cut analysis on $\phi$ to obtain new results on the expressivity of the corresponding deep ConvAC via measures of entanglement w.r.t. a bi-partition of its input patches, and relate them to the number of channels in each layer of the ConvAC.

\subsection{Bounds on the Entanglement Measure} \label{sec:mincutclaim:bounds}

\begin{figure}
\centering
\includegraphics[scale=0.33]{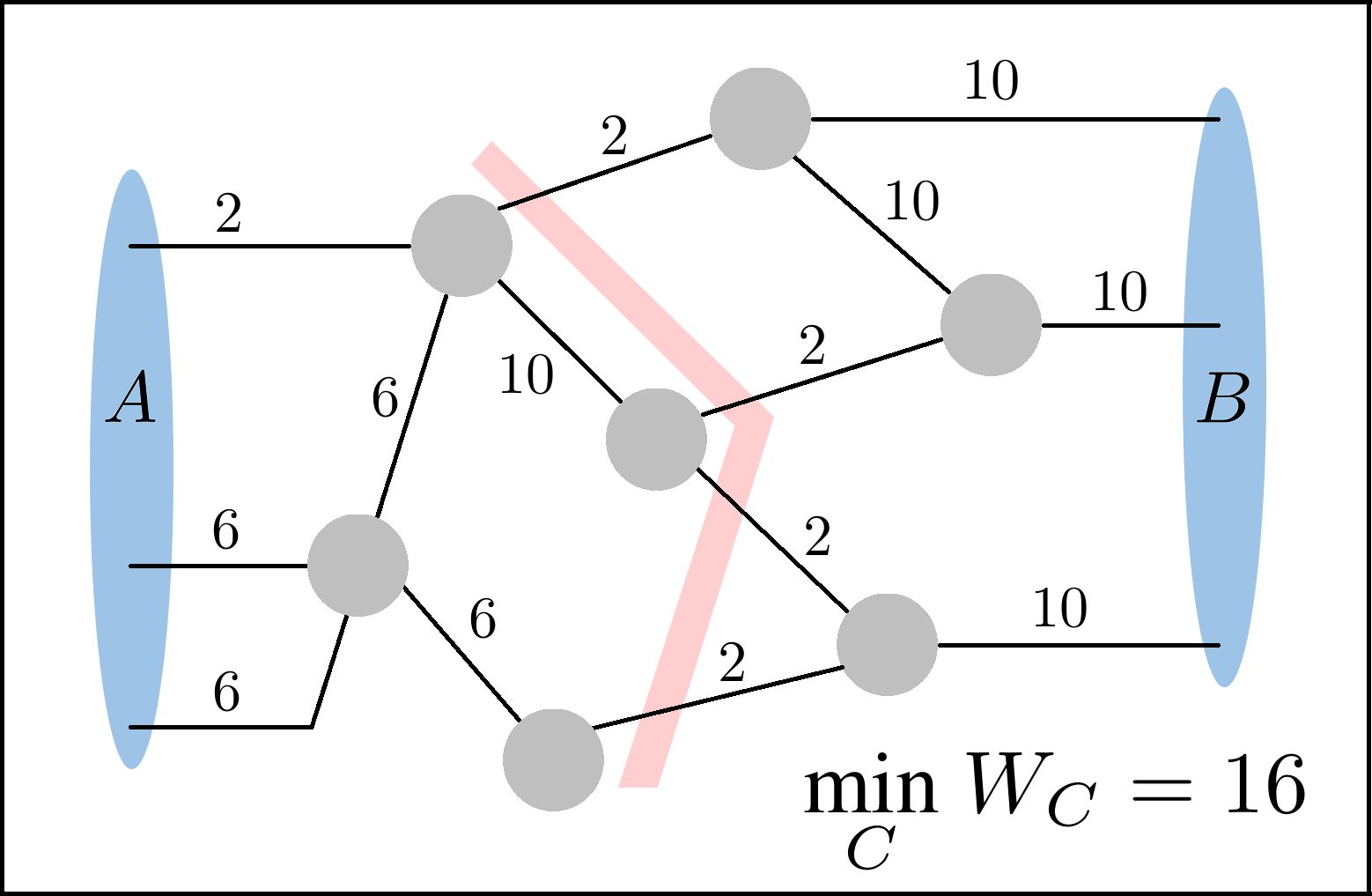}
\caption{An example for the minimal multiplicative cut between $A$ and $B$ in a simple TN.}
\label{fig:mincut}
\end{figure}
In claim~\ref{claim:upperbound} below, we provide an upper bound on the ability of a deep ConvAC to model correlations of its inputs, as measured by the Schmidt entanglement measure (see sec.~\ref{sec:CorrelationEntanglement:quantumentanglement}). This claim is closely related to attributes of TNs that are known in different forms in the literature.

\begin{claim} \label{claim:upperbound}
Let $(A,B)$ be a partition of $[N]$, and $\mat{\A^y}_{A,B}$ be the matricization \wrt~$(A,B)$ of a convolutional weights tensor $\A^y$ (eq.~\ref{eq:convac}) realized by a ConvAC depicted in fig.~\ref{fig:original_convac} with pooling window of size $2$ (the deep ConvAC network). Let $G(V,E,c)$ the graph representation of  $\phi$ corresponding to the ConvAC-weights TN, and let $(V^A,V^B)$ be the degree $1$ vertices partition in $G$ corresponding to $(A,B)$. Then, the rank of the matricization $\mat{\A^y}_{A,B}$ is no greater than: $\min_C W_C$, where $C$ is a cut w.r.t $(V^A,V^B)$ and $W_C$ is the multiplicative weight defined by eq.~\ref{CutWeight}.
\end{claim}
\begin{proof}
See appendix~\ref{app:Proofs:Upperbound}.
\end{proof}

Claim~\ref{claim:upperbound} states that a measure of the ability of the TN $\phi$ (that models the ConvAC-weights)  to represent correlations between two parts of its input, cannot be higher than the minimal weight over all the possible cuts in the network that separate between these two parts. Thus, if one wishes to construct a deep ConvAC that is expressive enough to model an intricate correlation structure according to some partition, it is advisable to verify that the convolutional network is able to support such correlations, by ensuring there is no cut separating these two parts in the corresponding TN that has a low weight. We will elaborate on such practical considerations in sec.~\ref{sec:mincutclaim:implications_channels}.

The upper bound provided above, alerts us when a deep ConvAC is too weak to model a desired correlation structure, according to the number of channels in each layer. Below, we provide a lower bound similar in spirit to a bound shown in~\cite{cui2016quantum}. Their claim is applicable for a TN with general tensors (no $\delta$ tensors), and we adapt it to the ConvAC-weights TN (that has $\delta$ tensors) which in effect ensures us that the entanglement measure cannot fall below a certain value for any specific arrangement of channels per layer.

\begin{theorem} \label{theorem:lowerbound}
Let $(A,B)$ be a partition of $[N]$, and $\mat{\A^y}_{A,B}$ be the matricization \wrt~$(A,B)$ of a convolutional weights tensor $\A^y$ (eq.~\ref{eq:convac}) realized by a ConvAC depicted in fig.~\ref{fig:original_convac} with pooling window of size $2$ (the deep ConvAC network). Let $G(V,E,c)$ the graph representation of  $\phi$ corresponding to the ConvAC-weights TN, and let $(V^A,V^B)$ be the degree $1$ vertices partition in $G$ corresponding to $(A,B)$.

Let $\phi^p$ be the TN represented by $G(V,E,c^p)$ where $c^p(e):=~\max_{n\in \N} p^n$ s.t. $p^n\leq c(e)$. In words, $\phi^p$ is a TN with the same connectivity as $\phi$, where all of the bond dimensions are modified to be equal the closest power of $p$ to their value in $\phi$ from below. Let $W^p_C$ be the weight of a cut $C$ \wrt~$(V^A,V^B)$ in the network $\phi^p$. Then, the rank of the matricization $\mat{\A^y}_{A,B}$ is at least: $~\max_p\min_C W^p_C$ almost always, i.e. for all configurations of the ConvAC network weights but a set of Lebesgue measure zero.
\end{theorem}
\begin{proof}
See appendix~\ref{app:Proofs:Lowerbound}.
\end{proof}

Theorem~\ref{theorem:lowerbound} above implies that the upper bound given in Claim~\ref{claim:upperbound} is saturated when all of the channel numbers in a deep ConvAC architecture are powers of some integer $p$. For a general arrangement of channel numbers, the upper bound is not tight and theorem~\ref{theorem:lowerbound} guarantees that the rank will not be lower than that of any ConvAC architecture with channel numbers which are powers of some integer $p$ yet are not higher than the original ConvAC channel numbers. Even though this is the lower bound we prove, we have a reason to believe the actual lower bound is much tighter. In appendix~\ref{app:Simulation}, we show simulations which indicate that deviations from the upper bound are actually quite rare and unsubstantial in value.

\subsection{Implications - How to Distribute the Number of Channels in a Deep Network} \label{sec:mincutclaim:implications_channels}
\begin{figure}
\centering
\includegraphics[width=\linewidth]{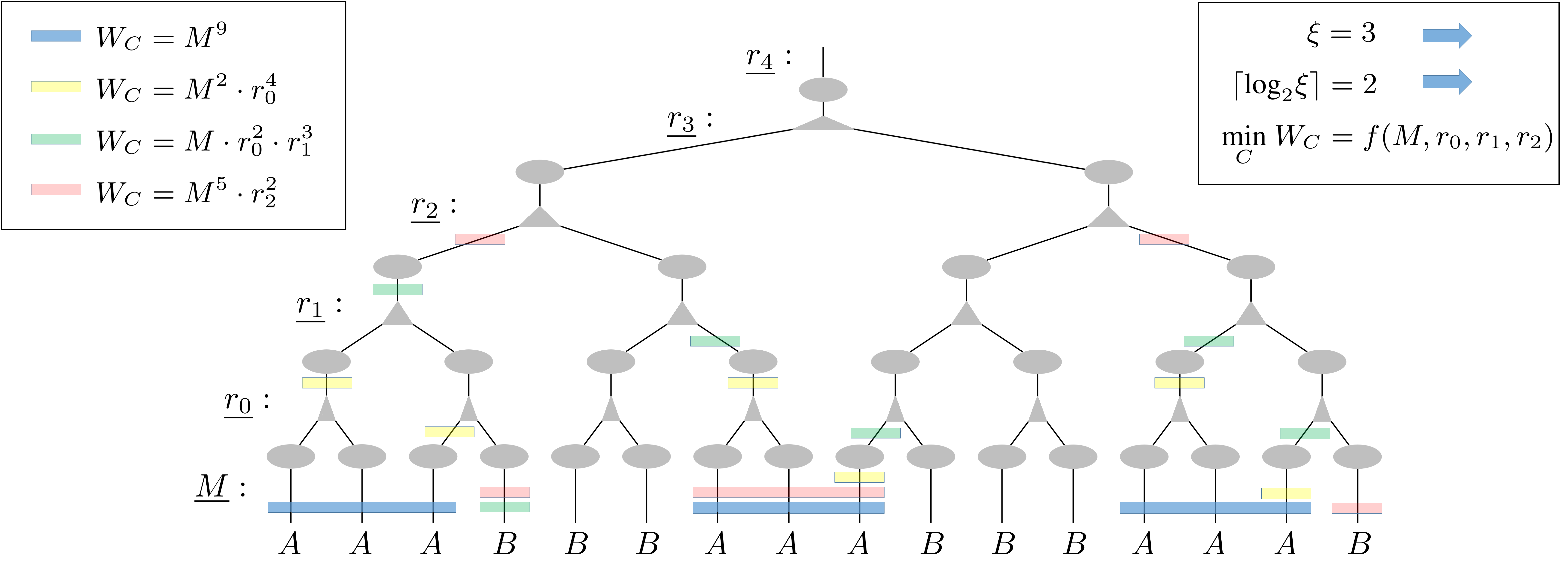}
\caption{(color available online) Examples for cuts w.r.t. a partition $(A,B)$ in an $N=16$ TN representing the weights of a deep ConvAC for which both $A$ and $B$ have contiguous segments of length $\xi=3$. Accordingly, any cut which includes edges from levels $l>\left \lceil{\log_2\xi}\right \rceil=2 $ must also include edges from levels $l\leq\left \lceil{\log_2\xi}\right \rceil=2 $. Moreover, in the minimal cut the edges from higher levels cannot be present as their contribution to the cut weight will be redundant. The bond dimensions of this TN were shown in sec.~\ref{sec:translations} to be equal to the number of channels in the respective layers in the ConvAC network represented by it. Thus, a direct consequence of the above is that for data characterized by short ranged correlations it is best to increase the number of channels in the lower layers, while for data characterized by long ranged correlations the channels in the deeper layers are important in order not to have `short-cuts' harming the required expressiveness of the function realized by the network.}
\label{fig:correlation_length}
\end{figure}

When physicists choose a TN to represent the coefficients tensor of a certain wave function, the entanglement characterizing the wave function is taken into consideration, and a network which can best model it is chosen. Thus, understanding the correlations characteristics of the wave function serves as a prior knowledge that helps restrict the hypothesis space to a suitable TN architecture. In accordance with the analogies discussed above, we will draw inspiration from this approach in the analysis of our results from the previous section, as it represents a `healthy' process of first quantifying the key correlations that the network is required to model, and then constructing it appropriately. This is in effect a control over the inductive bias of the network.

Accordingly, the bounds shown in the previous section not only provide us with theoretical observations regarding the role that the number of channels in each layer fulfils in the overall expressiveness of a deep ConvAC, but also entail practical implications for the construction of a deep network architecture when there is prior knowledge regarding the input. To get a feeling of what can be understood from the theoretical results, consider one-dimensional partitions similar in essence to the left-right partition and the interleaved partition depicted in fig~\ref{fig:partitions}. For a TN representing the ConvAC-weights  with pooling windows of size $2$, $N$ input patches and $L={\textrm {log}}_2N$ layers (an $N=8$ example is shown in fig.~\ref{fig:virtual_nodes}), it is simple to see that the minimal weight of a cut with respect to the left-right partition obeys:
\be
\label{mincut_left_right}
W^{\textrm {left-right} }_C={\textrm {min}}(r_{L-1},r_{L-2},...,r_l^{2^{(L-2-l)}},...,r_0^{N/4},M^{N/2}),
\ee
whereas the minimal weight of a cut w.r.t. the interleaved partition is guaranteed to be exponential in $N$ and obeys:
\be
\label{mincut_interleaved}
W^{\textrm {interleaved}}_C={\textrm {min}}(r_0^{N/4},M^{N/2}).
\ee

These are two examples that bring forth indications for the following `rule of thumb'. If one is interested in modeling elaborate correlations between pixels from opposite ends of an image, such as the ones characterizing face images for example, we see from the expression in eq.~\ref{mincut_left_right} that a small number of channels in deep layers can create an undesired `shortcut' which harms the expressiveness of the network in a way that prevents it form modeling the required correlations. In this case, it is advisable to keep more parameters in the deeper layers in order to obtain a higher entanglement measure for the required partition.  However, if one is interested in modeling only short ranged correlations, and knows that the typical input to the network will not exhibit relevant long ranged correlations, it is advisable to concentrate more parameters (in the form of more channels) in the lower levels,  as it raises the entanglement measure w.r.t the partition which corresponds to short ranged correlations.

The two partitions analyzed above represent two extreme cases of the shortest and longest ranged correlations. However, the min-cut results shown in sec.~\ref{sec:mincutclaim:bounds} apply to any partition of the inputs, so that implications regarding the channel numbers can be drawn for any intermediate length-scale of correlations. The relevant layers that contribute to the min-cut between partitions $(A,B)$ for which both $A$ and $B$ have contiguous segments of a certain length $\xi$ can be pictorially identified, see for example fig.~\ref{fig:correlation_length} for $\xi=3$. The minimal cut with respect to such a partition $(A,B)$ may only include the channel numbers $M,r_0,...,r_{\left \lceil{{\textrm {log}}_2\xi}\right \rceil}$. This is in fact a generalization of the treatment above with $\xi=1$ for the interleaved partition and $\xi=N/2$ for the left-right partition. Any cut which includes edges in higher levels is guaranteed to have a higher weight than the minimal cut, as it will have to additionally include a cut of edges in the lower levels in order for a separation between $A$ and $B$ to actually take place. This can be understood by flow considerations in the graph underlying the TN --- a cut that is located above a certain sub-branch can not assist in cutting the flow between $A$ and $B$ vertices that reside in that sub-branch.

For a data-set with features of a characteristic size $D$ (e.g. in a two-dimensional digit classification task it is the size of digits that are to be classified), such partitions of length scales $\xi<D$ are guaranteed to separate between different parts of a feature placed in any input location. However, in order to perform the classification task of this feature correctly, an elaborate function modeling a strong dependence between different parts of it must be realized by the network. As discussed above, this means that a high measure of entanglement with respect to this partition must be supported by the network, and now we are able to describe this measure of entanglement in terms of a min-cut in the TN graph. Concluding the above arguments, for a ConvAC with pooling windows of size $2$, if the characteristic features size is $D$ then the channel numbers up to layer $l=\left \lceil{{\textrm {log}}_2D}\right \rceil$ are more important than those of deeper layers. This understanding is verified and extended in section~\ref{sec:experiments}, where we present experiments exemplifying that the theoretical findings established above for the deep ConvAC, apply to a regular ConvNet architecture which involves the more common ReLU activations and average or max pooling. In natural images it may be hard to point out a single most important length scale $D$, however the conclusions presented in this section can be viewed as an incentive to better characterize the input correlations which are most relevant to the task at hand.

\subsection{Reproducing Depth Efficiency} \label{sec:mincutclaim:implications_depth}

The exponential depth efficiency result shown in \cite{cohen2016expressive}, can be straightforwardly reproduced by similar graph-theoretic considerations. In appendix~\ref{app:gen_pool}, we show an upper bound on the rank of matricization of the convolutional weights tensor for a case of a general pooling window. The bound implies that any amount of edges in a cut that are connected to the same $\delta$ tensor will contribute their bond dimension only once to the multiplicative weight of this cut, thus effectively reducing the upper bound when many cut edges belong to the same $\delta$ tensor. This does not affect our analysis of the deep network above, as the $\delta$ tensors in that network are only three legged (see fig.~\ref{fig:HT_TN}). Therefore, a cut containing more than one $\delta$ tensor leg can be replaced by an equivalent cut containing only one leg of that $\delta$ tensor, and the value of $\min_C W_C$ is unchanged.

Observing fig.~\hyperref[fig:CPExample]{~\ref{fig:CPExample}(b)} which shows the TN corresponding to the shallow ConvAC architecture, the central positioning of a single $\delta$ tensor implies that under \emph{any} partition of the inputs $(A,B)$, the minimal cut will obey $W_C^{\textrm {min}}={\textrm {min}}(M^{{\textrm {min}}(|A|,|B|)},k)$. Thus, in order to reach the exponential in $N$ measure of entanglement w.r.t. the interleaved partition that was obtained in eq.~\ref{mincut_interleaved} for the deep network, the number of channels in the single hidden layer of the shallow network $k$, must grow exponentially with $N$. Therefore, one must exponentially enlarge the size of the shallow network in order to achieve the expressiveness that a polynomially sized deep network achieves, and an exponential depth efficiency is shown.

Overall, the conclusions regarding the architecture of a deep ConvAC drawn in this section are motivated by physical understandings of the role of correlations, and their connection to Tensor Networks. We have seen how the description of a deep ConvAC in terms of a valid graph enables the analysis of the effect that the number of channels in each of its layers has on the correlations it is able to model, through the tool of a minimal cut in the graph. In the following section, we demonstrate by experiments that the results and conclusions drawn above also hold in the more common setting of a ConvNet with ReLU activations and average or max pooling.

\section{Experiments} \label{sec:experiments}

The min-cut analysis on the TN representing a deep ConvAC translates prior knowledge on how correlations among input variables (or image patches) are modeled, into the architectural design of number of channels per layer of the ConvAC.
For instance, in order to optimally classify features of a characteristic size $D$, more  channels are required up to layer $l=\left \lceil{{\textrm {log}}_2D}\right \rceil$ than in deeper layers. Moreover, when analyzing long ranged correlations that correspond to a large feature size, the number of channels in deeper layers was shown to act as an undesired `short-cut' (see e.g. the functional form of eq.~\ref{mincut_left_right}) and the addition of more channels in these layers is recommended. In this section, we demonstrate by experiments that these conclusions extend to the more common ConvNet architecture which involves ReLU activations and average or max pooling.

Two tasks were designed, one with a  short characteristic length to be referred to as the `local task' and one with a long characteristic length to be referred to as the `global task'. Both tasks are based on the MNIST data-set (\cite{lecun1998mnist}) and consist of $64\times64$ black background images on top of which resized binary MNIST images were
placed in random positions. For the local task, the MNIST images were shrunken to
small $8\times8$ images while for the global task they were enlarged to size $32\times32$. In both tasks the digit was to be identified correctly with a label $0,...,9$. See fig.~\ref{fig:MNIST_big_small}
for a sample of images from each task.

\begin{figure}
\centering
\includegraphics[scale=2]{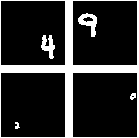}
\caption{Samples of the randomly positioned MNIST digits to be classified in the global task (above) and the local task (below).}
\label{fig:MNIST_big_small}
\end{figure}

We designed two different networks that tackle these two tasks, with a difference in the channel ordering scheme that is meant to emphasize the difference between the two tasks in accordance with the analysis above. In both networks, the first layer is a representation layer --- a $3\times 3$ (with stride 1) shared convolutional layer.
Following it are 6 hidden layers, each with $1\times 1$ shared convolution kernels followed by
ReLU activations and $2\times 2$ max pooling (with stride 2). Classification in both networks was preformed through $Y = 10$ network outputs, with prediction following the strongest activation. The difference between the two networks is in the channel ordering --- in the `wide-base' network they are wider in the beginning and narrow down in the deeper layers while in the `wide-tip' network they follow the opposite trend. Specifically, we set a parameter $r$ to determine each pair of such networks according to the following scheme:
\begin{itemize}
\vspace{-2mm}
\item Wide-base: [10; 4r; 4r; 2r; 2r; r; r; 10]
\vspace{-2mm}
\item Wide-tip: [10; r; r; 2r; 2r; 4r; 4r; 10]
\vspace{-2mm}
\end{itemize}
The channel numbers form left to right go from shallow to deep. The channel numbers were chosen to be gradually increased/decreased in iterations of two layers at a time as a trade-off ---  we wanted the network to be reasonably deep but not to have too many different channel numbers, in order to resemble conventional channel choices. The parameter count for both configurations is identical: $10 \cdot r + r \cdot r + r \cdot 2r +
2r \cdot 2r + 2r \cdot 4r + 4r \cdot 4r + 4r \cdot 10 = 31r^2 + 50r$. A result compliant with our theoretical expectation would be for the `wide-base' network to outperform the `wide-tip' network in the local classification task, and the converse to occur in the global classification task.
\begin{figure}
\centering
\includegraphics[width=\linewidth]{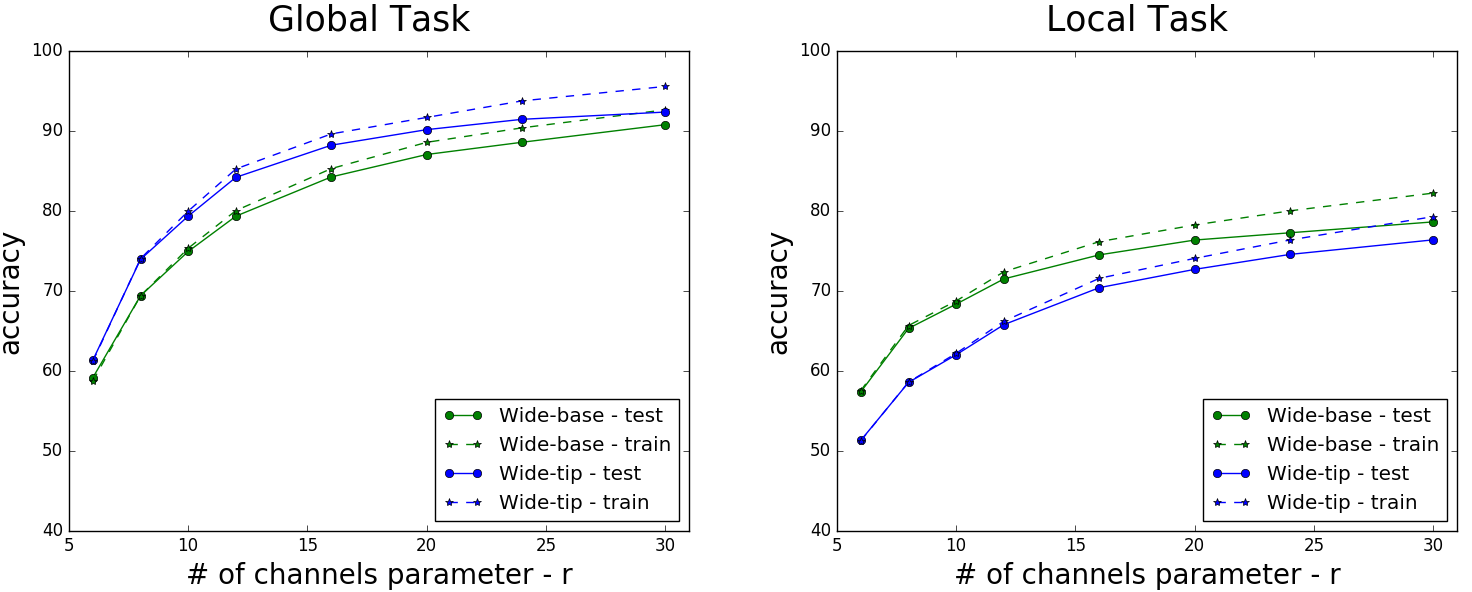}
\caption{(color available online) Results of applying deep convolutional rectifier networks with max pooling to the global and local
classification tasks. Two channel arrangements geometries were evaluated --- `wide-tip', which supports modeling correlations between far away regions, and `wide-base', which puts focus on correlations between regions that are close to each other. Each channel arrangements geometry outperforms the other on the
task which exhibits fitting correlations, demonstrating how prior knowledge regarding a task
at hand may be used to tailor the inductive bias through appropriate channel arrangements. Furthermore, these results demonstrate that the theoretical conclusions that were provided in sec~\ref{sec:mincutclaim} for a ConvAC, extend to the common ConvNet architecture which involves ReLU activations and max pooling.}
\label{fig:exp_results}
\end{figure}

Fig. \ref{fig:exp_results} shows the results of applying both the `wide-base' and `wide-tip' networks to the local and global tasks. Each task consisted of $60000$ training images and $10000$ test images, in correspondence with the MNIST database. Indeed, the `wide-base' network significantly outperforms the `wide-tip' network in the local classification task, whereas a clear opposite trend can be seen for the global classification task. This complies with our discussion above, according to which the `wide-base' network should be able to support short correlation lengths in the input, whereas the `wide-tip' network is predicted to put focus on longer correlation lengths. The low performance relative to the regular MNIST task is due to the randomization of positions, there is no such degree of freedom in the regular MNIST task. This is of no concern to us, only the relative performance between networks on each task is of interest in order to examine how our claims extend to the common ConvNet architecture.  The fact that the global task gets higher accuracies for all choices of $r$ is unsurprising, as it is clearly an easier task. An additional thing to note is that as $r$ grows, the accuracies of both tasks augment and with it, the difference between the performances of the architectures on both tasks decreases. This is attributed to the `hardness' of the task relative to the amount of parameters in the network. For these larger numbers of parameters in the networks, a harder task is expected to exhibit a more noticeable difference in performance between the two architectures.

Overall, in these experiments we show a notable compliance between
a typical length of the data and an appropriate network design. Our results can therefore be seen as a clear demonstration of how prior knowledge regarding a task at hand may be used to tailor the inductive bias of a deep convolutional network by designing the channel widths appropriately subject to the amount of resources at hand. We have shown how phenomena that were indicated by the theoretical analysis that was presented in this paper in the context of ConvACs, manifest themselves in the most prevalent and successful ConvNet architecture which involves ReLU activations and max pooling.

\section{Discussion} \label{sec:discussion}
The construction of a deep ConvAC in terms of a Tensor Network, is the main theoretical achievement of this paper. This method, which constructs the ConvAC by inner products between low order tensors rather than by outer products as has been performed to date, enabled us to carry a graph-theoretic analysis of a convolutional network, and tie its expressiveness to a minimal cut in the graph characterizing it. Specifically, network architecture related results were drawn, connecting the number of channels in each layer of the deep ConvAC with its ability to model given correlations of its inputs. These results effectively demonstrate a direct control over the inductive bias of the designed network via its channel numbers, and can help any practitioner in the design of a network that is meant to perform a task which is characterized by certain correlations among subsets of its input variables. The applicability of these results, which were theoretically proven for a deep ConvAC architecture, was demonstrated through experiments on a conventional deep convolutional network architecture (ConvNet) which involves ReLU activations and max pooling.

Our construction was enabled via a structural equivalence we drew between the function realized by a ConvAC and a quantum many-body wave function. This facilitated the transfer  of mathematical and conceptual tools employed by physicists when analyzing their wave functions of interest. Thus, we were able to transfer the concept of `entanglement measures' and use it as a well-defined quantifier of the deep network's expressive ability to model intricate correlation structures of its inputs. Moreover, since the prevalent tool in the numerical description of quantum many-body wave functions is the Tensor Network, the structural equivalence discussed above enabled us to harness results that were recently obtained in the physics community for our needs. Specifically, we employed bounds on the measures of entanglement of a function represented by a Tensor Network, that were shown by \cite{cui2016quantum} to be related to the min-cut in the Tensor Network graph. We adjusted this treatment, applying it to the Tensor Network which represents the weights tensor of a deep ConvAC and eventually employing it to attain our results.

Apart from the direct results discussed above, two important interdisciplinary bridges emerge from this work. The first is the description of a deep convolutional network as a Tensor Network that is subject to well-defined graph-theoretic tools. The results we drew  by using this connection, \ie~the relation between min-cut in the graph representation of a ConvAC to network expressivity measures, may constitute an initial example for using the connection to TNs for the application of graph-theoretic measures and tools to the analysis of the function realized by a deep convolutional network. The second bridge, is the connection between the two seemingly unrelated fields of quantum physics and deep learning. The field of quantum Tensor Networks is a rapidly evolving one, and the established construction of a successful deep learning architecture in the language of Tensor Networks may allow applications and insights to be transferred between the two domains. For example, the tree shaped Tensor Network that was shown in this work to be equivalent to a deep convolutional network, has been known in the physics community for nearly a decade to be inferior to another deep Tensor Network architecture by the name of \emph{Multiscale Entanglement Renormalization Ansatz} (MERA)(\cite{vidal2008class}).

The MERA Tensor Network architecture introduces overlaps by adding `disentangling' operations prior to the pooling operations, which in translation to machine learning network terms effectively mix activations that are intended to be pooled in different pooling windows. This constitutes an exemplar case of how the Tensor-Networks/deep-learning connection that was established in this work allows a bi-directional flow of tools and intuition. Physicists have a good grasp of how these specific overlapping operations allow a most efficient representation of functions that exhibit high correlations at all length scales (\cite{vidal2007entanglement}). Accordingly, a new view of the role of overlaps in the high expressivity of deep networks as effectively `disentangling' intricate correlations in the data can be established. In the other direction, as deep convolutional networks are the most empirically successful machine learning architectures to date, physicists may benefit from trading their current `overlaps by disentangling' scheme to the use of overlapping convolutional windows (which were recently proven to contribute exponentially to the expressive capacity of neural networks by \cite{sharir2017expressive}), in their search for expressive representations of quantum wave functions. The employment of convolutional networks for efficient quantum wave function representation was suggested recently by e.g. \cite{carleo2017solving}, \cite{deng2017quantum}.

An interesting direction indicated by our work is the characterization of the correlations structure of a given data-set. As we demonstrate above, fitting the inductive bias of the deep network to such input correlations is possible via its architectural parameters, namely the number of channels in each layer. We provided means of characterizing such correlations in a simple case, by tying them to the size of the features in an elementary `fixed feature size' classification task. However, in a general more natural case, the best way to define such correlations is unclear.  For example, \cite{lin2016critical} show interesting studies according to which the mutual information (which can be viewed as a measure of correlations) that characterizes various data-sets such as English Wikipedia, works of Bach, the human genome etc., decays polynomially with a critical exponent similar in value to that of the critical two-dimensional Ising model. Indeed, also in the realm of correlations characterization there is a lot to take from physics analyses, and the bridge formed in this work between deep convolutional networks and Tensor Networks is a natural place to start. Another possible direction to pursue, given the bridge between ConvACs and Tensor Networks, is the use of numerical methods, such as those inspired by the DMRG algorithm (\cite{white1992density}), for training deep networks by mapping a ConvNet architecture to a Tensor Network, performing the training there and then mapping the trained parameters back to the ConvNet architecture. Some attempts along those lines have been made, such as \cite{stoudenmire2016supervised} who trained a matrix product state (MPS) Tensor Network architecture to preform supervised learning tasks, but those attempts did not have the benefit of a mapping to an existing deep ConvNet architecture prior the Tensor Network optimization process.

To conclude, we have presented in this paper connections between the quantum entanglement of a many-body wave function and correlations modeled by a deep convolutional network (ConvAC to be specific), and demonstrated how the graphical model of Tensor Networks can describe the architecture of such a network. From these connections, we were able to draw novel results which provided us with theoretical observations regarding the role that the number of channels in each layer fulfils in the overall expressiveness of a deep convolutional network. Furthermore, the results yielded practical implications for the construction of a deep network architecture when there is prior knowledge regarding the input. We view this work as an exciting bridge for transfer of tools and ideas between fields, and hope it will reinforce a fruitful interdisciplinary discourse.

\section*{Acknowledgements}
We have benefited from discussions with Or Sharir, Ronen Tamari, Markus Hauru and Eyal Leviatan. This work is supported by Intel grant ICRI-CI $\#$9-2012-6133, by ISF Center grant 1790/12 and by the European Research Council (TheoryDL project).

\section*{References}
\small{
\bibliographystyle{plainnat}
\bibliography{references_phys}
}

\clearpage
\appendix

\section{A Recursive Construction of the Deep ConvAC Tensor Network} \label{app:HTrecursive}

\begin{figure}
\centering
\includegraphics[scale=0.28]{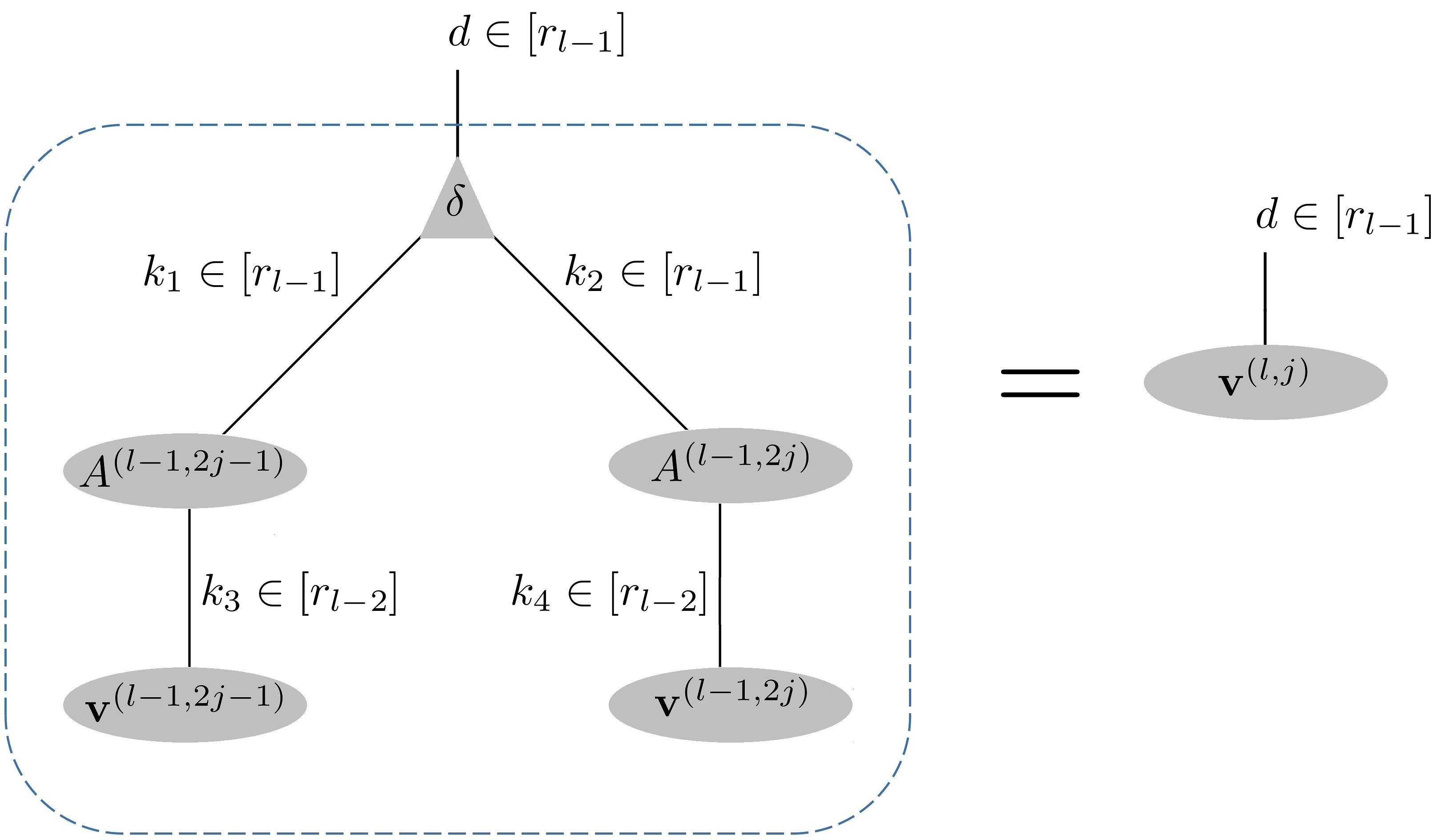}
\caption{A recursive building block of the deep ConvAC TN that was presented in section sec.~\ref{sec:translations:deep}. This scheme is the TN equivalent of two feature vectors in the $l-1$ level being operated on with the conv$\rightarrow$pool sequence of a deep ConvAC shown in fig.~\ref{fig:original_convac}, as is demonstrated below.}
\label{fig:HT_TN_Building_Block}
\end{figure}

The calculation performed by a one-dimensional ConvAC for a general $N$ (s.t. ${\textrm {log}}_2N\in\N$), is given by the recursively defined TN representation shown in fig.~\ref{fig:HT_TN_Building_Block}. $\vv^{(l,j)}\in\R^{r_{l-1}}$ is a vector of actual activations generated during a computation across in the $l^{th}$ level of the network shown in fig.~\ref{fig:original_convac}. Recall that $r_{-1}:=M$, and that $\vv^{(0,j)}\in\R^{M},~j\in[N]$ is a vector in the representation layer (see fig.~\ref{fig:HT_TN}). To demonstrate that this TN indeed defines the calculations performed by a ConvAC, we conjecture that the equality in fig.~\ref{fig:HT_TN_Building_Block} holds, namely that for $l=1,...,L=\log_2 N$ the $d^{th}$ component of each such activations vector is given in terms of the vectors in the preceding layer by:
\bea
v^{(l,j)}_d=
\sum_{k_1,k_2=1}^{r_l-1}
\sum_{k_3,k_4=1}^{r_{l-2}}
A^{(l-1,2j-1)}_{k_1k_3}v^{(l-1,2j-1)}_{k_3}
A^{(l-1,2j)}_{k_2k_4}v^{(l-1,2j)}_{k_4}
\delta_{k_1k_2d}
\nonumber \\
~~~~~~~~~~~=
\sum_{k_1,k_2=1}^{r_l-1}
(A^{(l-1,2j-1)}\vv^{(l-1,2j-1)})_{k_1}
(A^{(l-1,2j)}\vv^{(l-1,2j)})_{k_2}
\delta_{k_1k_2d},
\label{eq:buildingblock1}
\eea
where $d\in [r_{l-1}]$. In the first line of eq.~\ref{eq:buildingblock1} we simply followed the TN prescription and wrote a summation over all of the contracted indices in the left hand side of fig.~\ref{fig:HT_TN_Building_Block}, and in the second line we used the definition of matrix multiplication. According to the construction of $A^{(l,j)}$ given in sec.~\ref{sec:translations:deep}, the vector $\uu^{(l,j)}\in\R^{r_{l}}$ defined by $\uu^{(l,j)}:=A^{(l,j)}\vv^{(l,j)}$, upholds $u_\gamma=\left\langle\aaa^{l,j,\gamma},\vv^{(l,j)}\right\rangle,~\gamma\in[r_l]$ where the weights vector $\aaa^{l,j,\gamma} \in \R^{r_{l-1}}$ was introduced in eq.~\ref{eq:ht_decomp} . Thus, eq.~\ref{eq:buildingblock1} is reduced to:
\be
v^{(l,j)}_d=
\sum_{k_1,k_2=1}^{r_l-1}
\left\langle\aaa^{l-1,2j-1,k_1},\vv^{(l-1,2j-1)}\right\rangle
\left\langle\aaa^{l-1,2j,k_2},\vv^{(l-1,2j)}\right\rangle
\delta_{k_1k_2d}.
\ee
Finally, by definition of the $\delta$ tensor, the sum vanishes and we obtain the required expression for the operation of the ConvAC:
\be
v^{(l,j)}_d=
\left\langle\aaa^{l-1,2j-1,d},\vv^{(l-1,2j-1)}\right\rangle
\left\langle\aaa^{l-1,2j,d},\vv^{(l-1,2j)}\right\rangle,
\ee
where an activation in the $d^{th}$ feature map of the $l^{th}$ level holds the multiplicative pooling of the results of  two activation vectors from the previous layer convolved with the $d^{th}$ convolutional weight vector for that layer. Applying this procedure recursively is exactly the conv$\rightarrow$pool$\rightarrow...\rightarrow$conv$\rightarrow$pool scheme that lies at the heart of the ConvAC operation (fig.~\ref{fig:original_convac}). Recalling that $r_L:=Y$, the output of the network is given by:
\be
\h_y(\x_1...,\x_N) =A^{(L,1)}\vv^{(L,1)}.
\ee
\section{Deferred Proofs } \label{app:Proofs}

\subsection{Upper Bound on the Entanglement Measure } \label{app:Proofs:Upperbound}

\begin{figure}
\centering
\includegraphics[width=\linewidth]{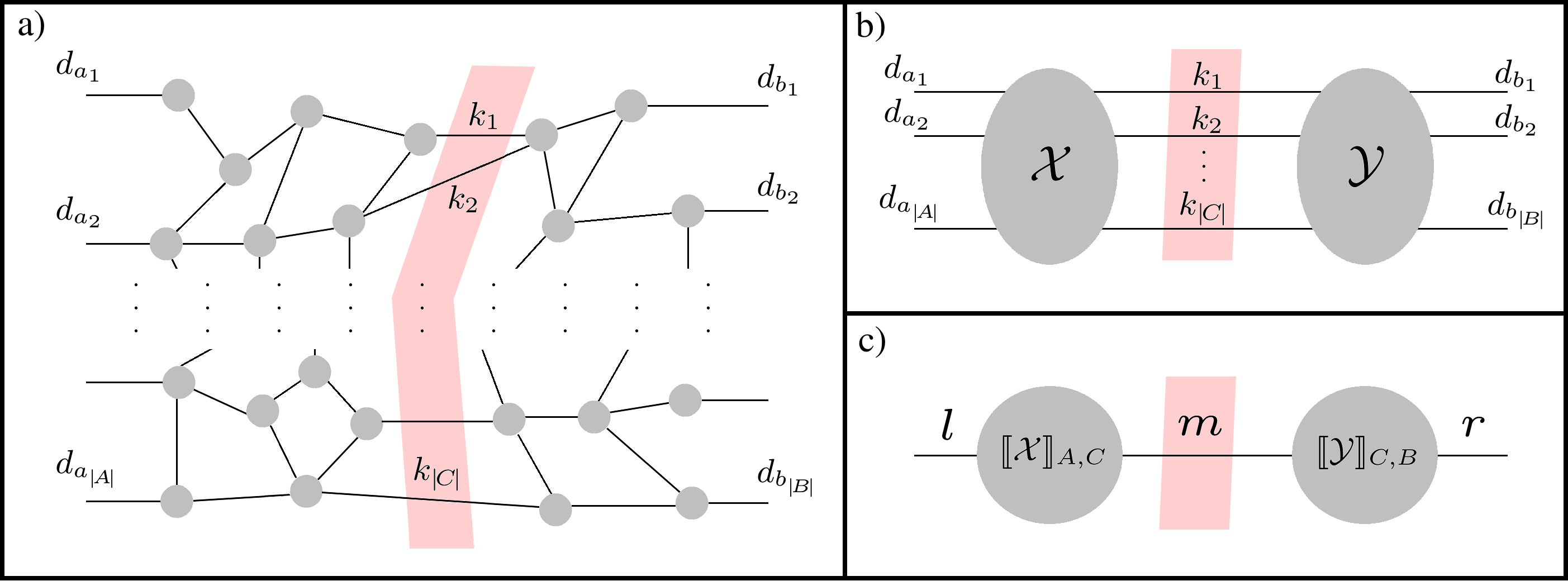}
\caption{Accompanying illustrations for the proof of claim~\ref{claim:upperbound}. a) An example for an arbitrarily inter-connected TN with $N$ external indices, arranged such that the indices corresponding to group $A$ are on the left and indices corresponding to group $B$ are on the right. The cut marked in pink in the middle separates between $A$ and $B$. b) A contraction of all the internal indices to the left and to the right of the cut results in two higher order tensors, each with external indices only from group $A$ or $B$, connected to each other by the edges of the cut. c) Finally coalescing the indices into three groups, results in a matrix that on one hand is equal to the matricization w.r.t. $(A,B)$ of the tensor represented by a TN in a), and on the other is equal to a multiplication of matrices, the rank of which is upper bounded by $\prod_{i=1}^{|C|} c_{k_i}$, thus proving claim~\ref{claim:upperbound}.}
\label{fig:upper_bound}
\end{figure}

Claim~\ref{claim:upperbound} effectively states that an upper bound on the rank of matricization w.r.t. $(A,B)$ of a tensor represented by a TN, is the minimal multiplication of the bond dimensions of edges composing a cut in the corresponding graph w.r.t. $(V^A,V^B)$. This claim is well know in different forms in the literature.
In the following, we provide a proof of the claim accompanied by a graphical example for clarity.
\medskip

\begin{proof}(of claim~\ref{claim:upperbound})

 We will use the example shown in fig.~\hyperref[fig:upper_bound]{~\ref{fig:upper_bound}(a)} of a general TN with arbitrary connectivity. The edges of the TN $\phi$ are marked by the index associated with them. Any index $p\in\{d,k\}$ runs between $1$ and its bond dimension marked by $c_p$, which upholds $c_p:=c(e_p)$ where $e_p\in E$ is the edge associated with the index $p$. For the given partition $(A,B)$, denote $A=\{a_1,...,a_{|A|}\}~,~B=\{b_1,...,b_{|B|}\}$ and let $I_A \cupdot I_B=\{d_{1},\ldots,d_{N}\}$ be the corresponding partition of external indices, where $I_A=\{d_{a_1},...,d_{a_{|A|}}\}$ and $I_B=\{d_{b_1},...,d_{b_{|B|}}\}$. Let ${\cal H} ^A$ and ${\cal H} ^B$ be the spaces corresponding to the different configurations of the indices in $I_A$ and $I_B$, respectively, their dimensions given by:
\be
{\textrm {dim}}({\cal H} ^A)=\prod_{i=1}^{|A|} c_{d_{a_i}}~,~{\textrm {dim}}({\cal H} ^B)=\prod_{i=1}^{|B|} c_{d_{b_i}}.
\label{totalExternalDimensions}
\ee

In the example shown in fig.~\hyperref[fig:upper_bound]{~\ref{fig:upper_bound}(a)}, the graph is arranged s.t. $A$ is on the left and $B$ is on the right. The marked cut $C$ that separates between $A$ and $B$ is arbitrarily chosen as a representative cut, and we denote the indices of the cut edges by $I_C=\{k_1,...,k_{|C|}\}$. It is noteworthy, that any index $k_i$ in the cut is allowed to be an external index,~\ie~the cut is allowed to contain any amount of external edges.

Now, two contractions can be preformed, separately contracting all the tensors to the left of the cut and to the right of it. We are left with two higher order tensors, ${\X}_{d_{a_1}...d_{a_{|A|}}k_1...k_{|C|}}$ and ${\Y}_{k_1...k_{|C|}d_{b_1}...d_{b_{|B|}}}$ each with external indices only from $I_A$ or $I_B$, connected to each other by the edges of the cut, as is depicted in fig.~\hyperref[fig:upper_bound]{~\ref{fig:upper_bound}(b)}. If any cut index $c_i$ is equal to any external index $d_j$, then respective tensor simply includes the term $\delta_{c_id_j}$.

Note that the space corresponding to the different configurations of the cut indices $I_C$ is of dimensions $\prod_{i=1}^{|C|} c_{k_i}$, which is exactly equal to $W_C$ (see eq.~\ref{CutWeight}), since by definition $c_{k_i}=c(e_{k_i})$. Next, coalescing the indices in $I_A$ into a single index representing all of the external indices to the left of the network: $l\in [{\textrm {dim}}({\cal H} ^A)]$, the indices in $I_B$ into a single index representing all of the external indices to the right of the network: $r\in [{\textrm {dim}}({\cal H} ^B)]$, and the indices in $I_C$ into a single index representing all of the cut indices: $m\in [W_C]$, a TN which is equal to the matricization $\mat{\A}_{A,B}$ is obtained (fig.~\hyperref[fig:upper_bound]{~\ref{fig:upper_bound}(c)}).

According to the TN in fig.~\hyperref[fig:upper_bound]{~\ref{fig:upper_bound}(c)}, the matricization $\mat{\A}_{A,B}$ can be written as a multiplication of two matrices. Component wise, this can be written as:
\be
(\mat{\A}_{A,B})_{lr}=\sum_{m=1}^{W_C} (\mat{\X}_{A,C})_{lm}(\mat{\Y}_{C,B})_{mr},
\ee
where any amount of cut indices that are also external indices translate as a blocks of the identity matrix of a respective size on diagonal. Finally, since this construction is true for any cut $C$ w.r.t $(A,B)$, the rank of $\mat{\A}_{A,B}$ upholds: ${\textrm {rank}}(\mat{\A}_{A,B})\leq\ \min_C W_C$, satisfying the claim for any general TN, and specifically for the ConvAC TN.
\end{proof}

In the following subsection we shall see that this upper bound is tight for a deep ConvAC network with a pooling window of size $2$ under certain conditions. In appendix~\ref{app:gen_pool}, we provide a smaller upper bound for a ConvAC network with a pooling window that is larger than $2$, which gives us insight regarding the lack in expressiveness of a shallow convolutional network relative to a deep convolutional network, discussed in sec.~\ref{sec:mincutclaim:implications_depth}.

\subsection{Lower Bound on the Entanglement Measure } \label{app:Proofs:Lowerbound}

In the following we prove theorem~\ref{theorem:lowerbound}, showing in effect that the upper bound on the rank of the matricization of the deep ConvAC convolutional weights tensor shown in claim~\ref{claim:upperbound} is tight when all of the channel numbers are powers of some integer $p$, and guaranteeing a positive result in all cases.  Our proof strategy is similar to the one taken in \cite{cui2016quantum}, however we must deal with the restricted $\delta$ tensors present in the network corresponding to a ConvAC (the triangle nodes in fig.~\ref{fig:HT_TN}). We first quote and show a few results that will be of use to us. We begin by quoting a claim regarding the prevalence of the maximal matrix rank for matrices whose entries are polynomial functions --- claim~\ref{claim:rank_everywhere}. Next, we quote a famous graph theory result known as the Undirected Menger's Theorem (\cite{menger1927allgemeinen},~\cite{elias1956note},~\cite{ford1956maximal}) which relates the number of edge disjoint paths in an undirected graph to the cardinality of the minimal cut  --- theorem~\ref{theorem:menger}. After this, we show that the rank of matricization of the tensor represented by $\phi^p$ that is defined in theorem~\ref{theorem:lowerbound}, is a lower bound on the rank of matricization of the tensor represented by $\phi$ --- lemma~\ref{lemma:phi_p_to phi}. Then, we prove that the upper bound in claim~\ref{claim:upperbound} is tight when all of the channel numbers are any powers of the same integer $p\in\N$ --- lemma~\ref{lemma:powers_of_p}. Finally, when all the preliminaries are in place, we show how the result in theorem~\ref{theorem:lowerbound} is implied.

\begin{claim}

\label{claim:rank_everywhere}
Let $M, N, K \in \N$, $1 \leq r \leq \min\{M,N\}$ and a polynomial mapping $A:\R^K \to \R^{M \times N}$,
i.e. for every $i \in [M]$ and $j\in [N]$ it holds that $A_{ij}:\R^K \to \R$ is a polynomial function.
If there exists a point $\x \in \R ^K$ s.t. ${\textrm {rank}}{(A(\x))} \geq r$, then the set $\{\x \in \R^K : \textrm{rank}{(A(\x))} < r\}$
has zero measure (w.r.t. the Lebesgue measure over $\R^K$).
\end{claim}
\begin{proof}
See \citet{sharir2016tensorial}.
\end{proof}

Claim~\ref{claim:rank_everywhere} implies that it suffices to show an assignment of the ConvAC network weights achieving a certain rank of matricization of the convolutional weights tensor, in order to show this is the rank for all configurations of the network weights but a set of Lebesgue measure zero. Essentially, this means that it is enough to provide a specific assignment that achieves the required bound in theorem~\ref{theorem:lowerbound} in order to prove the theorem. Next, we present the following well-known graph theory result:

\begin{theorem}(\cite{menger1927allgemeinen},~\cite{elias1956note},~\cite{ford1956maximal})
\label{theorem:menger}$[$Undirected Menger's Theorem$]\\$
Let $G = (V,E)$ be an undirected graph with a specified partition $(A,B)$ of the set of degree $1$ vertices. Let $MF(G)$ be the maximum number of edge disjoint paths (paths which are allowed to share vertices but not edges) in $G$ connecting a vertex in $A$ to a vertex in $B$. Let $MC(G)$ be the minimum cardinality of all edge-cut sets between $A$ and $B$. Then, $MF(G)=MC(G)$.
\end{theorem}
\begin{proof}
See e.g. \cite{cui2016quantum}.
\end{proof}

Theorem~\ref{theorem:menger} will assist us in the proof of lemma~\ref{lemma:powers_of_p}. We will use it in order to assert the existence of edge disjoint paths in an auxiliary graph (fig.~\ref{fig:phi_star}), which we eventually utilize in order to provide the required weights assignment in the lemma.

\begin{lemma}
\label{lemma:phi_p_to phi}
Let $(A,B)$ be a partition of $[N]$, and $\mat{\A^y}_{A,B}$ be the matricization \wrt~$(A,B)$ of a convolutional weights tensor $\A^y$ (eq.~\ref{eq:convac}) realized by a ConvAC depicted in fig.~\ref{fig:original_convac}. Let $\phi$ be the TN corresponding to this ConvAC network, and let $\phi^p$ be a TN with the same connectivity as $\phi$, where all of the bond dimensions are modified to be equal the closest power of $p$ to their value in $\phi$ from below. Let $(\A^p)^y$ be the tensor represented by $\phi^p$ and let there exist an assignment of all of the tensors in the network $\phi^p$ for which ${\textrm {rank}}(\mat{(\A^p)^y}_{A,B}) =R$. Then, $~{\textrm {rank}}(\mat{\A^y}_{A,B})$ is at least $R$ almost always, i.e. for all configurations of the weights of $\phi$ but a set of Lebesgue measure zero.

\end{lemma}
\begin{proof}

Consider the specific assignment of all of the tensors in the network $\phi^p$ which achieves ${\textrm {rank}}(\mat{(\A^p)^y}_{A,B}) =R$, and leads to the resultant tensor $(\A^p)^y$ upon contraction of the network. Observing the form of the deep ConvAC TN presented in sec.~\ref{sec:translations:deep}, we see that it is composed of $\delta$ tensors and of weight matrices $A^{(l,j)}\in\R^{r_l\times r_{l-1}}$. Recalling that the entries of the former are dictated by construction and obey eq.~\ref{eq:deltadef}, the assignment of all of the tensors in the network $\phi^p$ is an assignment of all entries of the weight matrices in $\phi^p$ denoted by $(A^p)^{(l,j)}$, $l\in[L]\vee\{0\},j\in[N/2^l]$.

We denote the bond dimension at level $l\in[L]\vee\{-1,0\}$ of $\phi^p$ by $r_l^p$ (recall that we defined $r_{-1}=M$). By the definition of $\phi^p$, this bond dimension cannot be higher than the bond dimension in the corresponding level in $\phi:~\forall l~r_l^p\leq r_l$. Accordingly, the matrices in $\phi$ do not have lower dimensions (rows or columns) than the corresponding matrices in $\phi^p$. Thus, one can choose an assignment of the weights of all the matrices in $\phi$ to uphold the given assignment for the matrices in $\phi^p$ in their upper left blocks, and assign zeros in the extra spaces:
\begin{equation}
\begin{array}{c}
(A^{(l,j)})_{i_1i_2}=\left\{ \begin{array}{c}
((A^p)^{(l,j)})_{i_1i_2},\quad i_1\leq r_{l}^p,i_2\leq r_{l-1}^p\\
0,\quad  ~~~~~~~~~~~otherwise
\end{array}\right..\\
\end{array}\label{eq:phi_from_phi_p_assignment}
\end{equation}

Next, we consider a contraction of all the internal indices of $\phi$, which by definition results in the convolutional weights tensor $\A^y$. In this contraction, one can split the sum over all of the indices that range in $[r_l]$ for which $r_l^p < r_l$ into two sums: one ranging in $[r_l^p]$ and the other in $r_l^p+[r_l-r_l^p]$. For clarity we will not provide an expression for the entire contraction of $\phi$ which involves many internal indices. To understand the sum splitting schematically, let $k_l$ be an index that ranges in $[r_l]$ for which $r_l^p < r_l$. We perform the following splitting on the sum over $k_l$:
\be
\sum_{k_l=1}^{r_l}\{\cdots\}~~\rightarrow~~\sum_{k_l^{\textrm{low}}=1}^{r_l^p}\{\cdots\}~~+\sum_{k_l^{\textrm{high}}=r_l^p+1}^{r_l}\{\cdots\},
\label{sum_split}
\ee
where $k_l$ is switched into $k_l^{\textrm{high/low}}$ in all of the summands in the respective sums. The overall contraction will now be split into many sums involving different `high' and `low' indices. According to the assignment of $A^{(l,j)}$ (eq.~\ref{eq:phi_from_phi_p_assignment}), all sums that include any index labeled by `high' will vanish, and we will be left with a single contraction sum over all the indices labeled by `low'. It is important to note that a $\delta$ tensor of dimension $r_l$ provides that same values as a $\delta$ tensor of dimension $r_l^p$ when observing only its first $r_l^p$ entries in each dimension, as is clear from the $\delta$ tensor definition in eq.~\ref{eq:deltadef}. Finally, we observe that this construction leads to $\A^y$ containing the tensor $(\A^p)^y$ as a hypercube in its entirety and holding zeros elsewhere, leading to $~{\textrm {rank}}(\mat{\A^y}_{A,B}) =  {\textrm {rank}}(\mat{(\A^p)^y}_{A,B}) = R$. Using claim~\ref{claim:rank_everywhere}, this specific assignment implies that $~{\textrm {rank}}(\mat{\A^y}_{A,B})$ is at least $R$ for all configurations of the weights of $\phi$ but a set of Lebesgue measure zero, satisfying the lemma.
\end{proof}

Lemma~\ref{lemma:phi_p_to phi} basically implies that showing that the upper bound on the rank of the  matricization of the deep ConvAC convolutional weights tensor that is presented in claim~\ref{claim:upperbound} is tight when all of the channel numbers are powers of some integer $p$ (which we show below in lemma~\ref{lemma:powers_of_p}), is enough in order to prove the lower bound stated in theorem~\ref{theorem:lowerbound}.

\medskip
\begin{lemma}
\label{lemma:powers_of_p}
Let $(A,B)$ be a partition of $[N]$, and $\mat{\A^y}_{A,B}$ be the matricization \wrt~$(A,B)$ of a convolutional weights tensor $\A^y$ (eq.~\ref{eq:convac}) realized by a ConvAC depicted in fig.~\ref{fig:original_convac} with pooling window of size $2$ (the deep ConvAC network). Let $G(V,E,c)$ represent the TN $\phi$ corresponding to the ConvAC network s.t. $\forall e\in E, \exists n\in \N: c(e)=p^n$ , and let $(V^A,V^B)$ be the vertices partition of $V^{\textrm {inputs}}$ in $G$ corresponding to $(A,B)$.  Let $W_C$ be the weight of a cut $C$ \wrt~$(V^A,V^B)$. Then, the rank of the matricization $\mat{\A^y}_{A,B}$ is at least $\min_C W_C$ almost always, \ie~for all configurations of the ConvAC network weights but a set of Lebesgue measure zero.
\end{lemma}

It is noteworthy, that lemma~\ref{lemma:powers_of_p} is stated similarly to claim~\ref{claim:upperbound}, with two differences: ${\boldsymbol{1)}}\min_C W_C$ appears as a lower bound on the rank of matricization of the convolutional weights tensor rather than an upper bound, and ${\boldsymbol{2)}}$  all of the channel numbers are restricted to powers of the same integer $p$.  That is to say, by proving this lemma we in fact show that the upper bound proven in claim~\ref{claim:upperbound} is tight for this quite general setting of channel numbers.

\medskip
\begin{proof}(of lemma~\ref{lemma:powers_of_p})

In the following, we provide an assignment of indices for the tensors in $\phi$ for which the rank of the matricization $\mat{\A^y}_{A,B}$ is at least: $\min_C W_C$. In accordance with claim~\ref{claim:rank_everywhere}, this will satisfy the lemma as it implies this rank is achieved for all configurations of the ConvAC network weights but a set of Lebesgue measure zero.

The proof of lemma~\ref{lemma:powers_of_p} is organized as follows. We begin with the construction of the TN $\phi^*$ presented in fig.~\ref{fig:phi_star} from the original network $\phi$, and the show that it suffices to analyse $\phi^*$ for our purposes. Next, we elaborate on the form that the $\delta$ tensors in $\phi$ assume when constructed in $\phi^*$. We then use this form to define the concept of $\delta$ restricted edge disjoint paths, which morally are paths from $A$ to $B$ that are guaranteed to be compliant with the form of a $\delta$ tensor when passing through it. Finally, we use such paths in order to provide an assignment of the indices for the tensors in $\phi^*$ which upholds the required condition.

\medskip
\underline{\emph{$\phi^*$ and the Equivalence of Ranks:}}
\medskip

\begin{figure}
\centering
\includegraphics[width=\linewidth]{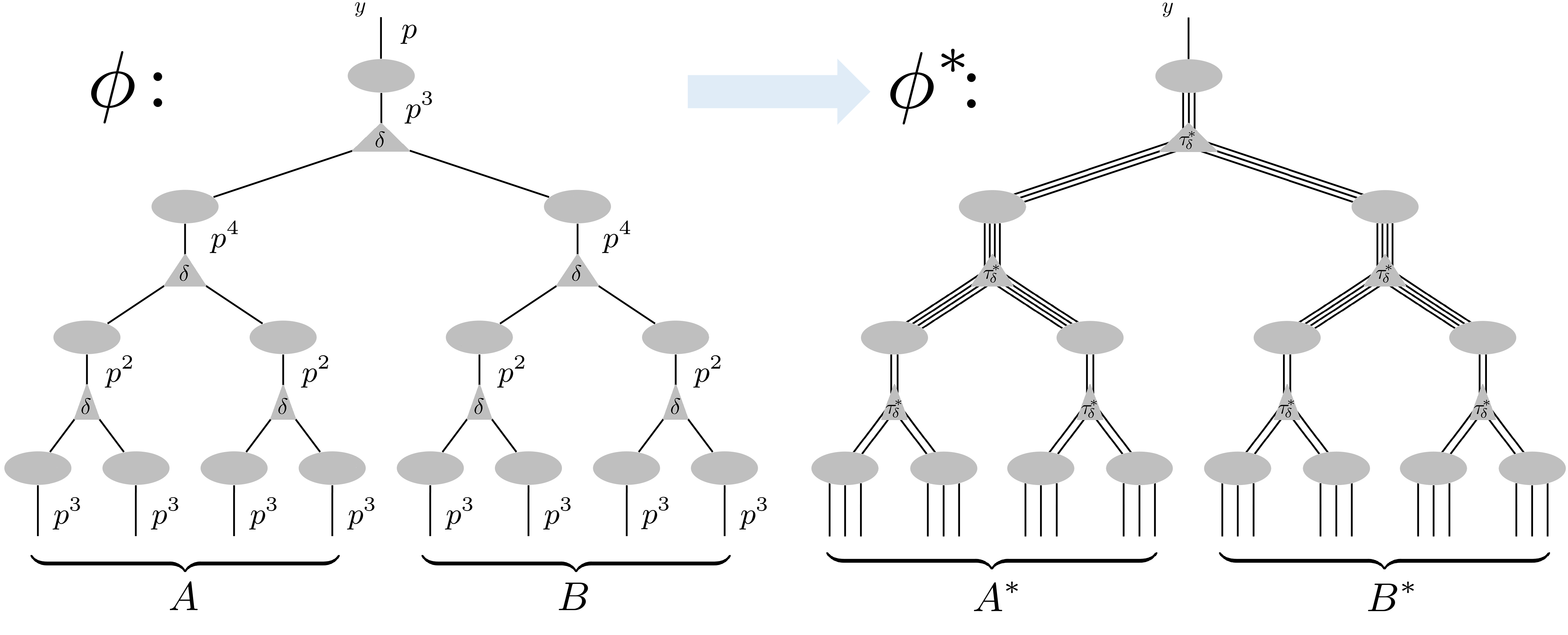}
\caption{An example for the construction of the TN $\phi^*$ out of the original TN $\phi$ which represents a deep ConvAC (section.~\ref{sec:translations:deep}), in the case where all of the bond dimensions are powers of some integer number $p$. $n_e$ edges with bond dimension $p$ are placed in $\phi^*$ in the position of each edge $e$ in $\phi$ that has a bond dimension $p^{n_e}$. This construction preserves the value of the minimal multiplicative cut between any two groups of external indices, $(A,B)$ in $\phi$ (here chosen as the left-right partition for example) which correspond to $(A^*,B^*)$ in $\phi^*$.}
\label{fig:phi_star}
\end{figure}
For the given partition $(A,B)$, denote $A=\{a_1,...,a_{|A|}\}~,~B=\{b_1,...,b_{|B|}\}$ and let $I_A \cupdot I_B=\{d_{1},\ldots,d_{N}\}$ be the corresponding partition of external indices, where $I_A=\{d_{a_1},...,d_{a_{|A|}}\}$ and $I_B=\{d_{b_1},...,d_{b_{|B|}}\}$. Let ${\cal H} ^A$ and ${\cal H} ^B$ with dimensions obeying eq.~\ref{totalExternalDimensions} be the spaces corresponding to the different configurations of the indices in $I_A$ and $I_B$, respectively.
We construct a TN $\phi^*$ with a graph $G^*(V^*,E^*)$ and a bond dimensions function $c^*~:~E^*\rightarrow\N$ for which there is a one-to-one correspondence between the tensor assignments in $\phi$ and tensor assignments in $\phi^*$, such that the resulting linear maps between ${\cal H}^A$ and ${\cal H}^B$ have the same rank. For each edge $e\in E$, denote $n_e := \log_p c(e)$. By the conditions of the lemma, $\forall e~:~n_e\in\N$ as $c(e)$ is a power of $p$ for all edges in $E$. The graph $G^*$ of the network $\phi^*$ is constructed as follows. Starting with $G^* = (V,\emptyset)$, for each edge $e = (u,v)\in E$ we insert $n_e$ parallel edges connecting $u$ to $v$ in $G^*$, to form the edge set $E^*$. Additionally, we define the bond dimensions function of the network $\phi^*$ to assign the value of $p$ to all of the added edges, \ie~ $\forall e^*\in E^*~:~c^*(e^*)=p$. In fig.~\ref{fig:phi_star} an example for such a construction of $\phi^*$ is shown for an $N=8$ ConvAC TN.

In the paragraphs below, we derive eq.~\ref{equiv_ranks} which shows that an analysis of $\phi^*$ suffices for our purposes. This result is intuitive in some sense, as the construction of $\phi^*$ keeps intact the architecture of the network and the distribution of the degrees of freedom to some extent. As it is the key to our proof, we formulate this argument hereinafter.

As each edge $e\in E$ was translated into $n_e$ edges in $E^*$, there are $N^*:=\log_p ({\textrm {dim}}({\cal H} ^A)\cdot{\textrm {dim}}({\cal H} ^B))$ external edges in $\phi^*$. Let $\A^*$ be the order $N^*$ tensor obtained by the contraction of the TN $\phi^*$. We denote by $(A^*,B^*)$ the partition of $[N^*]$ which corresponds to the partition $(A,B)$ of $[N]$. This means that an index number in $\A^*$ corresponding to an edge $e^*\in E^*$ would be in $A^*$ (resp. $B^*$) if the edge $e\in E$ from which it originated corresponded to an index number $\A$ that was in $A$ (resp. $B$). This is easily understood pictorially, see fig.~\ref{fig:phi_star}. Accordingly denote the corresponding partition of the degree $1$ vertices in $G^*$ by $(V^{A*},V^{B*})$. We will now show that the rank of the matricization of  $\A$ w.r.t. the partition $(A,B)$ is equal to the rank of the matricization of  $\A^*$ w.r.t. the partition $(A^*,B^*)$.

We denote by $\tau_v$ the tensors corresponding to a vertex $v\in V$ in the network $\phi$, and by $\tau^*_v$ the tensors corresponding to the same vertex $v$ in the network $\phi^*$. Let $z$ be the order of $\tau_v$, and denote the set of edges in $E$ incident to $v$ by $\{e_{k_1},...,e_{k_z}\}$ where $k_1,...,k_z$ are the corresponding indices. For every index $k_j,~j\in[z]$, let $K^{*j}=\{k^{*j}_1,...,k^{*j}_{n_{e_{k_j}}}\}$ be the indices corresponding to the edges which were added to $\phi^*$ in the location of $e_{k_j}$ in $\phi$. According to the construction above, there is a one-to-one correspondence between the elements in $K^{*j}$ and $k_j$, that can be written as:
\be
\label{index_trans}
k_j=h(K^{*j}):=1+\sum_{t=1}^{n_{e_{k_j}}} p^{t-1}(k^{*j}_t-1),
\ee
where $h~:~[p]^{\otimes n_{e_{k_j}}}\rightarrow[p^{n_{e_{k_j}}}]$.
Thus, if one has the entries of the tensors in $\phi^*$, the following assignment to the entries of the tensors in $\phi$:
\be
\label{phi_star_to_phi}
(\tau_v)_{k_1...k_z}= (\tau^*_v)_{h(K^{*1})...h(K^{*z})}
\ee
would ensure :
\be
\label{equiv_ranks}
~{\textrm {rank}} (\mat{\A}_{A,B})={\textrm {rank}} (\mat{\A^*}_{A^*,B^*}).
\ee
Effectively, we have shown that the claim to be proved regarding ${\textrm {rank}} (\mat{\A}_{A,B})$ can be equivalently proved for ${\textrm {rank}} (\mat{\A^*}_{A^*,B^*})$.

\medskip
\underline{\emph{The Form of the $\delta$ Tensor in $\phi^*$:}}
\medskip

It is worthwhile to elaborate on the form of a tensor in $\phi^*$ which corresponds to an order $3$ $\delta$ tensor in $\phi$. We denote by $\tau^v_\delta=\delta$ a $\delta$ tensor in $\phi$, and by $\tau^{*v}_\delta$ the corresponding tensor in $\phi^*$. Fig.~\hyperref[fig:app_B_supp]{~\ref{fig:app_B_supp}(a)} shows an example for a transformation to $\phi^*$ of an order $3$ $\delta$ tensor in $\phi$, all edges of which uphold $n_{e}=2$. From eqs.~\ref{index_trans} and~\ref{phi_star_to_phi}, and from the form of the $\delta$ tensor given in eq.~\ref{eq:deltadef}, it is evident that in this case an entry is non-zero in $\tau^{*v}_\delta$ only when $k^{*1}_1=k^{*2}_1=k^{*3}_1~{\it{and}}~k^{*1}_2=k^{*2}_2=k^{*3}_2$. In the general case, the condition for an entry of $1$ in $\tau^{*v}_\delta$ is:
\be
~\forall t\in[n_{e}]:~k^{*1}_t=k^{*2}_{t}=k^{*3}_{t},
\ee
where $n_{e}=\log_p c(e)$ for any edge $e$ incident to $v$ in $G$. Hence, a tensor $\tau^{*v}_\delta$ in $\phi^*$ which corresponds to a $\delta$ tensor in $\phi$ can be written as:
\be
\label{delta_star}
\tau^{*v}_\delta=\delta_{k^{*1}_1k^{*2}_1k^{*3}_1}\delta_{k^{*1}_2k^{*2}_2k^{*3}_2}...\delta_{k^{*1}_{n_{e}}k^{*2}_{n_{e}} k^{*3}_{n_{e}}}.
\ee

\medskip
\underline{\emph{$\delta$ Restricted Edge Disjoint Paths}}
\medskip

Consider an edge-cut set in $G$ that achieves the minimal multiplicative weight over all cuts w.r.t the partition $(V^A,V^B)$ in the graph $G$: $~C_{min}\in \argmin_C W_C$, and consider the corresponding edge-cut set $~C^*_{min}$ in $G^*$ s.t. for each edge $e\in ~C_{min}$, the $n_e$ edges constructed from it are in $~C^*_{min}$. By the construction of $G^*$, there are exactly $L:=\log_p(\min_C W_C)$ edges in $~C^*_{min}$ and their multiplicative weight upholds $W_{C^*_{min}}=W_{C_{min}}=p^L$.

A search for a minimal multiplicative cut, can be generally viewed as a classical min-cut problem when defining a maximum capacity for each edge that is a logarithm of its bond dimension. Then, a min-cut/max-flow value can be obtained classically in a graph with additive capacities and a final exponentiation of the result provides the minimal multiplicative value of the min-cut. Since all of the bond dimensions in $\phi^*$ are equal to $p$, such a process results in a network with all of its edges assigned capacity $1$. From the application of theorem~\ref{theorem:menger} on such a graph, it follows that the maximal number of \emph{edge disjoint paths} between $V^{A*}$ and $V^{B*}$ in the graph $G^*$, which are paths between $V^{A*}$ and $V^{B*}$ that are allowed to share vertices but are not allowed to share edges, is equal to the cardinality of the minimum edge-cut set $~C^*_{min}$ . In our case, this number is $L$, as argued above. Denote these edge disjoint paths by $q_1,...,q_L$.

In accordance with the form of $\tau^{*v}_\delta$, the tensors in $\phi^*$ corresponding to $\delta$ tensors in $\phi$ given in eq.~\ref{delta_star}, we introduce the concept of \emph{$\delta$ restricted edge disjoint paths} between $V^{A*}$ and $V^{B*}$ in the graph $G^*$, which besides being allowed to share vertices but not to share edges, uphold the following restriction. For every $\delta$ tensor $\tau^v_\delta$ of order $3$ in the graph $G$, with $e\in E$ a representative edge incident to $v$ in $G$, a maximum of  $n_e$ such paths can pass through $v$ in $G^*$, each assigned with a different number $t\in[n_e]$. The paths uphold that when passing through $v$ in $G^*$ each path enters through an edge with index $k^{*j_{in}}_{t_{in}}$ and leaves through an edge with index $k^{*j_{out}}_{t_{out}}$ only if $~j_{in}\neq j_{out}~:~j_{in},j_{out}\in[3]$ and $~t_{in}=t_{out}=t$, where no two paths can have the same $t$. This restriction imposed on the indices of $\tau^{*v}_\delta$ in $\phi^*$, to be called hereinafter the $\delta$ restriction, is easily understood pictorially, e.g. in fig.~\hyperref[fig:app_B_supp]{~\ref{fig:app_B_supp}(a)} the paths crossing the $\tau^{*v}_\delta$ tensor must only contain edges of the same color in order to uphold the $\delta$ restriction.

We set out to show, that for the network in question one can choose the $L$ edge disjoint paths to uphold the $\delta$ restriction. Then, a weight assignment compliant with the $\delta$ tensors in the network can be guaranteed to uphold the requirements of the lemma, despite the fact that most of the entries in the  $\delta$ tensors are equal to zero.

\begin{figure}
\centering
\includegraphics[width=\linewidth]{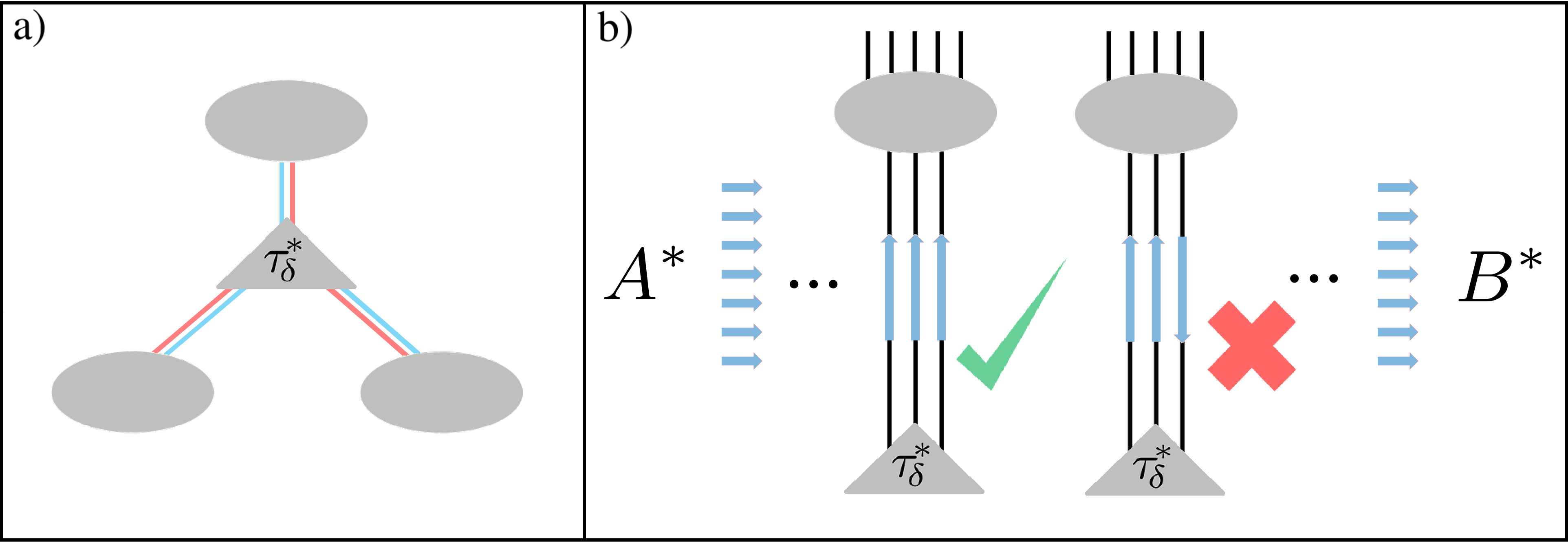}
\caption{a) An example for the tensor in $\phi^*$ which corresponds to a $\delta$ tensor $\tau^{v}_\delta\in\R^{p^2\times p^2\times p^2}$ in $\phi$. According to the construction of $\phi^*$ presented in fig.~\ref{fig:phi_star}, each edge is split into $n_e=2$ edges of bond dimension $p$. The $\delta$ tensor structure in $\phi$ translates into this $\tau^{*v}_\delta$ tensor holding a non-zero entry only when the indices corresponding to all of the edges that are marked by the same color are equal to each other (eq.~\ref{delta_star}). Additionally, paths crossing this $\tau^{*v}_\delta$ tensor must only contain edges of the same color in order to be called $\delta$ restricted edge disjoint paths. b) There are $L$ guaranteed edge disjoint paths between $V^{A*}$ and $V^{B*}$. In a flow directed from $V^{A*}$ to $V^{B*}$ (w.l.o.g), we argue that one can choose these paths such that they have the same flow direction in all edges in $\phi^*$ that originate from a certain edge in $\phi$.}
\label{fig:app_B_supp}
\end{figure}
Denote the set of $n_e$ edges in $G^*$ that originated from a certain edge $e$ in $G$, by $X^*_e\subset E^*$. We first show that one can choose the $L$ edge disjoint paths s.t. in a flow directed from $V^{A*}$ to $V^{B*}$ w.l.o.g, there is no set of edges $X^*_e$ that corresponds to any $e\in E$ for which two edges $e^*_i,e^*_j\in X^*_e\subset E^*$ belong to paths $q_i,q_j$ which flow in opposite directions. Fig.~\hyperref[fig:app_B_supp]{~\ref{fig:app_B_supp}(b)} clarifies this claim.

We observe the classical max-flow in the graph $G$, \ie~when assigning a maximum capacity for each edge $e$ that is equal to $n_e:=\log_pc(e)$, a maximum flow of $L$ is possible between $V^A$ and $V^B$ in $G$. Observing the paths in $G$ that flow w.l.o.g. from $V^A$ to $V^B$, together they can transfer a maximum capacity of $L$. Note that in $G$, these paths most certainly do not need to be edge disjoint paths. We argue that one can choose such paths from $V^A$ to $V^B$ in $G$ such that on each edge $e$ there is an integer capacity transferred. The existence of such paths in $G$ follows directly from the integral flow theorem (\cite{dantzig1956min}), which states that if each edge has integral capacity, then there exists such an integral maximal flow. Note, that these paths must also uphold the basic rule that the sum of capacities transferred on a certain edge $e\in E$, even if this is performed via several paths, is less than the edge maximum capacity $n_e$.

One can now construct $L$ paths in $G^*$ in a recursive manner, if a maximum additive capacity for each edge $e*\in E^*$ is similarly defined to be $\log_pc^*(e^*)=\log_pp:=1$. Starting with a single quanta of flow along some path in $G$, construct a single path in the corresponding position in $G^*$. Each edge that is part of this path in $G^*$ will transfer exactly one quanta of flow, as that is their maximum capacity that is chosen to be saturated in order to transfer the same amount of capacity that is transferred in $G$. Now, remove the full edges in $G^*$ and reduce the capacities of all edges along the original path in $G$ by one. Repeating this process until a capacity of $L$ is transferred in both graphs, since $n_e$ is the number of new edges added to $G^*$ in the place of each edge $e$, and it is also an upper bound on the \emph{integer} capacity this path transfers in $G$, it follows that in $G^*$ one finds $L$ paths between $V^{A*}$ and $V^{B*}$ that correspond exactly to the paths transferring integer capacity in $G$ guaranteed by integral flow theorem. These paths are edge disjoint since the edges of each path were removed from the graph when constructed. Choosing precisely these edge disjoint paths in $G^*$, one is guaranteed that the flow from $V^{A*}$ to $V^{B*}$ in all of the edges in $X^*_e$ that belong to these paths would be in the same direction, as they originated in the same edge $e$ in $G$ that had a flow in that single specific direction from $A$ to $B$. Pictorially, since the different edges in $X^*_e$ all originate from one single edge that obviously cannot have two opposite directions of net flow, they can all be chosen to transfer flow in the same direction.

Observing an order $3$ $\delta$ tensor $\tau^{v}_\delta$ in $\phi$, denote the three edges incident to $v$ in $G$ by $e_{1},e_{2},e_{3}\in E$, and denote $n_e:=n_{e_{1}}=n_{e_{2}}=n_{e_{3}}$. Now that we have asserted that all of the $L$ edge disjoint paths may uphold the above condition, we choose the paths as such, \ie~under this choice all of the edges in each respective set pass flow from $V^{A*}$ to $V^{B*}$ in the same direction. In this case, a maximum of $n_e$ paths can pass through the delta tensor. This can be easily understood by the following argument. Denote a set $X^*_{e_{i}}$ by `I' if the paths passing through its edges are incoming to the $\delta$ tensor in a flow from $V^{A*}$ to $V^{B*}$, and by `O' if they are outgoing from the $\delta$ tensor in such a flow. W.l.o.g. we assume that $X^*_{e_{1}},X^*_{e_{2}}$ are denoted by `I' and $X^*_{e_{3}}$ is denoted by `O'. in this case, only $n_e$ such edge disjoint paths can flow \emph{out} of the $\delta$ tensor. In the opposite case, where two groups of edges out of the three are denoted by `O' and only one group is denoted by `I', only $n_e$ such edge disjoint paths can flow \emph{into} the $\delta$ tensor.  The contrary, \ie~if more than $n_e$ such paths were to cross the $\delta$ tensor, would imply a cross flow of edge disjoint paths in at least one of the sets $X^*_{e_{1}},X^*_{e_{2}},X^*_{e_{3}}$, in contradiction to this choice of paths.

This provides us with the ability to distribute the paths in the following manner, that upholds the $\delta$ restriction described above. Assume w.l.o.g that $X^*_{e_{1}}$ is the set for which the most edges are in the chosen edge disjoint paths. Denote by $q_1,...,q_{N_2}$ the paths that include edges in $X^*_{e_{1}}$ and $X^*_{e_{2}}$, and by $q_{N_2+1},...,q_{N_2+N_3}$ the paths that include edges in $X^*_{e_{1}}$ and $X^*_{e_{3}}$. Finally, assign the index $t$ to the path $q_{t}$. From the statement above, it is guaranteed that $N_2+N_3\leq n_e$. Therefore, this choice of paths is guaranteed to uphold the delta restriction defined above, which states that each path must received a \emph{different} value $t\in [n_e]$. Specifically, this implies that the maximal number of $\delta$ restricted edge disjoint paths between $V^{A*}$ and $V^{B*}$ in the graph $G^*$ is $L$.

\medskip
\underline{\emph{The Assignment of Weights:}}
\medskip

We give below explicit tensor assignments for all the tensors in $\phi^*$ so that ${\textrm {rank}} (\mat{\A^*}_{A^*,B^*})={\textrm {min}}~W_C$, which in accordance with eq.~\ref{equiv_ranks} implies that ${\textrm {rank}} (\mat{\A}_{A,B})$ upholds this relation. Together with the translation from $\phi^*$ to $\phi$ given in eq.~\ref{phi_star_to_phi}, this will constitute a specific example of an overall assignment to the TN representing the ConvAC which achieves the lower bound shown in the lemma.

Observing the form of $\phi^*$, an example for which is shown in fig.~\ref{fig:phi_star}, we see that it is composed of tensors that correspond to $\delta$ tensors in $\phi$, denoted by $\tau^{*v}_\delta$, and of tensors that correspond to weight matrices in $\phi$, denoted by $A^{*(l,j)}$. Recalling that the entries of the former are dictated by construction and obey eq.~\ref{delta_star}, we are left with providing in assignment of all the tensors $A^{*(l,j)}$. The weight matrices in a ConvAC TN uphold $A^{(l,j)}\in\R^{r_l\times r_{l-1}}$, thus the corresponding tensors $A^{*(l,j)}$ are of order $\log_p (r_l\cdot r_{l-1})$ by construction, with $\log_p r_l$ edges directed upwards in the network and $\log_p r_{l-1}$ edges directed downwards.  For clarity, we omit the superscript from $A^{*(l,j)}$ and write the indices of such a weights tensor as:
\be
A^{k_1...k_{\log_p r_l}}_{k_{\log_p r_l+1}...k_{\log_p (r_l\cdot r_{l-1})}}.
\ee

We choose $L$ paths between $V^{A*}$ and $V^{B*}$ in the graph $G^*$ which are $\delta$ restricted edge disjoint paths, denoted by $q_1,...,q_{L}$. We are guaranteed to have this amount of $\delta$ restricted edge disjoint paths by the arguments made in the previous subsection. For any weights tensor $A^{*(l,j)}$ in $\phi^*$, let $n\in [\textrm {min}(L,\left \lfloor{\frac{1}{2}\log_p (r_l\cdot r_{l-1})}\right \rfloor)]\vee\{0\}$ be the number of $\delta$ restricted edge disjoint paths crossing it.

Let $\{g_{1\alpha},g_{1\beta},...,g_{n\alpha},g_{n\beta}\}$ with $g_{ix}\in [\log_p (r_l\cdot r_{l-1})],~i\in[n],~x\in\{\alpha,\beta\}$ be the numbers representing indices of $A^{*(l,j)}$ which correspond to edges that belong to any path $q_j~,~j\in[L]$,~\ie~the set of such indices is $\{k_{g_{1\alpha}},k_{g_{1\beta}},...,k_{g_{n\alpha}},k_{g_{n\beta}}\}$. Denoting $\bar{n}:=log_p (r_l\cdot r_{l-1})-2n$, let $\{f_1,...,f_{\bar{n}}\}$ with $f_i\in [\log_p (r_l\cdot r_{l-1})],~i\in[\bar{n}]$ be the numbers representing the remaining indices of $A^{*(l,j)}$,~\ie~the indices which correspond to edges that are not on any path $q_j~,~j\in[L]$. The set of such indices is $\{k_{f_1},...,k_{f_{\bar{n}}}\}$. Finally, the assignment of the entries of $A^{*(l,j)}$ is given by:

\be
\label{assignment}
A^{k_1...k_{\log_p r_l}}_{k_{\log_p r_l+1}...k_{\log_p (r_l\cdot r_{l-1})}}=
\delta_{k_{g_{1\alpha}}k_{g_{1\beta}}}\cdots
\delta_{k_{g_{n\alpha}}k_{g_{n\beta}}}
\delta_{1k_{f_1}}\cdots
\delta_{1k_{f_{\bar{n}}}}.
\ee

Effectively, the assignment in eq.~\ref{assignment} for the weights tensors ensures that their indices which correspond to two edges that are adjacent in one of the paths $q_i$, must be equal in order for the term not to vanish in the contraction of the entire TN $\phi$ . Since the paths $q_i$ are $\delta$ restricted, the $\tau^{*v}_\delta$ tensors in $\phi^*$ which corresponds to a $\delta$ tensor in $\phi$ are also able uphold this rule a priori. By this assignment, in accordance with the form of $\tau^{*v}_\delta$ given in eq.~\ref{delta_star}, the indices corresponding to all of the edges in a path $q_i$ are in fact enforced to receive the same value, ranging in $[p]$, in order for the term not to vanish in the contraction. An additional result of this assignment, is that all of the indices in the network which correspond to edges that do not belong to any path $q_i$, must be equal to $1$ in order for the term not to vanish (\ie~when they receive values of $2,...,p$ the term vanishes upon contraction).

according to the rules of TNs introduced in sec.~\ref{sec:TensorNetworks:intro}, the overall tensor $\A^*$ represented by the network $\phi^*$ is calculated by a global contraction which is a summation over all of the internal indices. Under the assignment in eq.~\ref{assignment}, upon a simple rearrangement of the tensor modes s.t. indices corresponding to $A^*$ are in the left, indices corresponding to $B^*$ are in the right and the indices corresponding to paths are first in their respective regions\footnote{This does not affect the rank of matricization as it is still preformed w.r.t. the partition $(A^*,B^*)$, and switching rows or columns leaves the rank unchanged.}, the only non-zero entries of $\A^*$ are:
\be
\text{{\Large $\A^*$}}{d_{q_1}...d_{q_{L}}{\overbrace{1~~.~~.~~.~~1}^{\tiny |A^*|-L}}d_{q_1}...d_{q_{L}}{\overbrace{1~~.~~.~~.~~1}^{\tiny |B^*|-L}}}=\text{{\large 1}},
\ee
where $d_{q_1},...,d_{q_{L}}\in[p]$ are the indices corresponding to the paths $q_1,...,q_{L}$, respectively. Upon matricization of $\A^*$ w.r.t. the partition $(A^*,B^*)$, it is clear that a matrix of rank $p^{L}=\min_C W_C$ with  $I_{p^{L}\times p^{L}}$ on its upper left block and zeros otherwise is received, and the lemma follows.
\end{proof}
With all the preliminaries in place, the proof of theorem \ref{theorem:lowerbound} readily follows:
\medskip
\begin{proof}(of theorem \ref{theorem:lowerbound})

For a specific $p$, consider the network $\phi^p$ such as defined in theorem \ref{theorem:lowerbound}, \ie~a TN with the same connectivity as $\phi$, where all of the bond dimensions are modified to be equal the closest power of $p$ to their value in $\phi$ from below. Let $(\A^p)^y$ be the weights tensor represented by $\phi^p$. According to lemma~\ref{lemma:powers_of_p}, such a network upholds that the rank of the matricization $\mat{(\A^p)^y}_{A,B}$ is at least: $~\min_C W^p_C$ almost always. According to lemma~\ref{lemma:phi_p_to phi}, a specific assignment for the weights of the tensors in $\phi^p$ that achieves this value suffices to imply that $\mat{\A^y}_{A,B}$ is at least: $\min_C W^p_C$ almost always, e.g. the assignment given in lemma~\ref{lemma:powers_of_p}. Specifically, this holds for $\phi^p$ with $p\in\argmax_p \min_C W^p_C$, satisfying the theorem.
\end{proof}

\section{Upper Bound for a General Pooling Window Case} \label{app:gen_pool}

In order to apply similar considerations to the ConvAC with general sized pooling windows, such as the one presented in fig.~\ref{fig:original_convac}, one must consider more closely the restrictions imposed by the $\delta$ tensors. To this end, we define the object underlying a ConvAC-weights TN with general sized pooling windows $\phi$ to be composed of the following three:
\begin{itemize}
\vspace{-2mm}
\item An undirected graph $G(V,E)$, with a set of vertices $V$ and a set of edges $E$. The set of nodes is divided into two subsets $V=V^{\textrm {tn}}\cupdot V^{\textrm {inputs}}$, where $V^{\textrm {inputs}}$ are the $N$ degree-$1$ virtual vertices and $V^{\textrm {tn}}$ corresponds to tensors of the TN.
\vspace{-2mm}
\item A function $f~:~E\rightarrow[b+N]$, where $b$ is the number of $\delta$ tensors in the network. If we label each $\delta$ tensor in the network by a number $i\in [b]$, this function upholds $f(e)=i$ for $e\in E$ that is incident to a vertex which represents the $i^{th}$ delta tensor in the ConvAC TN. For each edge $e\in E$ incident to a degree $1$ vertex, this function  assigns a different number $f(e)=i$ for $i\in b+[N]$. Such an edge is an external edge in the ConvAC TN, which according to the construction presented in sec.~\ref{sec:translations} is the only type of edge not incident to any $\delta$ tensor. In words, the function $f$ divides al the edges in $E$ into $b+N$ groups, where edges are in the same group if they are incident to the same vertex which represents a certain $\delta$ tensor in the network.
\vspace{-2mm}
\item A function $c~:~[b+N]\rightarrow\N$, associating a bond dimension $r\in\N$ with each different group of edges defined by the set: $E_i=\{e\in E: f(e)=i\}$.

\vspace{-2mm}

\end{itemize}

Observing an edge-cut set with respect to the partition $(A,B)$ and the corresponding set $G^C=\{f(e): e\in C\}$. We denote the elements of $G^C$ by $g^C_i, i\in[|G_C|]$. These elements represent the different groups that the edges in $C$ belong to (by the definition of $f$, edges incident to the same delta tensor belong to the same group). We define the modified weight of such an edge-cut set $C$ as:
\be
\tilde{W}_C=\prod\nolimits_{i=1}^{\abs{G_C}} c(g^C_i).
\label{CutWeightModified}
\ee
The weight definition given in eq.~\ref{CutWeightModified} can be intuitively viewed as a simple multiplication of the bond dimensions of all the edges in a cut, with a single restriction: the bond dimension of edges in the cut which are connected to a certain $\delta$ tensor, will only be multiplied once (such edges have equal bond dimensions by definition, see eq.~\ref{eq:deltadef}). An example of this modified weight can be seen in  fig.~\ref{fig:mincutmod}, where the replacement of a general tensor by a $\delta$ tensor results in a change in the minimal cut, due to the rules defined above. In the following claim, we provide an upper bound on the ability of a ConvAC with a general pooling window to model correlations of its inputs, as measured by the Schmidt entanglement measure (see sec.~\ref{sec:CorrelationEntanglement:quantumentanglement})

\begin{figure}
\centering
\includegraphics[width=\linewidth]{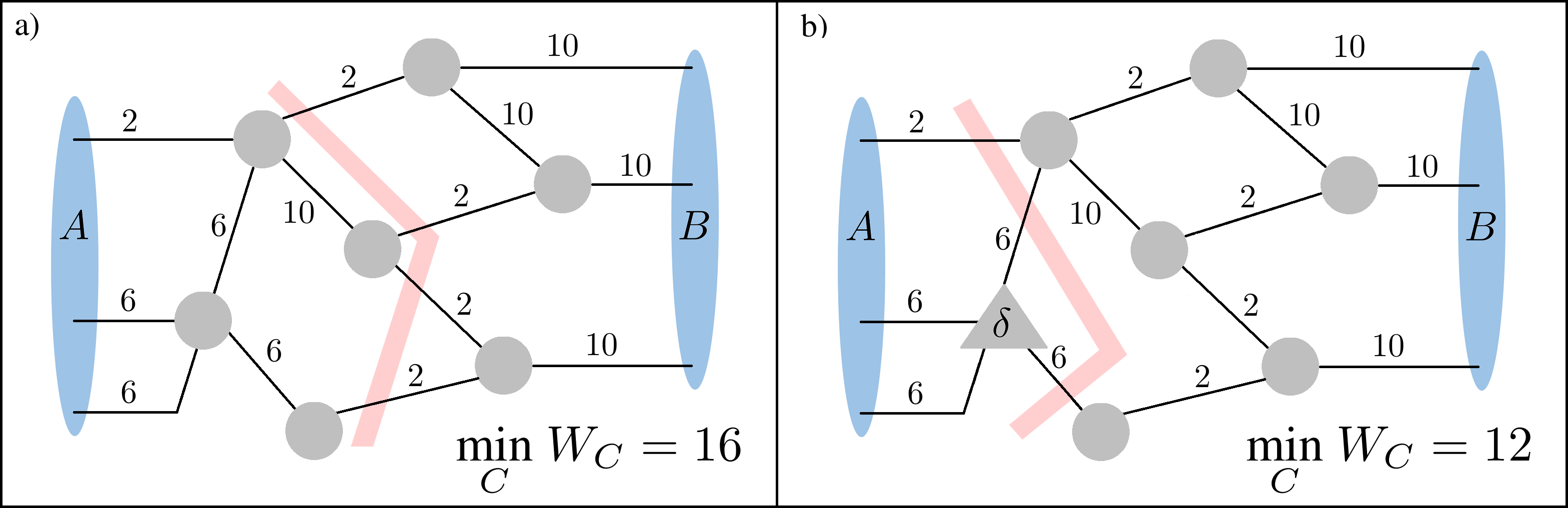}
\caption{An example for the effect that a $\delta$ tensor has on the upper bound on the rank of the matricization of the overall tensor represented by a TN. $~\min_C \tilde{W}_C$ is defined in eq.~\ref{CutWeightModified} and shown in claim~\ref{claim:upperbound_deltas} to be the upper bound on the rank of the matricization of the convolutional weights tensor of a ConvAC represented by a TN. In this example, the upper bound is reduced upon changing a single general tensor in the TN shown in a) (identical to fig.~\ref{fig:mincut} in the main text), whose entries are free to be equal any value, with a $\delta$ tensor in the TN shown in b) which obeys the constraint given in eq.~\ref{eq:deltadef}. The centrality of the $\delta$ tensor in the TN compliant with a shallow ConvAC (that is depicted in fig.~\hyperref[fig:CPExample]{~\ref{fig:CPExample}(b)}), is in effect the element which limits the expressiveness of the shallow network, as is discussed in sec.~\ref{sec:mincutclaim:implications_depth}}.
\label{fig:mincutmod}
\end{figure}

\begin{claim}\label{claim:upperbound_deltas}
Let $(A,B)$ be a partition of $[N]$, and $\mat{\A^y}_{A,B}$ be the matricization \wrt~$(A,B)$ of a convolutional weights tensor $\A^y$ (eq.~\ref{eq:convac}) realized by a ConvAC depicted in fig.~\ref{fig:original_convac} with a general pooling window. Let $G(V,E,f,c)$ represent the TN $\phi$ corresponding to the ConvAC network, and let $(V^A,V^B)$ be the vertices partition of $V^{\textrm {inputs}}$ in $G$ corresponding to $(A,B)$. Then, the rank of the matricization $\mat{\A^y}_{A,B}$ is no greater than: $~\min_C \tilde{W}_C$, where $C$ represents a cut w.r.t $(V^A,V^B)$ and $\tilde{W}_C$ is the modified multiplicative weight defined by eq.~\ref{CutWeightModified}.
\end{claim}

Having seen the proof of the claim~\ref{claim:upperbound} above and its accompanying graphics, the proof of the upper bound presented in claim~\ref{claim:upperbound_deltas} can be readily attained. The only difference between the two lies in the introduction of the $\delta$ tensors to the network, which allows us to derive the tighter lower bound shown in claim~\ref{claim:upperbound_deltas}.
\medskip
\begin{proof}(of claim~\ref{claim:upperbound_deltas})

The modification to the above proof of claim~\ref{claim:upperbound}  focuses on the coalescence of the cut indices $I_C$ into a single index $m\in [\prod_{i=1}^{|C|} c_{k_i}]$. Assume that any two indices of multiplicands in this product, denoted by $k_i$ and $k_j$, are connected to the same $\delta$ tensor that has some bond dimension $q:=c_{k_i}=c_{k_j}$. Upon contraction of the TN in fig.~\hyperref[fig:upper_bound]{~\ref{fig:upper_bound}(b)}, the cut indices are internal indices that are to be summed upon. However, whenever $k_i\in [q]$ and $k_j\in [q]$ are different, by the constraint imposed in the $\delta$ tensor definition (eq.~\ref{eq:deltadef}), the entire term vanishes and there is no contribution to the final value of $~\A_{d_1...d_N}$ calculated by this contraction. Thus, $k_i$, $k_j$ and any other index connected to the same $\delta$ tensor can be replaced by a representative index $k^\alpha \in [q]$ whenever they appear in the summation. $\alpha\in G^C$  upholding $c(\alpha)=q$, is the group index of the $\delta$ tensor, given by $\alpha=f(e_{k_i})=f(e_{k_j})$ with $e_{k_i}$ and $e_{k_j}$ the edges corresponding to the indices $k_i$ and $k_j$ in the network. Thus, the single index $m$ achieved by coalescing all of the cut indices can be defined in the range $m\in [\tilde{W}_C]$, with $\tilde{W}_C$ defined by eq.~\ref{CutWeightModified} upholding $\tilde{W}_C\leq\prod_{i=1}^{|C|} c_{k_i}$, where the equality is satisfied when no two edges in the cut are incident to the same $\delta$ tensor. Finally, the matricization $\mat{\A}_{A,B}$ can be written as a multiplication of two matrices as portrayed in fig.~\hyperref[fig:upper_bound]{~\ref{fig:upper_bound}(c)}:
\be
(\mat{\A}_{A,B})_{lr}=\sum_{m=1}^{\tilde{W}_C}(\mat{\X}_{A,C})_{lm}(\mat{\Y}_{C,B})_{mr},
\ee
$~l\in[{\textrm {dim}}({\cal H} ^A)],~r\in[{\textrm {dim}}({\cal H} ^B)]$. Recalling that the edge-cut set may include the external edges, we attain:
\be
{\textrm {rank}}(\mat{\A}_{A,B})\leq ~\underset{C}{\textrm {min}}~\tilde{W}_C.
\ee
\end{proof}

The result shown in claim~\ref{claim:upperbound_deltas} effectively reproduces the exponential depth efficiency result shown in \cite{cohen2016expressive}, as is explained in sec.~\ref{sec:mincutclaim:implications_depth}. In the deep network with pooling windows of size $2$ we were able to avoid such an influence of the $\delta$ tensor as $~\min_C \tilde{W}_C=\min_C W_C$ there. This is because the $\delta$ tensors in that network are only three legged (see fig.~\ref{fig:HT_TN}), which implies that a cut containing more than one $\delta$ tensor leg can be replaced by an equivalent cut containing only one leg of that $\delta$ tensor. This interestingly implies that for pooling windows of size $2$, the restriction imposed by the $\delta$ tensors, which we have shown to be effectively the same channel pooling restriction, does not harm the expressivity. In this special case, we have effectively shown that choosing a seemingly more elaborate pooling scheme which mixes all channels would not have benefitted the expressivity of the network.

\section{Upper Bound Deviations Simulation} \label{app:Simulation}

In this section, we describe simulations performed on an $N=16$ deep ConvAC TN (with pooling windows of size $2$), which are aimed at quantifying the prevalence of deviations from the upper bound on the ranks of the matricization of convolutional weights tensor presented in claim~\ref{claim:upperbound}. In appendix~\ref{app:Proofs:Lowerbound} we proved theorem~\ref{theorem:lowerbound}, showing in effect that this upper bound is tight when all of the channel numbers are powers of some integer $p$, and guaranteeing a positive result in all cases. However, for the general setting of channel numbers there is no theoretical guarantee that the upper bound is tight. Indeed, \cite{cui2016quantum} show a counter example where the matricization rank is effectively lower the minimal multiplicative cut for a general TN (that has no $\delta$ tensors such as in the ConvAC TN). There is no reason to believe that the upper bound is tight for the TN representing a ConvAC for a general setting of channel numbers, and indeed our simulations below show deviations from it. However, as is indicated below such deviations are negligible in prevalence and low in value. A theoretical  formulation of this is left for future work.

The experiments were preformed in matlab, and tensor contractions were computed using a function introduced by \cite{pfeifer2014ncon}. An $N=16$ with $M=2$ ConvAC TN was constructed (see figs.~\ref{fig:HT_TN} and~\ref{fig:HT_TN_Building_Block}), with the entries of the weights matrices randomized according to a normal distribution. The bond dimensions of layers $0$ through $3$ were drawn from the set of the first $6$ prime numbers: $[2,3,5,7,11,13]$, to a total of $360$ different arrangements of bond dimensions. This was done in order to resemble a situation as distinct as possible from the case where all of the bond dimensions are powers of the same integer $p$, for which the tight upper bound is guaranteed. Per bond dimension arrangement, all of the $\frac{1}{2}\cdot {16\choose8}=6435$ different partitions were checked, for a total of $360\cdot6435=2.3166\cdot 10^6$ different configurations. As argued in appendix~\ref{app:Proofs:Lowerbound}, the logarithm of the upper bound on the rank of the convolutional weights tensor matricization that is shown in claim~\ref{claim:upperbound}, is actually the max-flow in a network with the same connectivity that has edge capacities which are equal to the logarithm of the respective bond dimensions. Therefore, a configuration for which the rank of matricization is equal to the exponentiation of the max-flow through such a corresponding network, effectively reaches the upper bound. We calculated the max-flow independently for each configuration using the Ford-Fulkerson algorithm (\cite{ford1956maximal}), and set out to search for deviations from such an equivalence.

The results of the above described simulations are as follows. Only $1300$ configurations, which constitute a negligible fraction of the $2.3166$ million configurations that were checked, failed to reach the upper bound and uphold the min-cut max-flow equivalence described above. Moreover, in those rare occasions that a deviation occurred, the percentage of deviations from the upper bound did not exceed $10\%$ of the value of the upper bound. This check was performed on a bond setting that is furthest away from all channel numbers being powers of the same integer, yet the tightness of the upper bound emerges as quite robust, justifying experimentally our general view of the minimal weight over all cuts in the network, $~\min_C W_C$, as the effective indication for the matricization rank of the convolutional weights tensor w.r.t. the partition of interest. A caveat to be stated with this  conclusion is that we checked only up to $N=16$, and the discrepancies that were revealed here might become more substantial for larger networks. As mentioned above, this is left for future theoretical analysis, however the lower bound shown in theorem~\ref{theorem:lowerbound} guarantees a positive result regarding the rank of the matricization of the convolutional weights tensor in all of the cases.\end{document}